%% file: main.tex
\documentclass{article} %
\usepackage{iclr2027_conference,times}

\input{math_commands.tex}

\usepackage{amsmath}
\usepackage{thmtools}
\usepackage{amsthm}      %
\usepackage{thm-restate}
\usepackage{url}
\usepackage{derivative}
\usepackage{subcaption}
\usepackage{float}
\usepackage{adjustbox}
\usepackage{booktabs}
\usepackage{multicol}
\usepackage{rotating}

\theoremstyle{plain}
\newtheorem{theorem}{Theorem}[section]

\newtheorem{lemma}[theorem]{Lemma}

\theoremstyle{definition}

\newtheorem{assumption}[theorem]{Assumption}
\theoremstyle{remark}

\usepackage{hyperref}    %
\usepackage{cleveref}
\Crefformat{equation}{Equation~#2#1#3}

\makeatletter
\AddToHook{cmd/appendix/before}{\def\cref@section@alias{appendix}\def\cref@subsection@alias{appendix}}
\makeatother

\newcommand{\modelname}{\underline{SP}arse \underline{E}quivalent \underline{E}quation \underline{D}iscovery \underline{A}uto\underline{E}ncoder}
\newcommand{\modelnameshort}{SPEED-AE}

\title{Identifying ODEs from Unstructured Data with Causal Representation Learning}

\iclrfinalcopy

\author{Alessandro Trenta\thanks{Corresponding Author. Email: \texttt{alessandro.trenta@phd.unipi.it}} \\
Department of Computer Science\\
University of Pisa\\
\texttt{alessandro.trenta@phd.unipi.it} \\
\And
Riccardo Massidda \\
Department of Computer Science\\
University of Pisa\\
\texttt{riccardo.massidda@di.unipi.it} \\
\AND
Davide Bacciu \\
Department of Computer Science\\
University of Pisa\\
\texttt{davide.bacciu@unipi.it} \\
\And
Sara Magliacane \\
Saarland Informatics Campus\\
Saarland University \& University of Amsterdam\\
\texttt{sara.magliacane@gmail.com} \\
}

\begin{document}

\maketitle

\begin{abstract}
    We study the problem of recovering the governing ODE of a dynamical system from unstructured, high-dimensional observations such as images. Existing methods for ODE discovery typically assume direct measurements of the variables, or do not provide theoretical guarantees on the learned variables and equations. While Causal Representation Learning (CRL) methods provide guarantees on identifying variables from high-dimensional observations up to component-wise diffeomorphisms, we show that in general these variables cannot be used directly as input to equation discovery methods, which typically assume that the variables will lead to sparse equations.
    So we introduce SParse Equivalent Equation Discovery AutoEncoder (SPEED-AE), a framework that combines a pretrained CRL method with a component-wise autoencoder that learns transformations of variables that are amenable to sparse ODE discovery. We show that for polynomial ODEs, this additional step allows us to restrict the identifiability of each variable from polynomial to monomial diffeomorphisms. Experiments on Lotka-Volterra, Lorenz, and a two-pendulum system show that SPEED-AE improves on the disentanglement of the CRL methods and that it recovers ODEs that are closest to the ground truth, while achieving state-of-the-art forecasting performance.
\end{abstract}

\input{sections/introduction}

\input{sections/background}
\input{sections/theory}

\input{sections/method}

\input{sections/related}

\input{sections/experiments}

\input{sections/conclusion}

\bibliography{SPEED-AE}
\bibliographystyle{iclr2027_conference}
\clearpage
\appendix

\input{sections/appendix.tex}

\end{document}

%% file: math_commands.tex
\usepackage{amsmath,amsfonts,bm}

\def\eqref#1{equation~\ref{#1}}

\def\1{\bm{1}}

\def\rvu{{\mathbf{i}}}

\def\rvu{{\mathbf{u}}}

\def\rvx{{\mathbf{x}}}

\def\rvz{{\mathbf{z}}}

\def\vf{{\bm{f}}}

\def\vu{{\bm{u}}}

\DeclareMathAlphabet{\mathsfit}{\encodingdefault}{\sfdefault}{m}{sl}
\SetMathAlphabet{\mathsfit}{bold}{\encodingdefault}{\sfdefault}{bx}{n}

\newcommand{\R}{\mathbb{R}}

\newcommand{\real}{\mathbb{R}}

\newcommand{\zh}{\hat{z}}
\newcommand{\zt}{\tilde{z}}

\newcommand{\rvzh}{\hat{\rvz}}
\newcommand{\rvzt}{\tilde{\rvz}}
\newcommand{\enc}{\psi^\text{enc}}
\newcommand{\dec}{\psi^\text{dec}}
\newcommand{\sindyenc}{\psi^{\text{enc}}_{\text{S}}}
\newcommand{\sindydec}{\psi^{\text{dec}}_{\text{S}}}
\newcommand{\mnnenc}{\psi^{\text{enc}}_{\text{M}}}
\newcommand{\mnndec}{\psi^{\text{dec}}_{\text{M}}}
\newcommand{\crlenc}{\psi^{\text{enc}}_{\text{CRL}}}
\newcommand{\crldec}{\psi^{\text[dec]}_{\text{CRL}}}
\newcommand{\speedenc}{\phi^{\text{enc}}}
\newcommand{\speeddec}{\phi^{\text{dec}}}
\newcommand{\crldiffeo}{\hat{h}}

\newcommand{\zcrl}{\zh}
\newcommand{\zspeed}{\zt}
\newcommand{\vzcrl}{\rvzh}
\newcommand{\vzspeed}{\rvzt}

%% file: sections/introduction.tex
\section{Introduction}\label{sec:introduction}

Machine learning is increasingly used for scientific discovery, especially for dynamical systems and Ordinary Differential Equations (ODEs). Much of this work targets prediction and forecasting aspects of physical systems \citep{Chen2018NeuralOrdinaryDifferential, Hochreiter1997LongShortTermMemory, Vaswani2017AttentionAllYou}, possibly using ODEs as a tool to model the unknown dynamics of the system \citep{Chen2018NeuralOrdinaryDifferential, Dupont2019AugmentedNeuralODEs, Heinonen2018LearningUnknownODE, Rubanova2019LatentOrdinaryDifferential}. 
Equation discovery, i.e., learning ODE systems from data, has been extensively studied when direct access to the variables is available, both from a theoretical \citep{Scholl2024SymbolicRecoveryDifferential, Yao2024MarryingCausalRepresentation} and a practical point of view \citep{Brunton2016DiscoveringGoverningEquations, Kaheman2020SINDyPIRobustAlgorithm, Chen2024ScalableMechanisticNeural}. Most results focus on identifiability in linear ODEs from single trajectories \citep{Stanhope2014IdentifiabilityLinearLinearinParameters, Qiu2022IdentifiabilityAnalysisLinear}, sparse coefficients \citep{Casolo2025IdentifiabilityChallengesSparse}, and discrete observations \citep{Wang2024IdentifiabilityAsymptoticsLearning}. By contrast, theoretical identifiability in non-linear ODEs remains limited to known functional forms \citep{Grewal1976IdentifiabilityLinearNonlinear, Stanhope2014IdentifiabilityLinearLinearinParameters} or single initial conditions \citep{Yao2024MarryingCausalRepresentation}.

Few works address settings where variables are not directly observed, and only high-dimensional measurements (e.g., images or videos) are available. Existing approaches either learn equations on latent variables without guarantees that these variables correspond to the ground-truth  ~\citep{Champion2019DatadrivenDiscoveryCoordinates,Auzina2023ModulatedNeuralODEs}, or assume the underlying functional form of the ODE is known~\citep{Linial2021GenerativeODEModeling} and focus on parameter identification. 

In this work, we provide theoretical guarantees for identifying the true variables from unstructured data up to simple indeterminacies, on which we then learn ODE equations that correspond up to similar indeterminancies to the ground truth ones.
We build on Causal Representation Learning (CRL) \citep{Scholkopf2021CausalRepresentationLearning} methods, which provide guarantees on identifying variables from high-dimensional observations up to permutations and component-wise diffeomorphisms. 
We first show that the variables learned by CRL methods cannot in general be used directly as input to common equation discovery methods, e.g., SiNDy \citep{Brunton2016DiscoveringGoverningEquations}, which typically assume that the variables will lead to sparse equations. 
To tackle this problem, we introduce \modelnameshort{} (\modelname{}) a framework that combines a pretrained CRL method with a component-wise autoencoder that learns transformations of variables that are amenable to sparse ODE discovery. \Cref{fig:speed-ae} shows our \modelnameshort{} works: the frozen CRL model, which identifies the latents up to permutation and component-wise diffeomorphism, is followed by a component-wise AE. In the resulting latent space, the ODE is built from a predefined library, with sparsity imposed to recover an equation close to the true one. 

We also investigate the identifiability of these learned variables and their equations, showing that, in general, we cannot have any guarantees on their relation to the ground truth ones, even if we assume sparsity in the equations or disentangled representations. In Lemma~\ref{thm:identifiability_explicit} we show that for polynomial ODEs, considering sparsity in the learned equations allows us to restrict the identifiability of each variable from component-wise polynomial to monomial diffeomorphisms. 
Experiments on Lotka-Volterra, Lorenz, and a two-pendulum system show that SPEED-AE improves on the disentanglement of the CRL methods and that it recovers ODEs that are closest to the ground truth, while achieving state-of-the-art forecasting performance.

\begin{figure}
    \centering
    \includegraphics[width=0.95\linewidth]{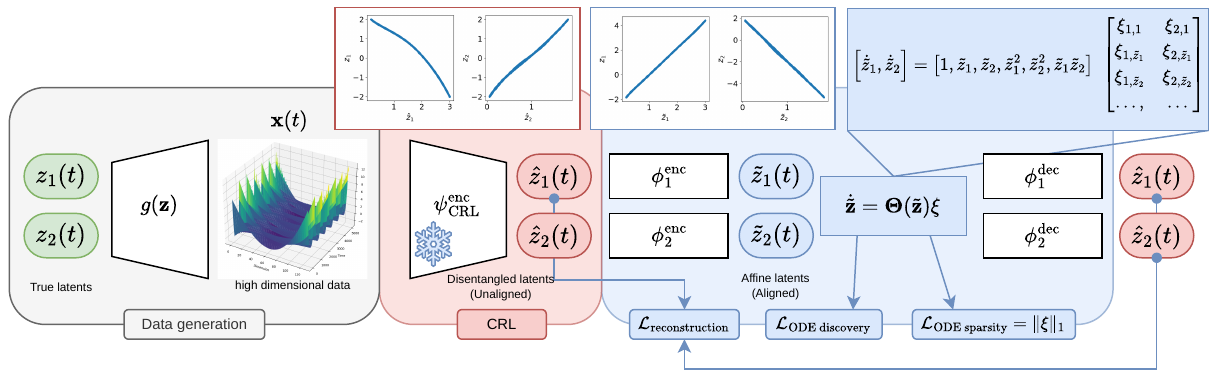}
    \caption{%
    \modelnameshort{} pipeline: we observe high-dimensional data from an invertible mapping $\rvx(t)=g(\rvz(t))$ of the true latents $\rvz(t)$. First, \modelnameshort{} employs a pretrained CRL model to obtain low-dimensional latents $\rvzh(t)$ (\textcolor{red!60!white}{red} block). We assume that the recovered latents correspond to the ground-truth $\rvz$ up to a permutation and component-wise diffeomorphism $z_i=\crldiffeo_i(\zcrl_{\pi(i)})$. For simplicity, in the figure we assume that the permutation is the identity function.
    \modelnameshort{} learns a set of component-wise autoencoders $(\speedenc_i, \speeddec_i)$ and an ODE by combining the functions in a library $\Theta(\vzspeed)$, imposing sparsity of this representation (\textcolor{blue!60!white}{blue} block). The output latents $\vzspeed(t)$ identify the ground-truth up to permutation and monomial and, in most cases, linear transformation.}
    \label{fig:speed-ae}
\end{figure}

%% file: sections/background.tex
\section{Background}\label{sec:background}

\paragraph{Ordinary Differential Equations (ODEs).}
Dynamical systems are commonly represented by ODEs \citep{Strogatz2019NonlinearDynamicsChaos}, which model the evolution of their state $\rvz$ through an equation of the form
\begin{equation}\label{eqn:ode}
    \dot{\rvz}(t) = \odv{\rvz}{t} = \vf(\rvz(t)), \qquad \rvz(0) = \rvz_0,
\end{equation}
where $\rvz_0$ is the initial state. We focus on \emph{autonomous} ODEs, for which $\vf$ does not explicitly depend on $t$, so that $\dot{\rvz}(t)$ solely depends on the current state $\rvz(t)$. If $\vf$ is Lipschitz continuous with respect to $\rvz$, the ODE admits a unique solution.
Here, we address the recovery of the underlying ODE of a low-dimensional system $\rvz(t)\in\mathcal{Z}\subseteq\real^d$ having access only to high-dimensional trajectories of unstructured data $\rvx(t) = g(\rvz(t))$ with $\rvx\in \mathcal{X} \subseteq \real^D$ and $D\gg d$, where $g: \mathcal{Z} \to \mathcal{X} $ is an unknown invertible mixing function, as assumed in Causal Representation Learning \citep{Scholkopf2021CausalRepresentationLearning}.

\paragraph{Equation Discovery.}
ODE discovery aims to recover the symbolic expression of a dynamical system from measurements of the state variables $\rvz$ along one or more trajectories. SINDy \citep{Brunton2016DiscoveringGoverningEquations} applies a library $\Theta$ of $L$ candidate functions, usually polynomials and some trigonometric functions, to the state variables $\rvz$ and seeks a sparse coefficient matrix $\Xi$ by minimizing
\begin{equation}\label{eqn:sindy_loss}
        \mathcal{L}_{\text{SINDy}}(\Xi) =
        \left\|\dot{\rvz} - \Xi^\top\Theta(\rvz)\right\|_2^2 + \beta_1\left\|\Xi\right\|_{0}.
\end{equation}
Here, $\Theta(\rvz)\in\real^{L}$ contains the candidate functions evaluated at $\rvz$, while $\Xi\in\real^{L \times d}$ selects them through its coefficients; the $L_0$ term promotes sparsity.
In practice, derivatives $\dot{\rvz}$ are unavailable and must be estimated from data, e.g.\ via finite differences. SINDyAE extends SINDy to recover ODEs on hidden state variables $\rvz\in\real^d$ from high-dimensional observations $\rvx\in\real^D$~\citep{Champion2019DatadrivenDiscoveryCoordinates}, jointly learning an encoder $\enc\colon\real^D\to\real^d$ and a decoder $\dec\colon\real^d\to\real^D$ by minimizing
\begin{align}
\mathcal{L}_{\text{SINDyAE}}(\Xi,\enc,\dec)
&= \underbrace{
    \left\| \rvx - \dec(\enc(\rvx)) \right\|_2^2
    }_{\text{Reconstruction loss}}
+ \beta_{\rvz} \underbrace{\left\| [\nabla\enc(\rvx)]\,\dot{\rvx}
  - \Xi^\top\Theta\!\left(\enc(\rvx)\right) \right\|_2^2}_{\text{SINDy loss in }\dot{\rvz}} \nonumber \\
&\quad + \beta_{\rvx} \underbrace{\left\| \dot{\rvx}
  - [\nabla\dec\!\left(\enc(\rvx)\right)]\,
    \Xi^\top\Theta\!\left(\enc(\rvx)\right) \right\|_2^2}_{\text{SINDy loss in }\dot{\rvx}}
+ \beta_1 \underbrace{\left\| \Xi \right\|_1}_{\text{sparsity}}.
\end{align}
Here $[\nabla\enc(\rvx)]\in\real^{d\times D}$ and $[\nabla\dec(\enc(\rvx))]\in\real^{D \times d}$ are the encoder and decoder Jacobians evaluated at $\rvx$ and $\enc(\rvx)$, respectively;
$\{\beta_{\rvx}, \beta_{\rvz}, \beta_1\}$ weigh the different loss terms. The $L_1$ penalty, coupled with sequential thresholding, is used as a smooth proxy of the $L_0$ optimization.
\citet{Pervez2024MechanisticNeuralNetworks} and 
\citet{Chen2024ScalableMechanisticNeural} propose a related approach that replaces derivative fitting by a Mechanistic Neural Network (MNN) that learns the coefficients $\Xi$ and efficiently simulates trajectories. In our experiments, we compare to SINDyAE and MNNAE, which we further discuss in \Cref{app:baselines}.

\paragraph{Causal Representation Learning (CRL).}
CRL methods \citep{Scholkopf2021CausalRepresentationLearning} can provably recover the ground truth latent variables $\rvz\in\real^d$ from observations $\rvx\in\real^D$ up to certain specific classes of transformations, under different sets of assumptions or types of data. A CRL method typically learns a disentangled encoder $\enc\colon\real^D\to\real^d$ where the latent space $\hat{\rvz}=\enc(\rvx)$ relates to the ground-truth variables $\rvz$ through a diffeomorphism $h\colon\real^d\to\real^d$. We focus on methods guaranteeing \emph{identifiability up to permutation and component-wise diffeomorphisms}, i.e., $z_i = h_i(\hat{z}_{\pi(i)})$ for each $i\in\{1,\ldots,d\}$ and permutation $\pi$. This type of identifiability guarantee on the learned variables is often also called \emph{disentanglement}. In our experiments, we consider methods that use temporal data and interventions or actions to identify these variables, in particular CITRIS \citep{Lippe2022CITRISCausalIdentifiability} and DMSVAE \citep{Lachapelle2022DisentanglementMechanismSparsity, Lachapelle2026NonparametricPartialDisentanglement}, which we also discuss in App.~\ref{app:crl_details}, but in general our method is agnostic to the specific CRL method with the same identifiability class.

%% file: sections/theory.tex
\section{Identifiability Analysis for Polynomial ODEs}%
\label{sec:identifiability}

\begin{figure}
    \centering
    \includegraphics[width=0.85\linewidth]{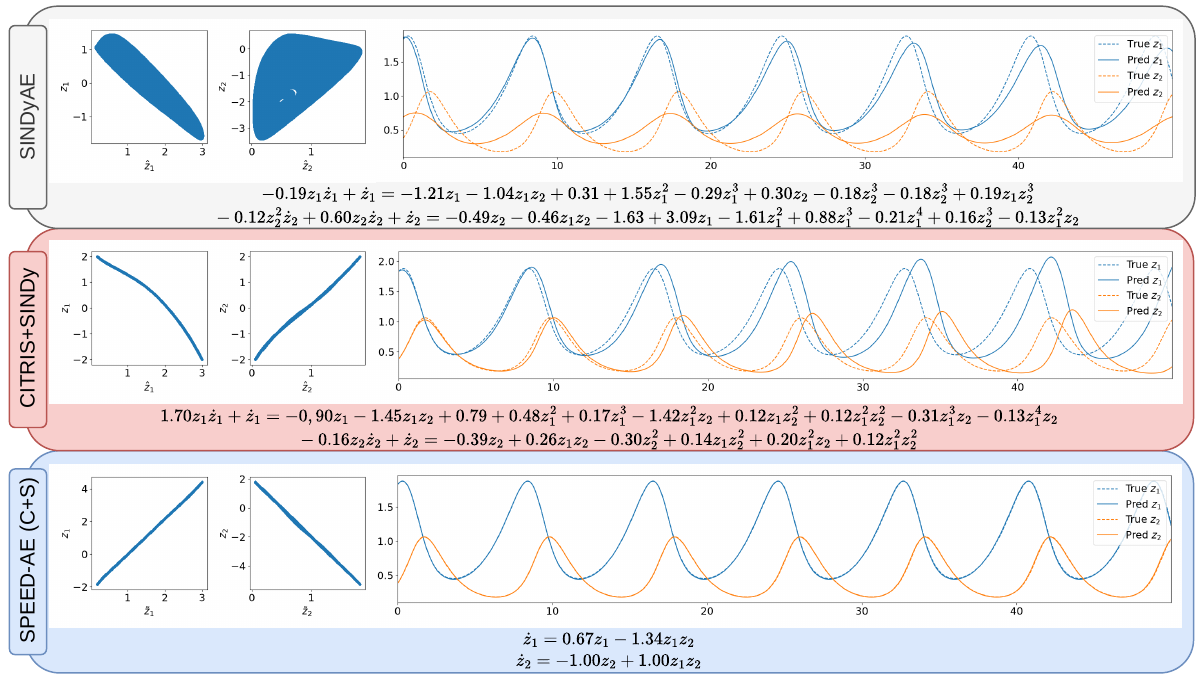}
    \caption{Comparison between SINDyAE, SINDy applied to CRL latents (using CITRIS), and \modelnameshort{} (with CITRIS as CRL method and SINDy as equation discovery method) on the Lotka-Volterra dataset ($\dot{z_1} = 0.66 z_1 - 1.33 z_1z_2$, $\dot{z_2} = -z_2 + z_1z_2$). For each model, the left box reports scatter plots of the true versus learned latents. We then compare a test trajectory with the one predicted from only the initial condition and the discovered equation, after converting both to the true latent space, following the evaluation protocol in Section \ref{sec:experiments}. \modelnameshort{} is the only method that recovers a sparse equation and achieves strong long-term prediction performance.}
    \label{fig:model_comparison}
\end{figure}

We study the identifiability of systems governed by a latent ODE $\dot{\rvz} = \vf(\rvz)$ from high-dimensional observations $\rvx(t) = g(\rvz(t))$, where $g$ is an invertible mixing function. Given a smooth invertible encoder $\tilde{\rvz}=\enc(\rvx)$, the map $h$ from the latent space $\tilde{\rvz}$ to the ground-truth variables $\rvz$ is a diffeomorphism. Consequently, as we derive in Prop.~\ref{proposition:variable_change}, the encoded latent variables $\tilde{\rvz}$ are also governed by an ODE $\dot{\tilde{\rvz}}=\tilde{\vf}(\tilde{\rvz})=[\nabla h(\tilde{\rvz})]^{-1}\vf(h(\tilde{\rvz}))$.
Formally, whenever the variables of the two ODEs are related by a diffeomorphism, we say that the two ODEs are \textit{conjugate}~\citep{Sideris2013OrdinaryDifferentialEquations}.

By jointly learning an autoencoder $\dec\circ\enc$ and an ODE $\tilde{\vf}$ over the latent space $\tilde{\rvz}=\enc(\rvz)$, SINDyAE~\citep{Champion2019DatadrivenDiscoveryCoordinates} effectively recovers a diffeomorphically conjugate ODE. However, despite enforcing sparsity in $\tilde{\vf}$, SINDyAE does not provide any guarantee on the recovery of the form of the true ODE or the ground-truth variables. As we show in a toy example in App.~\ref{proof:sindyae}, even for polynomial ODEs, the SINDyAE loss cannot guarantee that the learned variables correspond one-to-one to the state variables. Intuitively, this suggests that the equations learned on these
variables might in general be of a more complicated form than the ground truth ones.

CRL methods learn a disentangled representation $\vzcrl$ that identifies the true latents $\rvz$ up to permutation $\pi$ and component-wise diffeomorphism $z_i = h_i(\zcrl_{\pi(i)})$. Yet, even if we used an equation discovery method, e.g., SINDy \citep{Brunton2016DiscoveringGoverningEquations}, directly on these variables and the learned ODE $\hat{\vf}$, this could potentially still have an arbitrary form. In particular, given the transformations $\{{h}_i\}_{i=1}^d$ induced by a perfectly disentangled encoder $\crlenc$, we can always write a conjugate ODE as
\begin{equation}\label{eqn:single_ode_change}
    \dot{\hat{z}}_{\pi(i)} = \hat{f}_{\pi(i)}(\hat{\rvz}) =
    [h'_i(\hat{z}_{\pi(i)})]^{-1}f_i({h}_1(\hat{z}_{\pi(1)}), \ldots, {h}_d(\hat{z}_{\pi(d)}))
    \quad
    \text{for } i=1,\ldots,d,
\end{equation}
where $h'_i(\hat{z}_{\pi(i)})$ are the the component-wise derivatives. Thus, even with perfect disentanglement, we cannot in general guarantee that the learned ODE will be similar to the ground truth one. We therefore restrict ourselves to polynomial ODEs and assume that these component-wise diffeomorphisms are polynomial. In particular, we show that if the learned representations are disentangled and governed by a sparse polynomial ODE, we can identify the true latents up to component-wise monomial transformations. For tractability reasons, we introduce two assumptions (see Ass.~\ref{app:proofs/identifiability_explicit}): after the change of variables, powers of polynomials and their full expansion do not undergo cancellations of terms (as per our remarks in App.~\ref{app:proofs/identifiability_explicit}, this happens rarely).

\begin{restatable}{lemma}{identifiabilitythm}\label{thm:identifiability_explicit}
    Let $\dot{\rvz}=\vf(\rvz)$ and $\dot{\tilde{\rvz}}=\tilde{\vf}(\tilde{\rvz})$ be two polynomial ODEs 
    such that for every $i \in [d]$ 
    $z_i=h_i(\zspeed_{\pi(i)})$ for an invertible polynomial $h_i$ and permutation $\pi$. 
    Let $q_i(z_i)$ be the highest degree polynomial such that $f_i(\rvz) = q_i(z_i)r_i(\rvz)$, where $r_i(\rvz)$ is the remainder polynomial and $N(r_i)$ be the number of monomial terms in $r_i(\rvz)$. Similarly, we define $\tilde{r}_{\pi(i)}(\vzspeed)$ and $N(\tilde{r}_{\pi(i)})$ for $\tilde{f}_{\pi(i)}$. 
    Under Ass.~\ref{ass:non_cancellation_power} and \ref{ass:non_cancellation_expansion}, $N(\tilde{r}_{\pi(i)})\geq N(r_i)$. If $N(\tilde{r}_{\pi(i)})= N(r_i)$, then $h_i$ is a monomial transformation, i.e., $z_i=h_i(\zspeed_{\pi(i)}) = a_i\zspeed_{\pi(i)}^{p_i}$ with $p_i \in \mathbb{N}^+$. Moreover, if $q_i$ is constant, then $p_i =1$, and $h_i$ is affine.

\end{restatable}

The proof is in App.~\ref{app:proofs/identifiability_explicit}. Intuitively, \Cref{thm:identifiability_explicit} states that for polynomial ODEs, if the learned variables $\tilde{\rvz}$ correspond to the ground truth ones $\rvz$ up to permutations and component-wise polynomial diffeomorphisms, and we learn $\tilde{\vf}(\tilde{\rvz})$ that is as sparse as possible (i.e., the number of its active monomial terms is as low as possible, which is $N(\tilde{r}_{\pi(i)})= N(r_i)$), then this will lead to component-wise transformations that are monomials. Although this result holds only for polynomial ODEs and polynomial diffeomorphisms, we conjecture that if a representation is informed by sparsity and the form of the governing ODE, it results in a simpler component-wise transformation than with a standard causal representation learning method. The intuition of learning transformations of variables that have sparse ODE equations is similar to SINDyAE \citep{Champion2019DatadrivenDiscoveryCoordinates}, but we can now provide theoretical guarantees on the identifiability of the variables up to simple indeterminacies.

%% file: sections/method.tex
\section{SParse Equivalent Equation Discovery AutoEncoder}%
\label{sec:methodology}

The analysis in Sec.~\ref{sec:identifiability} shows that learning sparse equations does not provide any guarantees on the learned variables, while disentanglement of the learned variables in general does not provide guarantees on the learned equations. Moreover, as shown in Fig.~\ref{fig:model_comparison}, the representations learned by CRL methods might not be amenable to equation discovery methods, leading to complicated equations that are far from the ground truth, even in simple cases, like polynomial ODEs. So we propose \modelname{} (\modelnameshort{}), which introduces an additional step after a pretrained CRL method: a component-wise autoencoder that learns a transformation of each CRL variable, such that these transformed variables have sparse ODE equations. 

As a first step, \modelnameshort{} runs a CRL method to obtain latents $\vzcrl = \crlenc(\rvx)$ that we assume recover the true variables up to permutation and component-wise diffeomorphism $z_i = \crldiffeo_i(\zcrl_{\pi(i)})$. Using these variables, it learns component-wise autoencoders $\zspeed_i =\speedenc_i(\zcrl_i)$ together with a sparse ODE over $\zspeed_i$, via an equation discovery model. The component-wise autoencoders of SPEED-AE guarantee that the new latents $\tilde{\rvz}$ are still disentangled with respect to the true variables $\rvz$. Our model does not assume any latent ODE form. However, under the assumptions of \Cref{thm:identifiability_explicit}, \modelnameshort{} identifies the true latents up to simple monomial transformations.
\modelnameshort{} is agnostic about the CRL methods (as long as they guarantee identifiability up to permutation and component-wise diffeomorphisms) and about the equation discovery approaches.
\begin{equation}
\label{eqn:model_generic_loss}
        \mathcal{L}_{\text{\modelnameshort{}}} = \sum_{i=1}^{d}\left\|\speeddec\left(\speedenc(\zcrl_i)\right) - \zcrl_i\right\|_2^2, + \mathcal{L}_{\text{ODE discovery}} + \mathcal{L}_{\text{ODE sparsity}}.
\end{equation}

For instance, we can use a SINDy-based loss representing the ODE as $\dot{\rvz}= \Xi^\top\Theta(\vzspeed)$, where $\Theta(\vzspeed)\in\mathbb{R}^{L}$ is a library of $L$ functions of $\vzspeed$ weighted by coefficients $\Xi\in\mathbb{R}^{L\times d}$:
\begin{equation}\label{eqn:sindy_based_loss}
    \mathcal{L}_{\text{ODE discovery}} =
    \beta_{\rvz}
    \left\|
    [\nabla\speedenc(\vzcrl)](\dot{\vzcrl})
    -
    \Xi^\top\Theta(\vzspeed)
    \right\|_2^2
    +
    \beta_{\rvx}
    \left\|
    \dot{\vzcrl}
    -
    [\nabla\speeddec(\vzspeed)]
    (\Xi^\top\Theta(\vzspeed))
    \right\|_2^2.
\end{equation}
The same framework can instead use an MNN-based loss:
\begin{equation}\label{eqn:mnn_based_loss}
    \mathcal{L}_{\text{ODE discovery}} = \beta_{\rvz}\left\|\vzspeed - \text{MNN}(\vzspeed(0))\right\|_2^2 + \beta_{\rvx}\left\|\vzcrl - \speeddec(\text{MNN}(\speedenc(\vzcrl(0)))\right\|_2^2,
\end{equation}
where $\text{MNN}(\vzspeed_0)$ is the solution to the learned ODE $\dot{\tilde{\rvz}}=\Xi^\top\Theta(\vzspeed)$ simulated by the MNN. In both variants, the second term, weighted by $\beta_{\rvx}$, prevents collapse of $\vzspeed$ by enforcing consistency between dynamics in $\vzspeed$ and their push-forward $\vzcrl$, similarly to \citet{Champion2019DatadrivenDiscoveryCoordinates}. 
For the sparsity loss, we follow  \citet{Brunton2016DiscoveringGoverningEquations,Champion2019DatadrivenDiscoveryCoordinates, Chen2024ScalableMechanisticNeural} and use an $L_1$ penalty with sequential thresholding as a proxy for the $L_0$ norm. The penalty sparsifies $\Xi$, while coefficients below the threshold are periodically set to $0$. 
To reflect the setting of Sec.~\ref{sec:identifiability}, we first implement the encoder $\speedenc$ and decoder $\speeddec$ as learnable polynomial functions. However, our results on simple synthetic experiments (App.~\ref{app:ablations}) are unsatisfactory, as these representations were not flexible enough during training. Hence, in our final model, we use MLPs for both $\speedenc$ and $\speeddec$. This is also justified empirically that the CRL methods we use only guarantee  identifiability up to general diffeomorphism only. Thus, it is not enough to only consider polynomial transformations and the autoencoder has to represent a more general function.
Implementation details are in App.~\ref{app:implementation}.

%% file: sections/related.tex
\section{Related Work}%
\label{sec:related}

\paragraph{Equation discovery.} Learning ODE systems from data has been extensively studied when we can observe the state variables directly, both from a theoretical \citep{Scholl2024SymbolicRecoveryDifferential} and a practical point of view \citep{Brunton2016DiscoveringGoverningEquations, Kaheman2020SINDyPIRobustAlgorithm, Chen2024ScalableMechanisticNeural}. SINDY \citep{Brunton2016DiscoveringGoverningEquations}, the principal sparse regression approach, considers a library of functions evaluated on the variables $\Theta(\vzcrl)$ (usually polynomials and some trigonometric functions) and learns the coefficients $\Xi$ of a linear model that predicts the variables' derivatives $\dot{\vzcrl}$ from the linear combination of the functions in the library $\Xi^\top\Theta(\vzcrl)$. With the implicit assumption that most dynamical systems in nature have a sparse representation, a sparsity regularization on $\xi$ is employed to select the simplest model among all solutions. Mechanistic Neural Networks (MNNs, \citep{Pervez2024MechanisticNeuralNetworks, Chen2024ScalableMechanisticNeural}) are general approaches to model the evolution of dynamical systems by building an internal ODE representation. The mechanistic encoder maps the trajectory into the coefficients and other parameters of the ODE, whose solution is then calculated by solving an equivalent linear system efficiently. By using the state variables $\rvz$ directly as the internal ODE variables and matching the available trajectories, the MNN can be effectively used for equation discovery \citep{Pervez2024MechanisticNeuralNetworks} based on a library of basis functions. Finally, symbolic recovery approaches \citep{Becker2023PredictingOrdinaryDifferential, DAscoli2022DeepSymbolicRegression, DAscoli2024ODEFormerSymbolicRegression} treat variables and operation operators (such as addition, multiplication, exponential) as tokens in a transformer-based regression approach.
Recent work on theoretical results for ODE identifiability focuses on linear or affine ODEs \citep{Qiu2022IdentifiabilityAnalysisLinear, Duan2020IdentificationAffineDynamical}, and on particular settings such as unobserved variables \citep{Wang2024IdentifiabilityAnalysisLinear}, discrete observations \citep{Wang2024IdentifiabilityAsymptoticsLearninga}, or sparse coefficients \cite{Casolo2025IdentifiabilityChallengesSparse}. By contrast, theoretical identifiability in non-linear ODEs remains limited to known functional forms \citep{Grewal1976IdentifiabilityLinearNonlinear, Stanhope2014IdentifiabilityLinearLinearinParameters} or single initial conditions \citep{Yao2024MarryingCausalRepresentation}.

\paragraph{Equation discovery on unstructured data.} Few works address settings where variables are not directly observed, and only high-dimensional measurements (e.g., images or videos) are available. Existing approaches either learn equations on latent variables without guarantees that these variables correspond to the ground-truth  ~\citep{Champion2019DatadrivenDiscoveryCoordinates,Auzina2023ModulatedNeuralODEs}, or assume the underlying functional form of the ODE is known~\citep{Linial2021GenerativeODEModeling} and focus on parameter identification. The most related work to us, SindyAE \citep{Champion2019DatadrivenDiscoveryCoordinates} uses an autoencoder to learn representations of the state variables in the latent space, on which it then learns ODEs with Sparse Identification for Nonlinear Dynamics (SINDy) \citep{Brunton2016DiscoveringGoverningEquations}. Similarly to us it encourages learning representations that can be used in sparse equations, but as opposed to us it does not provide any theoretical guarantee or analysis on the quality of the learned representations.

\paragraph{Causal Representation Learning.} Causal Representation Learning (CRL) ~\citep{Scholkopf2021CausalRepresentationLearning} methods provably identify a latent representation $\hat{\rvz}$ that identifies the true latents $\rvz$ up to permutation and a component-wise transformation. CRL methods exploit multi-view and multi-environment settings \citep{Kugelgen2021SelfSupervisedLearningData, Xu2024SparsityPrinciplePartially, Yao2023MultiViewCausalRepresentation}, interventional data \citep{Lippe2022CITRISCausalIdentifiability, Kugelgen2023NonparametricIdentifiabilityCausal, Ahuja2023InterventionalCausalRepresentation, Squires2023LinearCausalDisentanglement}, temporal data and actions \citep{Lippe2023BISCUITCausalRepresentation,Lachapelle2022DisentanglementMechanismSparsity,Lachapelle2026NonparametricPartialDisentanglement} or partially observable settings \citep{Xu2024SparsityPrinciplePartially}. While CRL methods focus on identifiability of state variables from unstructured data, they generally do not learn ODE systems, and do not leverage sparsity in the resulting equations to improve identifiability, as we do. Recently, \citet{Yao2024MarryingCausalRepresentation} identify time-invariant trajectory-specific parameters, but then assume that the state variables are directly observed and that the trajectories are generated from a single initial condition. In our work we do not assume neither of these assumptions, but we then assume that the parameters of the ODE we are learning are the same across all trajectories. These works are therefore complementary and could be potentially combined.

%% file: sections/experiments.tex
\section{Experiments}%
\label{sec:experiments}

We evaluate \modelnameshort{} on different versions of three dynamical systems: Lotka-Volterra, Lorenz, and two pendulum dynamical systems, but instead of observing the state directly, we observe high-dimensional measurements $\rvx$, produced by an invertible mixing function on the true state variables $\rvz$. We discuss in detail the data generation and model training in App.~\ref{app:experimental_details}, including the pretraining of the CRL models in each setting by considering a separate dataset of trajectories that also contains interventions on each variable $z_i$ with probability $p=0.01$. 

We consider two CRL approaches, CITRIS \citep{Lippe2022CITRISCausalIdentifiability} and DMSVAE \citep{Lachapelle2022DisentanglementMechanismSparsity, Lachapelle2026NonparametricPartialDisentanglement} and two ODE discovery approaches, SINDy \citep{Brunton2016DiscoveringGoverningEquations} and MNN \citep{Pervez2024MechanisticNeuralNetworks, Chen2024ScalableMechanisticNeural}. We denote each 
combination as \modelnameshort{}(X+Y), where X indicates the CRL method (C for CITRIS or D for DMSVAE), and Y indicates the equation discovery method (S for SINDy or M for MNN). We compare with 
the only two baselines that learn both the latents $\vzspeed$ and the governing equations without additional information: 
SINDyAE \cite{Champion2019DatadrivenDiscoveryCoordinates} and MNNAE, which combines the MNN with an AE, discussed in App.~\ref{app:baselines}. In the original SINDyAE, the true image derivatives $\dot{\rvx}$ are provided as input; we instead use empirical derivatives to better reflect realistic conditions. Moreover, we also compare, as an ablation, a direct combination of the CRL methods that we use with the equation discovery methods that we use.

We evaluate three different metrics: (a) correlation discrepancy between the ground truth and learned variables, (b) trajectory prediction, and (c) ODE recovery. We define the correlation discrepancy ($CorrD$) as the MAE between the Pearson correlation of the ground truth variables $\rvz$ and the cross-correlation matrix between the ground truth $\rvz$  and learned variables $\tilde{\rvz}$, defined as
$
CorrD= \frac{1}{d^2} \sum_{i,j}^d | \rho(\rvz_i,\vzspeed_j)-\rho(\rvz_i,\rvz_j) |.
$ We use this metric to evaluate the quality of disentanglement of the state variables, while allowing for dependences between each ground truth state variables. Intuitively, lower is better and $CorrD=0$ corresponds to perfect disentangelment. 
For trajectory prediction and ODE recovery, we have to first address a key challenge: after training, each model represents the system in its own latent space, which means that they are not directly comparable. To compare them, we therefore convert all of the learned representations to the same space, the space of true latents $\rvz$. %
Here we summarize our evaluation and provide more details in App.~\ref{app:evaluation}. 

\begin{figure}
    \centering
    \includegraphics[width=0.95\linewidth]{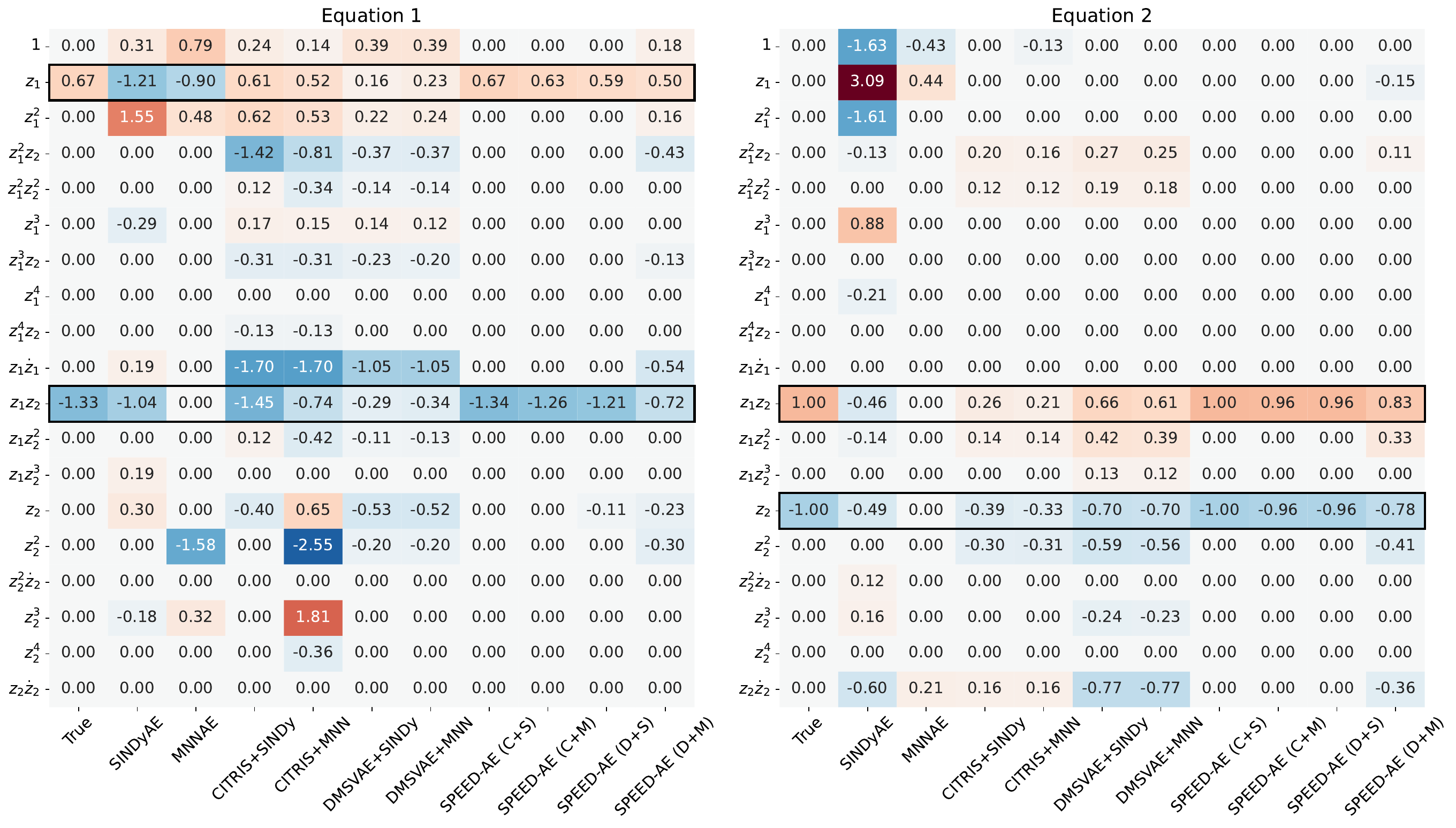}
    \caption{Learned coefficients for Lotka-Volterra for each polynomial of the state variables (each row) and different methods (each column, where the last four are versions of SPEED-AE). The first column shows the ground truth coefficients and the black boxes highlight the non-zero coefficients.}
    \label{fig:prey_predator1_coeffs}
\end{figure}

For trajectory prediction evaluation, we first learn a component-wise MLP $\sigma_i$ mapping each learned variable $\vzspeed_i$ to the ground-truth variable $\rvz_{\pi(i)}$ by choosing a permutation $\pi$ that maximises their correlation. We then compute the latent prediction error $\rvz\text{Err}$ as the $L^2$ distance between the true trajectories $\rvz(t)$ and the predicted trajectories projected to the same space $\sigma(\tilde{\rvz})(t)$, averaging over all timesteps $t$. We then compare the observation prediction error $\rvx\text{Err}$, as the $L^2$ distance between the observations $\rvx(t)$ and predicted observations $\tilde{\rvx}(t)=\dec(\speeddec(\tilde{\rvz}))(t)$, where $\speeddec$ is the SPEED-AE decoder and $\dec$ is the CRL decoder.

For ODE discovery, we first learn $d$ component-wise polynomial maps $\zspeed_i=\text{Poly}_i(z_i)$, then express the recovered ODEs $\dot{\vzspeed}=\tilde{\vf}(\vzspeed)$ via a change of variable through the learned polynomial, as shown in Prop.~\ref{proposition:variable_change}. We compare the resulting coefficients with those of the true ODE using the Sum of Absolute and Squared Errors (CoeffSAE and CoeffSSE). 
Since both the ODEs and the maps are polynomial, the converted ODE remains polynomial, although possibly implicit. In this case, the error increases due to the presence of terms such as $z_i\dot{z_i}$, which are penalised. The best case is when the map is affine, $\zspeed_i=a_iz_i + c_i$, which always leads to an explicit ODE and lower coefficient error. 
Since SINDyAE and MNNAE do not define an assignment between $\vzspeed$ and $\rvz$, we select the permutation with the best Mean Correlation Coefficient before learning the map. 
 
App.~\ref{app:ablations} reports ablations on the CRL latents and \modelnameshort{} architecture. With perfect disentanglement, \modelnameshort{} achieves negligible errors in CorrD and ODE recovery. This indicates that performance from our models is only limited by the disentanglement of the CRL method. Our experiments with polynomial encoders and decoders show that, even with perfect disentanglement and the simplest diffeomorphisms between true and CRL latents, these maps are not flexible or stable enough during training to provide good results, motivating the choice of $\speedenc,\speeddec$ being MLPs.

\begin{figure}
    \centering
    \includegraphics[width=0.85\linewidth]{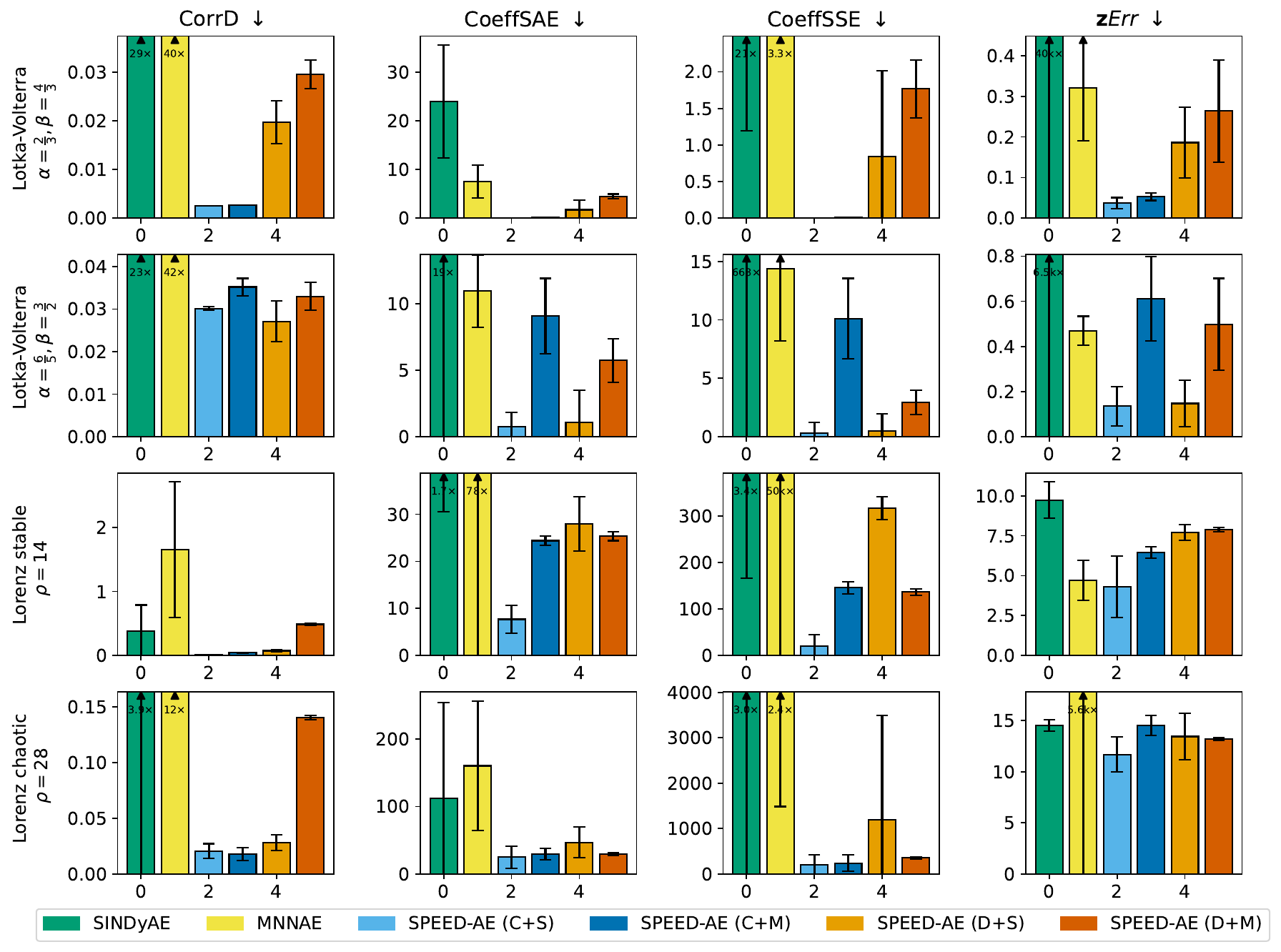}
    \caption{Results for the Lotka-Volterra and Lorenz experiments with mean and standard deviation over $10$ seeds. For all metrics, lower is better. We limit the height of the boxes for better readability, reporting the actual scales of the outlier results inside their bars.}
    \label{fig:results_full}
\end{figure}

\subsection{Lotka-Volterra and Lorenz with polynomial libraries}

We consider two systems with polynomial libraries of order $3$: Lotka-Volterra and Lorenz.
The Lotka-Volterra model \citep{Strogatz2019NonlinearDynamicsChaos} is a system of two first-order ODEs
\begin{equation}\label{eqn:predator_prey}
        \dot{z}_1 =  \alpha z_1-\beta z_1z_2,\hspace{2cm} \dot{z}_2 = -\gamma z_2 + \delta z_1z_2,
\end{equation}
for prey $z_1$ and predator $z_2$ populations. We fix $\delta=\gamma=1$ and consider two cases for $\alpha,\beta$. In the first, we set $\alpha=\frac{2}{3},\beta=\frac{4}{3}$ and train the whole pipeline. In the second, we set $\alpha=\frac{6}{5}, \beta=\frac{3}{2}$ but use the same CRL models from the first case, i.e., we train only the component-wise autoencoders and the equation discovery components of \modelnameshort{}, while the baselines are trained from scratch. The objective is to show that \modelnameshort{} allows us to transfer learned disentangled representations to new dynamics without performance drops. We generate $120$ trajectories of $5000$ steps with $\Delta t=0.01$ and an $80/20/20$ split. The observation $\rvx$ is generated as in the Lorenz setting of \cite{Champion2019DatadrivenDiscoveryCoordinates}. On $[-1,1]$, we evaluate the first four Legendre polynomials on a grid of $128$ points, obtaining $\vu_1,\vu_2,\vu_3,\vu_4\in\mathbb{R}^{128}$, and define 
\begin{equation}\label{eqn:predator_prey_mapping}
    \rvx(t) = \rvu_1 z_1(t) + \rvu_2z_2(t) + \rvu_3z_1^3(t) + \rvu_4z_2^3(t)\in \R^{128}.
\end{equation}
This map from $\mathbb{R}^2$ to its image in $\mathbb{R}^{128}$ is non-linear, smooth, and invertible on its image, since Legendre polynomials are a complete basis in $[-1,1]$ and are therefore linearly independent.
The Lorenz system \citep{Strogatz2019NonlinearDynamicsChaos} is a system of three ODEs describing atmospheric convection:
\begin{equation}\label{eqn:lorenz}
        \dot{z}_1 = \sigma(z_2 - z_1), \quad
        \dot{z}_2 = z_1(\rho - z_3) - z_2, \quad
        \dot{z}_3 = z_1z_2 - \beta z_3,
\end{equation}
where $\sigma, \rho, \beta$ are physical fluid constants. We set $\sigma=10,\beta=\frac{8}{3}$, and consider $\rho=14$ and $\rho=28$, which produce stable and chaotic trajectories, respectively. Following \citet{Champion2019DatadrivenDiscoveryCoordinates}, we generate trajectories of $250$ time steps with $\Delta t = 0.02$.
The observations are generated in the same way as \citet{Champion2019DatadrivenDiscoveryCoordinates} and analogous to the Lotka-Volterra setup. With $3$ variables, we use $6$ Legendre polynomials, obtaining a non-linear, smooth and map between $\mathbb{R}^3$ and $\mathbb{R}^{128}$.

Fig.~\ref{fig:prey_predator1_coeffs} shows the learned coefficients for both equations of Lotka-Volterra for all of the evaluated methods. \modelnameshort{} recovers ODEs that after conversion, are closest to the ground truth, especially when using CITRIS, while the baselines recover more complicated equations with multiple unnecessary terms. In particular, the coefficients of the CRL methods used directly with the equation discovery methods (e.g., CITRIS + SINDy) perform worse than the corresponding versions of our framework (e.g., SPEED-AE (C + S)), showing the effectiveness of \modelnameshort{} and that disentanglement alone is insufficient
The results for Lorenz are in App.~\ref{app:lorenz}, showing similar trends.

Fig.~\ref{fig:results_full} shows that for Lotka-Volterra all \modelnameshort{} variants beat the baselines by a large margin, except for MNNAE on $\rvz$Err. However, MNNAE's strong performance on forecasting comes at the cost of no identifiable latents. For Lorenz, the \textbf{stable} case ($\rho=14$), correlation and ODE recovery results (\Cref{fig:results_full}) show that \modelnameshort{} outperforms both baselines. The slight performance drop of DMSVAE-based variants reflects a weaker disentanglement compared to CITRIS (see App.~\ref{app:lorenz}). The strong forecasting results of MNNAE are again countered by poor correlation and ODE recovery. Under \textbf{chaotic} dynamics ($\rho=28$), \modelnameshort{} remains unmatched in ODE recovery. Forecasting results are similar to the stable case, except that MNNAE fails in more than one seed.

We report the results on the $\rvx$Err metric in App.~\ref{app:add_results}, where we also compare with a Long Short-Term Memory (LSTM) \citep{Hochreiter1997LongShortTermMemory} and a Transformer \citep{Vaswani2017AttentionAllYou}, which learn to predict $\rvx$ directly. These models are effective on short trajectories, such as the Lorenz ones, but collapse on the longer ones from Lotka-Volterra, where recovering the underlying set of variables and ODEs yields stable long-horizon performance with errors up to $20$ times lower.

\subsection{Pendulums and non-polynomial libraries}\label{sec:pendulum}
We test \modelnameshort{} on a second-order system and actual image data. We consider a system of two independent pendulums governed by $\ddot{z}_1 = -\sin(z_1)$ and $\ddot{z}_2 = -\sin(z_2)$, similarly to \citet{Champion2019DatadrivenDiscoveryCoordinates}. Images have resolution $64\times 64$ and separate channels for the two pendulums. We use a library that includes $\sin(z)$, as in SINDyAE, the main baseline for this experiment. %
Given their consistently strong performance, we only use CITRIS for CRL and SINDy for ODE discovery.

Tab.~\ref{tab:double_pendulum_results} shows that \modelnameshort{} reduces forecasting errors by a factor of $3-4$ relative to SINDyAE, which fails to identify and disentangle the two pendulums, instead learning only one variable correlated with a pendulum (see App.~\ref{app:pendulum}). The ODEs recovered by \modelnameshort{} are nearly identical to each other and to the the true ones, reflecting that the pendulums share the same dynamics. Although CITRIS+SINDy forecasts accurately, its equations are less simple and interpretable than those of \modelnameshort{}, providing empirical support for our discussion in Sec.~\ref{sec:identifiability} beyond the polynomial case. %

\begin{table}
    \centering
    \scriptsize
    \caption{Results for the Pendulum experiment for the model.}
        \begin{tabular}{lcccc}
            \toprule
            \textbf{Model} & CorrD &$\mathbf{x}$Err & $\mathbf{z}$Err & Recovered ODE \\
            \midrule
             SINDyAE & $0.431$ & $18.888$ & $1.621$ & $\begin{gathered}
                 \ddot{z}_1 - 0.302\ddot{z}_1z_1 + 1.20\ddot{z}_1z_1^2 + 2.39 \ddot{z}_1\dot{z}_1z_1 =  \\= -0.32 + 0.75z_2 - 0.43 z_1 + 0.17 z_1^3 + 0.06z_1^2 + \ldots \\
                 \ddot{z}_2 + 0.01\ddot{z}_2z_2 + 0.03\ddot{z}_2z_2^2 + 0.05\ddot{z}_2\dot{z}_2z_2 = \\ = 0.03 + 0.18z_1 - 0.07 z_1^3 -0.03z_1^2 +\ldots
             \end{gathered}$ \\
             \midrule
             CITRIS+SINDy & $0.020$ & $6.824$ & $0.239$ & $\begin{gathered}
                 \ddot{z}_1 - 0.15\dot{z}_1^2 + 0.15\ddot{z}_1z_1 + 0.02 \dot{z}_1^2z_1^2 -0.012 \dot{z}_2^2z_2^2 + \ldots =  \\= 0.64 + 11.31 z_1 -1 31z_1^2 + \cos(0.14 \dot{z}_1z_1 - 0.93z_1) + \ldots \\
                 \ddot{z}_2 - 0.16\ddot{z}_2z_2 - 0.16\dot{z}_2^2 + 0.32\dot{z}_2z_2 +\ldots = \\ = 1.48 + 4.03 z_2 - 0.04 z_2^2 + \cos(0.14z_2\dot{z}_2 + 0.84z_2-0.43)\ldots
             \end{gathered}$ \\
             \midrule
             \modelnameshort{} (C+S) & $0.004$ & $6.155$ & $0.177$ & $\begin{gathered}
                 \ddot{z}_1 = - 0.27 z_1 - 0.61 \sin(1.05 z_1) \\
                 \ddot{z}_2 = - 0.27 z_2 - 0.61 \sin(1.07 z_2)
             \end{gathered}$ \\
            \bottomrule
        \end{tabular}
        \label{tab:double_pendulum_results}
\end{table}

%% file: sections/conclusion.tex
\section{Conclusion}\label{sec:conclusion}

We show that neither disentanglement nor sparsity alone is enough to simultaneously learn the true structural variables and their governing equations from high-dimensional data. 
We propose \modelnameshort{}, a pipeline that learns a componentwise transformation of disentangled variables from a any CRL model that is informed by sparsity of the learned equations given a library of functions. \modelnameshort{} achieves the best performance among baselines, especially in terms of correlation with the true latents and the recovered ODEs.  While our methodology is general, as any CRL and ODE recovery methods can be plugged into our pipeline, one of our limitations it that the results depend on the performance of each step. In particular, a poor disentanglement can hinder the final performance, although \modelnameshort{} can drastically improve CRL results even in this case. Moreover, our theoretical analysis is limited by the non-cancellation assumptions and polynomial diffeomorphisms.
We believe the empirical and theoretical results of this work represent a step forward towards learning meaningful and transferable representations of dynamical systems from unstructured data and that these results can be further improved in future works by the community.

%% file: sections/appendix.tex
\section{Proofs of the Theoretical Statements and Discussion}\label{app:proofs}
In this Section, we provide proofs of the theoretical statements and discuss the limits and challenges of the problems in \Cref{sec:identifiability}.

\subsection{Preliminaries}

First, we report a simple proposition that shows that for any ODE $\dot{\rvz} = \vf(\rvz)$ defined on $\rvz$ and any diffeomorphism $h$ that maps $\tilde{\rvz}$ to $\rvz$, we can define a conjugated ODE with the change of variable formula.
Intuitively this means that learning an ODE on our learned variables $\tilde{\rvz}$ does not provide any constraints in terms of the class of functions $h$.

\begin{restatable}[Change of Variable]{proposition}{variablechange}
\label{proposition:variable_change}
    Consider an ODE $\dot{\rvz} = \vf(\rvz)$ defined on $\rvz$, and let $h$ be a diffeomorphism that maps a different set of variables $\tilde{\rvz}$ to $\rvz=h(\tilde{\rvz})$. Then, given a solution $\rvz(t)$ with initial condition $\rvz_0$, the transformed trajectory $\tilde{\rvz} = h^{-1}(\rvz)$ solves the following Cauchy problem:
    \begin{equation}%
    \label{eqn:ode_general_change}
        \dot{\tilde{\rvz}} =
        [\nabla h(\tilde{\rvz})]^{-1}\vf(h(\tilde{\rvz}))
        \quad\quad
        \vzcrl(0) = h^{-1}(\rvz_0)
    \end{equation}
\end{restatable}
\begin{proof}
    It is sufficient to consider the time derivative of $\rvz = h(\vzspeed)$ and apply the chain rule on the right-hand side, obtaining
    \begin{equation}\label{eqn:variable_change_lemma_start}
        \dot{\rvz} = [\nabla h (\vzspeed)]\dot{\vzspeed}.
    \end{equation}
    Then, we can substitute the ODE for $\rvz$ on the left-hand side, obtaining
    \begin{equation}\label{variable_change_lemma_sub}
        \vf(\rvz) = [\nabla h (\vzspeed)]\dot{\vzspeed}.
    \end{equation}
    Finally, we substitute $\rvz=h(\vzspeed)$:
    \begin{equation}
        [\nabla h (\vzspeed)]\dot{\vzspeed} = \vf(h(\vzspeed)).
    \end{equation}
    Since $h$ is a diffeomorphism, $[\nabla h (\vzspeed)]$ is invertible everywhere and thus
    $\dot{\tilde{\rvz}}=[\nabla h(\tilde{\rvz})]^{-1}\vf(h(\tilde{\rvz}))$.
\end{proof}

\subsection{SINDyAE loss cannot guarantee disentanglement for polynomial ODEs}%
\label{proof:sindyae}

We show that following two polynomial ODEs are equivalent in terms of the SINDyAE loss, since they are related by a diffeomorphism and have the same number of terms
\begin{equation}
\vf(\rvz) =
\begin{pmatrix}
  -z_2 + z_1^3 + z_1 z_2^2 \\
  \phantom{-}z_1 + z_1^2 z_2 + z_2^3
\end{pmatrix}
\qquad
\tilde{\vf}(\tilde{\rvz})=
\begin{pmatrix}
  \phantom{-}\tilde{z}_2 + 2\tilde{z}_1^3 + 2\tilde{z}_1 \tilde{z}_2^2 \\
  -\tilde{z}_1 + 2\tilde{z}_1^2 \tilde{z}_2 + 2\tilde{z}_2^3,
\end{pmatrix}
\end{equation}
although their representations are not disentangled (as we see below).

Let 
$
A = \begin{pmatrix} 1 & \phantom{-}1 \\ 1 & -1 \end{pmatrix},
$
consider the diffeomorphism
\[
  \rvz = h(\tilde{\rvz}) \;=\; A\tilde{\rvz}
  \;=\; \bigl( \tilde{z}_1 + \tilde{z}_2,\ \tilde{z}_1 - \tilde{z}_2 \bigr).
\]
Since $h$ is a linear transformation, it holds $\nabla h(\tilde{\rvz}) = A$ for every $\tilde{\rvz}$. Furthermore, since $A^{-1} = \frac{1}{2}A$, we can also compute $[\nabla h(\tilde{z})]^{-1} = \frac{1}{2}A$.
We rewrite slightly the original polynomial ODE as 
\[
\vf(\rvz) =
\begin{pmatrix}
  -z_2 + z_1^3 + z_1 z_2^2 \\
  \phantom{-}z_1 + z_1^2 z_2 + z_2^3
\end{pmatrix} 
= J\rvz + \lvert \rvz \rvert^{2}\rvz,
\]
using matrix 
$J = \begin{pmatrix} 0 & -1 \\ 1 & 0 \end{pmatrix}$ and notation $\lvert \rvz \rvert^{2} = z_1^{2} + z_2^{2}.$

We can write the conjugate ODE, using the change of variable form from \Cref{proposition:variable_change}, as
\begin{align}
  \tilde{\vf}(\tilde{\rvz}) &= [\nabla h(\tilde{z})]^{-1} \vf(h(\tilde{z}))\\
  &= \frac{1}{2}A\left( J A \tilde{\rvz} + \lvert A\tilde{\rvz} \rvert^{2} A\tilde{\rvz} \right)\\
  &=\frac{1}{2}A J A \tilde{\rvz} + \frac{1}{2}A \lvert A\tilde{\rvz} \rvert^{2} A\tilde{\rvz}\\
  &= -J \tilde{\rvz} + \frac{1}{2}A \lvert A\tilde{\rvz} \rvert^{2} A\tilde{\rvz} \\
  &= -J \tilde{\rvz} + \frac{1}{2}A 2 \lvert \tilde{\rvz} \rvert^{2} A\tilde{\rvz}\\
  &= -J \tilde{\rvz} + \lvert \tilde{\rvz} \rvert^{2}A   A\tilde{\rvz}\\
  &= -J \tilde{\rvz} + \lvert \tilde{\rvz} \rvert^{2} 2 I \tilde{\rvz}\\
  &= -J \tilde{\rvz} + 2 \lvert \tilde{\rvz} \rvert^{2} \tilde{\rvz}\\
  &= \begin{pmatrix}
  \phantom{-}\tilde{z}_2 + 2\tilde{z}_1^3 + 2\tilde{z}_1 \tilde{z}_2^2 \\
  -\tilde{z}_1 + 2\tilde{z}_1^2 \tilde{z}_2 + 2\tilde{z}_2^3.
\end{pmatrix}
\end{align}

In particular, even if both ODE systems are polynomial and each component of $h$ is a polynomial function, the SINDy loss cannot guarantee on its own that the learned variables $\tilde{\rvz}$ are disentangled, i.e., they correspond to the ground truth variables $\rvz$ up to permutation $\pi$ and component-wise transformations $h_i$ that only depend on a single $z_{\pi(i)}$. This suggests that the learned equations on these variables might in general be more complicated than the ground truth ones.

\subsection{Identifiability of Latent variables and ODE}\label{app:proofs/identifiability_explicit}

Before proceeding to the proof of the main result, we pose the following assumptions on the polynomial functions involved in our problem.
\begin{assumption}[No cancellations in Powers of Polynomials] \label{ass:non_cancellation_power}
    Let $h_i(\tilde{z}_{\pi(i)}) = \sum_{k\leq p}{\beta_{k}\tilde{z}_{\pi(i)}^k}$ be any of the component-wise polynomial diffeomorphisms that we consider, with degree $p$. Then, the symbolic representation of $h_i(\tilde{z}_{\pi(i)})^\ell$ has the maximum possible number of terms, i.e., no terms cancel out.  %
\end{assumption}

Consider the square of a generic polynomial of degree $d=2$: $(1+\beta_1z+\beta_2z^2)^2$, which expands to \[1 + 2 \beta_1 z + (\beta_1^2 + 2\beta_2) z^2 + 2 \beta_1 \beta_2 z^3 + \beta_2^2 z^4.\] Then, under assumption \ref{ass:non_cancellation_power}, all the monomials appearing in the expansion must have a non-zero coefficient. In other terms, it implies that $\beta_1 \neq 0$, $\beta_2 \neq 0$, and $\beta_1^2+2\beta_2\neq 0$.
This assumption does not say that in every possible polynomial power, every possible term $z^k$ with degree between $0$ and $p\ell$ has a non-zero coefficient. For example, $(1+\beta_2z^2+\beta_4z^4)^2$ expands to a polynomial containing only only even powers of $z$. Here, Assumption \ref{ass:non_cancellation_power} indicates that all even exponents appear. %
We remark that this assumption is not restrictive since, if we assume that the coefficients of $P(z)$ are chosen at random in some interval, having cancellations has zero probability.
Finally, we remark that powers of polynomials are an active research topic in algebra \citep{Zannier2008CompositeLacunaryPolynomials, Schinzel2009NumberTermsPower}. Without assumptions, no bounds on the number of terms in a power of a polynomial would be helpful for our theoretical analysis. In particular, by accurately selecting the coefficients of $P(z)$, one can even find cases where $P(z)^\ell$ has fewer terms than $P(z)$ \citep{Renyi1947MinimalNumberTerms, Coppersmith1991PolynomialsWhosePowers, Abbott2002SparseSquaresPolynomials}.

\begin{assumption}[Variable changes do not annihilate terms]\label{ass:non_cancellation_expansion}
    Consider component-wise polynomial maps $z_i = h(\zspeed_{\pi(i)})$. In the context of \Cref{thm:identifiability_explicit}, with $f_i(z_i) = q_i(z_i)r_i(\rvz)$, when performing the variable change on the system of ODEs, obtaining a new system of equations
    \begin{equation}\label{eqn:speed_single_ode_change}
        \dot{\zspeed}_{\pi(i)} = \frac{q_i(h_i(\zspeed_{\pi(i)}))}{h'_i(\zspeed_{\pi(i)})}r_i(h_1(\zspeed_{\pi(1)}), \ldots, h_d(\zspeed_{\pi(d)})),
    \end{equation}
    all monomial terms that appear in the symbolic expansion of $r_i(h(\vzspeed_{\pi(i)}))$ have non-zero coefficients.
\end{assumption}

Similar remarks apply to Assumption \ref{ass:non_cancellation_expansion}, where cancellation of terms in the expansion of $r_i(h(\rvz))$ requires specific conditions on the coefficients of $r_i$ and $h_j$, which have measure zero when the coefficients are sampled randomly. For instance, the ODE $f(z) = z + z^2 - z^3$ can be decomposed as $f(z)=q(z)r(z)$ where $q(z) = z$ and $r(z) = 1 + z - z^2$. The diffeomorphism $h(\tilde{z}) = \tilde{z} + \alpha$ produces the conjugate ODE $\dot{\tilde{z}} = \left(\alpha + \alpha^2 - \alpha^3\right) + \left(1 + 2\alpha - 3\alpha^2\right)\tilde{z} + \left(1 - 3\alpha\right)\tilde{z}^2 - \tilde{z}^3$, which in turn factorizes as $[h'(\tilde{z})]^{-1}q(h(\tilde{z}))h'(\tilde{z}) r(h(\tilde{z}))$, where
\begin{align}
    q(h(\tilde{z})) &= \tilde{z} + \alpha\\
    h'(\tilde{z}) &= 1\\
    r(h(\tilde{z})) &= \left(1 + \alpha - \alpha^2\right) + \left(1 - 2\alpha\right)\tilde{z} - \tilde{z}^2.
\end{align}
Hence, satisfying Assumption \ref{ass:non_cancellation_expansion} requires that $1+\alpha-\alpha^2\neq0$ and $1 - 2\alpha\neq0$.

Then, before proving our main result, we introduce a technical lemma, which sets a lower bound on the number of terms in the symbolic representation of a power of a polynomial.

\begin{lemma}[Number of terms in a polynomial power]\label{lemma:power_expansion}
    Let $P(z)$ be a polynomial with degree $p$, $P(z) = \sum_{j\leq p}{\beta_jz^j}$ and let $n$ be the number of non-zero coefficients. Then, assuming that no cancellations happen (Assumption \ref{ass:non_cancellation_power}), the number of terms in $P(z)^\ell$ is at least $1+(n-1)\ell$.
\end{lemma}
\begin{proof}
    We assume that the polynomial $P(z)$ has a non-zero constant term $\beta_0$. If not, it is sufficient to factor out the smallest power of $z$ and work with the remaining polynomial, which has the same number of terms.
    Therefore, we prove the statement by induction on the power $\ell$:
    \begin{itemize}
        \item \textbf{Base case.} $\ell=1$. The statement holds as the polynomial has $1+(n-1)\ell=n$ terms.
        \item \textbf{Inductive step} $\ell\implies \ell+1$. Consider the $\ell+1$-th power as the following product:
        \begin{equation}\label{eqn:poly_power_inductive1}
            P(z)^{\ell+1} = P(z)^\ell P(z) = \left(\sum_{j\leq \ell p}\gamma_jz^j\right)\left(\beta_0 + \ldots \beta_pz^p\right),
        \end{equation}
        where the first parentheses is the result of expanding $P(z)^\ell$. We now expand this product by considering two particular terms: one given by the highest-degree term of $\beta_pz^p$ of $P(z)$ multiplied by $P(z)^\ell$ and the other given by the remainder of the difference $P(z)-\beta_pz^p$ multiplied by $\gamma_0$:
        \begin{equation}\label{eqn:poly_power_inductive2}
            \begin{split}
                P(z)^{\ell+1} & = \left(\sum_{j\leq \ell p}\gamma_jz^j\right)\left(\beta_0 + \ldots \beta_pz^p\right)\\
                & = \beta_pz^p\left(\sum_{j\leq \ell p}\gamma_jz^j\right) + \gamma_0\left(\beta_0 + \ldots \beta_{p-1}z^{p-1}\right) + Q(z),
            \end{split}
        \end{equation}
        where $Q(z)$ is the remainder of terms. The first term of this sum involves only terms with degree at least $p$ and has the same number of terms as $P(z)^\ell$, which, by inductive hypothesis, is at least $1+(n-1)\ell$. The second term, instead, has one fewer term than $P(z)$ and maximum degree $p-1$. Assuming no cancellations (Assumption \ref{ass:non_cancellation_power}), the number of terms in \eqref{eqn:poly_power_inductive2} is at least the sum of the number of terms in its two terms. Hence, $P(z)^{\ell+1}$ has at least $1+(n-1)\ell + (n-1) = 1+(n-1)(l+1)$ terms as wanted.
    \end{itemize}
\end{proof}

In the following, we denote monomials through multi-indexes $\ell\in\mathbb{N}^d$ with degree $|\ell| = \sum_{i=1}^d\ell_i$. Therefore, we indicate multi-powers of vectors as $\rvz^{\ell} = z_1^{\ell_1}\cdots z_d^{\ell_d}$. We are now ready to prove our main result, which we restate here.

\identifiabilitythm*
\begin{proof}
    We start from the implicit version of the transformed ODE (Prop. \ref{proposition:variable_change}) for a generic index $i$:
    \begin{equation}\label{eqn:variable_change_implicit}
        [h'_i(\zspeed_{\pi(i)})]\dot{\zspeed}_{\pi(i)} = f_i(h_1(\zspeed_{\pi(1)}), \ldots, h_d(\zspeed_{\pi(d)})).
    \end{equation}
    For the new ODE to be polynomial, it is necessary that
    \begin{equation}
        \dot{\zspeed}_{\pi(i)} = \tilde{f_i}(\vzspeed) = \frac{f_i(h_1(\zspeed_{\pi(1)}), \ldots, h_d(\zspeed_{\pi(d)}))}{h'_i(\zspeed_{\pi(i)})}
    \end{equation}
    is a polynomial and, thus, that $f_i(h_1(\zspeed_{\pi(1)}), \ldots, h_d(\zspeed_{\pi(d)}))$ is divisible by $h'_i(\zspeed_{\pi(i)})$.  Most importantly, if $h'_i(\zeta) = 0$ for some $\zeta\in \mathbb{R}$, i.e., $\zeta$ is a root of $h'_i(\zspeed_{\pi(i)})$, we must have that
    \begin{equation}
        f_i(h_1(\zspeed_{\pi(1)}), \ldots, h_i(\zeta), \ldots, h_d(\zspeed_{\pi(d)})) = 0
    \end{equation}
    for every value of the other $\zspeed_j$, as otherwise $\tilde{f}_i$ would not be a polynomial. Thus, $h_i(\zeta)$ always annihilates $f_i(\rvz)$, and $z_i-h_i(\zeta)$ must divide $f_i$.
    
    We now consider the decomposition $f_i(\rvz) = q_i(z_i)r_i(\rvz)$ where $q_i(z_i)$ has the highest degree.
    In this composition, it must be that $z_i-h_i(\zeta)$ does not divide $r_i(\rvz)$, as otherwise we could write $f_i(\rvz) = q_i(z_i)(z_i -h_i(\zeta))r'_i(\rvz)$ and $q_i$ would not have the highest degree possible. Thus, we must have $h'_i(\zspeed_{\pi(i)})|q_i(h_i(\zspeed_{\pi(i)}))$ and
    \begin{equation}
        \tilde{f}(\vzspeed) = \frac{q_i(h_i(\zspeed_{\pi(i)}))}{h'_i(\zspeed_{\pi(i)})}r_i(h(\vzspeed_{\pi}))
    \end{equation}
    and, for $\tilde{f}$ to be a polynomial, $h'_i(\zspeed_{\pi(i)})$ must divide $q_i(h_i(\zspeed_{\pi(i)}))$, leaving $r_i(h(\vzspeed_{\pi}))$ untouched. Thus, we can decompose the conjugated ODE as $\tilde{f}(\vzspeed) = \tilde{q}_i(\zspeed_{\pi(i)})\tilde{r}_i(\vzspeed)$, where
    \begin{align}
    \tilde{q}_i(\zspeed_{\pi(i)})&= \frac{q_i(h_i(\zspeed_{\pi(i)}))}{h'_i(\zspeed_{\pi(i)})}\\
    \tilde{r}_i(\vzspeed) &= r_i(h(\vzspeed_{\pi})).
    \end{align}

    Thus, the sparsity in $\tilde{f}_{\pi(i)}$ is strongly related to the sparsity of $\tilde{r}_{\pi(i)}(\vzspeed) = r_i(h(\vzspeed_{\pi}))$. In general, $r_i(\rvz)$ has the form $r_i(\rvz) = \sum_{|\ell|\leq \text{deg}(f_i)}{\alpha_{i,\ell}z_1^{\ell_1}\cdots z_d^{\ell_d}}$.
    We start by expanding $r_i(h_1(\zspeed_{\pi(1)}), \ldots, h_d(\zspeed_{\pi(d)}))$, where we recall $h_j(\zspeed)$ are all polynomials as well:
    \begin{equation}\label{eqn:variable_change_single_expanded}
        \begin{split}
           \tilde{r}_i(\tilde{\rvz}) = r_i(h_1(\zspeed_{\pi(1)}), \ldots, h_d(\zspeed_{\pi(d)})) & = \sum_{|\ell|\leq \text{deg}(\vf)}{\alpha_{i,\ell}\,  h_1(\zspeed_{\pi(1)})^{\ell_1}\cdots h_d(\zspeed_{\pi(d)})^{\ell_d}}\\
            & = \sum_{|\ell|\leq \text{deg}(\vf)}{\alpha_{i,\ell} \prod_{j=1}^{d}{\left( \sum_{k\leq p_j}{\beta_{jk}\zspeed_{\pi(j)}^{k}}\right)^{\ell_j}}}.
        \end{split}
    \end{equation}

    We now focus on a single term related to the coefficient $\alpha_{i,\ell}$. Our objective is to calculate the total number of terms that will appear once we expand the calculations. To do so, we need to evaluate the number of terms resulting from taking the $\ell_j$ power of the polynomials. If we let $n_j$ be the number of terms in $h_j(\zspeed_{\pi(j)})$, Lemma \ref{lemma:power_expansion} tells us that $h_j(\zspeed_{\pi(j)})^{\ell_j}$ has at least $1+(n_j-1)\ell_j$ terms. Furthermore, since they all involve a single variable and they are all different, the symbolic expansion of the product $h_1(\zspeed_{\pi(1)})^{\ell_1}\cdots h_d(\zspeed_{\pi(d)})^{\ell_d}$ will have the maximum number of terms possible, which is at least $\prod_{j=1}^d (1+(n_j-1)\ell_j)$. In the end, a single term $\alpha_{i,\ell}$ now accounts for at least $\prod_{j=1}^d (1+(n_j-1)\ell_j)$ new ones.

    Additionally, we notice that the highest-degree monomial among these is given by $\prod_{j=1}^{d}\zspeed_{\pi(j)}^{p_j\ell_j}$. This is because this term comes from the product of the highest-degree monomial in each diffeomorphism $h_j(\zspeed_{\pi(j)})$, which has degree $p_j$ as defined above, and the exponent $\ell_j$. We call this monomial term the \emph{representative} of the multi-index $\ell$.

    Given two multi-indices $\ell^{(1)}, \ell^{(2)}$, their representative monomials are distinct (in terms of degrees of their variables). This is because if they were to be equal, we would have that, for each $j$, $\ell^{(1)}_jp_j=\ell^{(2)}_jp_j$ and, thus, $\ell^{(1)}_j=\ell^{(2)}_j$. In particular, since the representatives are all distinct, $\tilde{r}_{\pi(i)}$ has at least one term for each monomial in $r_i$, that is $N(\tilde{r}_{\pi(i)}) \geq N(r_i)$.

    Clearly, we can have that different multi-indices $\ell^{(1)}, \ell^{(2)}, \ldots, \ell^{(K)}$ all produce a monomial term related to some multi-index $\tilde{\ell}$ in the new variables $\tilde{\rvz}$. However, Assumption \ref{ass:non_cancellation_expansion} ensures that this term remains active and the contributions coming from the different $\ell^{(j)}$ do not cancel out. Hence, minimizing the number of terms in the transformed ODE requires minimizing $\prod_{k=1}^d (1+(n_j-1)l_j)$, as all the terms will survive.
    
    Since the multi-indices $\ell$, the corresponding powers $\ell_j$, and the coefficients $\alpha_j$ are all given by the original ODE and are thus fixed, this leaves us with a single option for minimizing the number of terms after the change of variable: minimizing $n_j$, i.e., the number of terms in $h_j(\zspeed_{\pi(j)}) = \sum_{k=1}^{p_j}{\beta_{jk}\zspeed_{\pi(j)}^{k}}$. In fact, if $n_j=1$, each power will involve exactly one term, and the ODE keeps the same number of terms. This means that, to have the minimal number of terms, we must have that all $\beta_{jk}$ are zero except one, i.e., $h_j$ are all monomial transformations $h_j(\zspeed_{\pi(j)})=a_jz_{\pi(j)}^{p_j}$.

    For the special case, we now look again at the structure of
    \begin{equation}
        \tilde{f}(\vzspeed) = \frac{q_i(h_i(\zspeed_{\pi(i)}))}{h'_i(\zspeed_{\pi(i)})}r_i(h(\vzspeed_{\pi}))
    \end{equation}
    and we ask when $h'_i(\zspeed_{\pi(i)})$ divides $q_i(h_i(\zspeed_{\pi(i)}))$. As stated earlier, whenever $\zeta$ is a root of $h_i(\zspeed_{\pi(i)})$, $h_i(\zeta)$ annihilates $f_i$ and $z_i-h_i(\zeta)$ must divide $f_i$.
    This necessary condition limits which polynomial transformations $h_i$ are admissible given the form of $f_i(z)$. In particular,
    if $q_i(z_i) = d_i$ is a constant, it means that $h'_i(\zspeed_{\pi(i)})$ must be a constant as well: $h'_i(\vzspeed) = a_i\in\R$. Thus, $h_i(\zspeed_{\pi(i)}) = a_i\zspeed_{\pi(i)}+b_i$, and $h_i$ is necessarily an affine function.
\end{proof}

\paragraph{Counterexample to the general case.} \Cref{thm:identifiability_explicit} shows that an affine map $h_i$ can be recovered if $q_i(z_i)$ is a constant. One can also show that if $q_i$ has a single root $d_i(z_i-c_i)^{k_i}$, then $h_i(\zspeed_i) = a_i(\zspeed_{\pi(i)}-b_i)^k + c_i$. This comes from the fact that, for $h_i'(\zspeed_{\pi(i)})$ to divide $q_i(h_i(\zspeed_{\pi(i)}))$, $h'$ must have a single root, with arbitraty multiplicity.

One might wonder what happens if $q_i(z_i)$ is in a very general form. Unfortunately, we found no satisfactory property for $h_i(\zspeed_i)$ in this case, and we provide a counterexample in this sense.

Consider $\pi(i)=i$ and a system of polynomial ODEs with the $i$-th component being
\begin{equation}
    \dot{z}_i = f_i(z) = (z_i+2)(z_i-2)r_i(\rvz).
\end{equation}
For $h_i(\zspeed_i)$ to be admissible, the only condition is that $h'(\zspeed_i)$ must divide $(h_i(\zspeed_i)+2)(h_i(\zspeed_i)-2)$. For example, the polynomial $h(\zspeed_i)= \zspeed_i^3 -3\zspeed_i$ fulfills these requests, and if restricted to an appropriate domain, constitutes a diffeomorphism. Even more complex cases can be constructed as, the divisibility condition provides a way to construct examples.

\paragraph{On the assumption that $h$ is polynomial.} In our theoretical results, we assume that $h$ is polynomial, which might not be true, especially in the experimental settings. We remark that, in compact spaces, any function can be approximated up to arbitrary precision by a polynomial with sufficient degree.

We investigated whether this hypothesis is necessary for our theoretical analysis. In particular, we tried to understand if, given $\dot{z}=f(z)$ polynomial and $\dot{\zspeed} = \tilde{f}(\zspeed)$ polynomial as well, it is necessary that $h(\zspeed)$ is polynomial. Unfortunately, this is not the case, which makes the assumption necessary. Consider
\begin{equation}\label{eqn:thm_counter_start_ode}
    \dot{z} = f(z) = 1,
\end{equation}
Now, consider the mapping $z=h(\zspeed) = \arctan(\zspeed)$. By applying the change of variable, we have
\begin{equation}
    \frac{1}{1+\zspeed^2}\dot{\zspeed} = 1
\end{equation}
and, consequently, $\dot{\zspeed} = 1+\zspeed^2$. Both ODEs are polynomial, but the diffeomorphisms $h:\R\rightarrow (-\pi/2,\pi/2)$ and its inverse are not polynomial.

\section{\modelnameshort{} Implementation Details}\label{app:implementation}
In this Section, we describe the implementation of \modelnameshort{} for the SINDy and MNN loss cases.

In each case, the architecture of \modelnameshort{} consists of a set of $d$ component-wise encoders $\speedenc_i$ and $d$ component-wise decoders $\speeddec_i$, as well as the set of ODE coefficients $\mathbf{\Xi}$ for the library $\mathbf{\Theta}(\vzspeed)$. The encoders map each variable $\vzcrl_i$ coming from the disentangled CRL representation to a new encoding $\zspeed_i = \speedenc_i(\zcrl_i)$, while the decoders aim to recover the CRL variables $\zcrl_i$. The ODE coefficients $\Xi$ define the learned ODE as $\dot{\vzspeed} = \mathbf{\Theta}(\vzspeed)^\top\Xi$, similar to \citep{Brunton2016DiscoveringGoverningEquations, Champion2019DatadrivenDiscoveryCoordinates, Chen2024ScalableMechanisticNeural}. Below, we describe the two specific modes we used to train \modelnameshort{}.
\paragraph{SINDy-based loss: \modelnameshort{} (X+S)} We use a similar approach to \citet{Champion2019DatadrivenDiscoveryCoordinates}, where the key idea is to learn a SINDy-like loss $\left\|\dot{\vzspeed} - \mathbf{\Theta}(\vzspeed)^\top\Xi\right\|_2$ both on the learned representations $\vzspeed$ and on the original variables $\vzcrl$. To calculate this loss, we need to obtain the reference values for the derivatives $\dot{\vzspeed}$, which can be obtained via the chain rule $\dot{\vzspeed} = (\nabla_{\vzcrl}\speedenc)(\dot{\vzcrl})$. In contrast, the derivatives of the CRL-encoded trajectories $\dot{\vzcrl}$ are calculated via finite differences with the pysindy library \citep{DeSilva2020PySINDyPythonPackage, Kaptanoglu2022PySINDyComprehensivePython}.
The complete loss reads as:
\begin{equation}\label{eqn:speed_sindy_loss}
    \begin{split}
        \mathcal{L}_{\text{\modelnameshort{}(X+S)}} = & \sum_{i=1}^d\left\|\speeddec_i\left(\speedenc_i(\zcrl_i)\right) - \zcrl_i\right\|_2^2 +\beta_{\rvz} \left\|(\nabla_{\vzcrl}\speedenc)(\dot{\vzcrl}) - \Theta(\vzspeed)^\top\Xi\right\|_2^2 + \\
        & +\beta_{\rvx}\left\|\dot{\vzcrl} - (\nabla_{\vzspeed}\speeddec)(\Theta(\vzspeed)^\top\Xi)\right\|_2^2 + \beta_1\left\|\Xi\right\|_{1},
    \end{split}
\end{equation}
where the names of the weights are chosen for consistency with those of MNNAE and SINDyAE.
The first term imposes a low reconstruction error on the autoencoders, the second is the SINDy loss on $\dot{\vzspeed}$ with the push-forwarded derivatives from $\dot{\vzcrl}$, and the third term calculates the SINDy loss on the original variables $\vzcrl$, based on the push-forward of the SINDy derivatives through the decoder $\dot{\vzcrl}_{\text{SINDy}} = (\nabla_{\vzspeed}\speeddec)(\Theta(\vzspeed)^\top\Xi)$. The last term induces sparsity with the $L^1$ norm, while sequential thresholding is applied every $100$ epochs.

\paragraph{MNN-based loss: \modelnameshort{} (X+M)} In case of the MNN \citep{Pervez2024MechanisticNeuralNetworks, Chen2024ScalableMechanisticNeural}, derivatives are not necessary for both $\vzspeed$ and $\vzcrl$. Instead, the MNN calculates the trajectory from the initial condition and the ODE given by $\dot{\vzspeed} = \mathbf{\Theta}(\vzspeed)^\top\Xi$. The complete loss reads at
\begin{equation}\label{eqn:speed_mnn_loss}
    \begin{split}
        \mathcal{L}_{\text{\modelnameshort{}(X+M)}} = & \sum_{i=1}^d\left\|\speeddec_i\left(\speedenc_i(\zcrl_i)\right) - \zcrl_i\right\|_2^2 + \beta_{\rvz}\left\|\vzspeed - \text{MNN}(\vzspeed(0))\right\|_2^2 + \\ & + \beta_{\rvx}\left\|\vzcrl - \speeddec(\text{MNN}(\speedenc(\vzcrl(0)))\right\|_2^2 + \beta_1\|\Xi\|_1,
    \end{split}
\end{equation}
where the names of the weights are chosen for consistency with those of MNNAE and SINDyAE.
Similar to the above case, the third term ensures that the AE and the MNN model work together to reconstruct the trajectories from the initial condition, avoiding the collapse of the learned representation. As in the SINDy-based loss, the last term induces sparsity with the $L^1$ norm, while sequential thresholding is applied every $100$ epochs. For this model, the training dataset is generated by unfolding trajectories of a fixed length from the available trajectories, i.e., considering almost every point in the dataset as an initial condition.

\section{Baselines}\label{app:baselines}
Here we provide the details on the definitions and implementation of the employed baselines.

\paragraph{SINDyAE.} SINDyAutoencoder was introduced in \citet{Champion2019DatadrivenDiscoveryCoordinates} to recover latent coordinates and governing equations in a single step. It consists of
\begin{itemize}
    \item an Autoencoder, i.e., an encoder $\vzspeed = \sindyenc(\rvz)$ and a decoder $\tilde{\rvx}=\sindydec(\vzspeed)$, both of which are implemented with MLPs.
    \item An underlying SINDy model, consisting of a library of functions $\Theta(\vzspeed)\in\R^{L\times d}$ and the learnable linear coefficients $\Xi\in\R^L$ that define the ODE as $\Theta(\vzspeed)^\top \Xi$.
\end{itemize}
The model is trained with the following loss:
\begin{equation}\label{eqn:sindyae_loss_app}
    \begin{split}
        \mathcal{L}_{\text{SINDyAE}} = & \left\|\rvx - \sindyenc\sindydec(\rvx)\right\|_2^2 + \beta_{\rvx}\left\|\dot{\rvx} - (\nabla_{\vzspeed} \sindydec)(\Theta(\vzspeed)^\top\Xi)\right\|_2^2 + \\ & + \beta_{\rvz}\left\|(\nabla_\rvx\sindyenc)(\dot{\rvx}) - \Theta(\vzspeed)^\top\Xi\right\|_2^2 + \beta_{1}\left\|\Xi\right\|_{1},
    \end{split}
\end{equation}
where each term represents, in order:
\begin{itemize}
    \item The reconstruction of the autoencoder.
    \item The loss on the derivative of $\rvx$, which is compared to the pushforward of the predicted derivative $\dot{\vzspeed}=\Theta(\vzspeed)^\top \Xi$ through the decoder $\sindydec$ with the chain rule.
    \item The loss on the derivative of $\rvz$, which compares the predicted derivative $\dot{\vzspeed}=\Theta(\vzspeed)^\top \Xi$ with the pushforward of the ground truth one $\dot{\rvx}$ through the encoder $\sindyenc$.
    \item The sparsity loss on $\Xi$.
\end{itemize}
To actually sparsify the equations, every $500$ epochs, the coefficients in $\Xi$ below a certain threshold are zeroed and masked. We use the original implementation available on SINDyAutoencoder's GitHub. %

\paragraph{MNNAE.} We employ the original MNN model from \cite{Pervez2024MechanisticNeuralNetworks, Chen2024ScalableMechanisticNeural} and combine it with an Autoencoder, i.e., similarly to the PDE solving and Discovery of Physical Parameters experiments in \citet{Pervez2024MechanisticNeuralNetworks}. MNNAE consists of:
\begin{itemize}
    \item an Autoencoder, i.e., an encoder $\vzspeed = \mnnenc(\rvz)$ and a decoder $\tilde{\rvx}=\mnndec(\vzspeed)$, both of which are implemented with MLPs.
    \item An underlying Mechanistic Neural Network, as detailed in the Lorenz experiment from \cite{Pervez2024MechanisticNeuralNetworks}. A mechanistic encoder inputs a trajectory and predicts the coefficients and parameters of the underlying ODE. The fast ODE solver of \citet{Chen2024ScalableMechanisticNeural} solves the ODE as a constrained optimization problem and predicts the trajectory from the initial condition, which is then matched to the data.
\end{itemize}
As a preprocessing step, the model unrolls trajectories with $t=50$ steps per batch, resulting in batches of shape $(B,t,D)$. The model is trained with the following loss:
\begin{equation}\label{eqn:mnnae_loss_app}
    \begin{split}
        \mathcal{L}_{\text{MNNAE}} = & \left\|\rvx - \mnnenc\mnndec(\rvx)\right\|_2^2 + \beta_{\rvz}\left\|\vzspeed - \text{MNN}(\vzspeed(0))\right\|_2^2 + \\ & + \beta_{\rvx}\left\|\rvx - \mnndec(\text{MNN}(\vzspeed(0))\right\|_2^2 + \beta_1\|\Xi\|_1,
    \end{split}
\end{equation}
where each term represents, in order:
\begin{itemize}
    \item The reconstruction of the autoencoder.
    \item The loss between the true encoded trajectory $\vzspeed$ and the one predicted by the MNN from the initial condition of the batch $\text{MNN}(\vzspeed(0))$.
    \item The loss between the true high-dimensional trajectory $\rvx$ and the decoded trajectory predicted by the MNN: $\mnndec(\text{MNN}(\vzspeed(0)))$.
    \item The sparsity loss on $\Xi$.
\end{itemize}
Here, MNN$(\vzspeed(0)$ represents the solution of the ODE defined by the mechanistic encoder for the same number of steps as the batch, i.e., 50.
Similar to SINDyAE, the equation is sparsified by thresholding coefficients below $0.1$ every $100$ epochs, for $10$ times in total, as in \cite{Pervez2024MechanisticNeuralNetworks}. We use the original implementation available on GitHub. %

\paragraph{Long Short-Term Memory.} We implement an LSTM \cite{Hochreiter1997LongShortTermMemory} with linear readout layer and dropout with $p=0.1$. To train the LSTM, we roll out the time series in batches of $50$ steps. The LSTM receives as input the first $n_{\text{in}}$ steps of the $\rvx$ batch, and predicts the following $50 - n_{\text{in}}$ steps autoregressively. We select the best model on the validation trajectories, choosing between $\{1,2,3\}$ layers and $n_{\text{in}} \in \{1, 10, 25, 49\}$ with the same batch size of \modelnameshort{} models. LSTMs are trained for $500$ epochs with the ADAM optimizer and learning rate $0.001$.

\paragraph{Transformer.} We use a Transformer encoder model \citep{Vaswani2017AttentionAllYou} with linear readout and positional encoding to predict the next steps of $\rvx$ from previous ones. Data is processed similarly to the LSTM model, where we roll out batches of trajectories of length $50$. The Transformer receives as input the first $n_{\text{in}}$ steps and autoregressively predicts the rest of the trajectory. We train the Transformer with the ADAM optimizer and cosine warmup scheduling for $500$ epochs. We select the best model from $\{2,3\}$ layers, $\{4,8\}$ attention heads, and $n_{\text{in}}\in \{1,10,25,49\}$.

\section{Details on the CRL methodologies}\label{app:crl_details}
In this Section, we provide additional details on the CRL methods we used for the disentanglement step of our training pipeline. We remark that our work is not specifically tied to these methodologies, as any model capable of disentangling causal variables can be adopted.
\subsection{CITRIS}\label{app:crl_details/citris}
CITRIS \citep{Lippe2022CITRISCausalIdentifiability} is a CRL approach designed for temporally intervened sequences, under the assumption that each true variable can be intervened upon in the data. While the method is designed for a more general setting with possibly multi-dimensional causal variables and minimal causal variables (i.e., the smallest set of variables that can be identified), in this work, we always assume that each variable can be intervened on at any time $t$. CITRIS needs to know the variable that has been intervened on at time $t$, that is, the intervention target $I_t \in \{0,1,\ldots,d\}$, where $0$ means no intervention. For our purposes, we utilize CITRIS-NF, which consists of a pre-trained autoencoder that maps the images to a lower-dimensional space, followed by an invertible Normalizing Flow model that actually performs the disentanglement, as described in the original implementation by \citet{Lippe2022CITRISCausalIdentifiability}, which outputs the final latents $\vzcrl^t=\vzcrl(t)$. The model then consists of:
\begin{itemize}
    \item a learnable assignment function $\psi:\{1,\ldots, M\}\rightarrow \{0, \ldots, K\} $, where, in our case, both $M$ (the number of latent variables) and $K$ (the number of causal blocks, together with an additional dimension for non-intervened ones) are set to $d$, the number of true latents,
    \item an autoregressive transition prior $p_{t1}(\hat{\rvz}^t|\hat{\rvz}^{t-1},I^{t})$, which is factorized as $p(\hat{\rvz}^t|\hat{\rvz}^{t-1},I^{t}) = \prod_{i=1}^{d}p(\hat{z}^t_{\psi_i}|\hat{\rvz}^{t-1},I^{t}_i)$, so that each varible $\zh_i$ depends only on the past state and on its own intervention target,
    \item a target classifier, which is a small additional network trained to predict the intervention target $I^{t+1}$ given the past and current state, which adds stability to the training procedure.
\end{itemize}
The overall loss is given by an ELBO calculated from the reconstruction error and the KL divergence on the transition prior, weighted by a hyperparameter $\beta_{t1}$, and the classification error from the target classifier, weighted by a hyperparameter $\beta_{\text{classifier}}$.

\subsection{DMSVAE}\label{app:crl_details/dvms}
Disentanglement via Mechanism Sparsity \citep{Lachapelle2022DisentanglementMechanismSparsity, Lachapelle2026NonparametricPartialDisentanglement} assumes that the causal graph between variables at two time steps is sparse, i.e., the graph $G^z$ that represents the influences between $z_i(t-1)$ and $z_j(t)$ is sparse, or that the causal graphs $G^a$ between some observed auxiliary variables $a_i(t)$ and the state $z_j(t)$ is sparse. Since in most dynamical systems the first assumption is very uncommon (usually, each ODE $\dot{z}_i = f_i(\rvz(t))$ involves almost every variable), we exploit the second one in a similar setting to CITRIS. In fact, we use the interventional targets $I(t)$ as observed auxiliary variables $a(t)$, which leads to a causal graph $G^a$ being equal to the identity and, therefore, sparse. The identifiability comes from imposing such sparsity in the learned transition model. The full model (DMSVAE) consists of:
\begin{itemize}
    \item a variational autoencoder $\crlenc, \crldec$,
    \item a transition model $p(\zcrl_i(t)|\vzcrl(t-1),I(t))$, parametrized by $d$ MLPs,
    \item a causal graph $\hat{G}^a$, learned by binary masks applied on the transition model inputs, gating which variables it can use to predict the next one.
\end{itemize}
The model is trained with an ELBO loss based on the reconstruction error and the transition prior, weighted by $\beta_{t1}$, and a sparsity loss (the $L^0$ norm), on the causal graph $\hat{G}^a$, weighted by $\beta_{\text{sparse}}$.

\section{Details on the Experimental Setup}\label{app:experimental_details}
In this Section, we describe the complete setup of our experiments, from data generation to model selection and training, and provide the specific parameters and hyperparameters of each experiment.

\paragraph{Data Generation.} In the experiments, we consider two kinds of trajectories: standard and interventional ones. To generate the first, we consider the ground truth ODE $\dot{\rvz} = \vf(\rvz)$ and sample $n_\text{train}, n_{\text{val}}, n_{\text{test}}$ initial conditions respectively for the training, validation, and test datasets. Each trajectory is calculated from the ODE and initial condition up to time $T$ with time step $\Delta t$ using the Runge-Kutta 45 solver from the SciPy library \citep{Virtanen2020SciPy10Fundamental}. 

To generate the interventional trajectories, we instead iterate through the time dimension and, at each time step, choose a variable to intervene on with the same probability $p_{\text{intervention}}=0.01$, or choose not to intervene with the remaining probability. If a variable $i$ is selected, its next value $z_i(t+\Delta t)$ is randomly sampled, while the others evolve normally from $\rvz(t)$. If no variable is selected, all of them evolve normally. This is the same procedure used in CITRIS \citep{Lippe2022CITRISCausalIdentifiability}, which we also adapt for DMSVAE \citep{Lachapelle2022DisentanglementMechanismSparsity}. Starting from a randomly sampled initial condition, we generate $s_{\text{train}}, s_{\text{val}}, s_{\text{test}}$ time steps for the training, validation, and test interventional trajectories, respectively. We provide the specific parameters of data generation in Table \ref{tab:data_generation}. To ensure that the standard trajectories are in-distribution for the CRL models, we first generate the interventional ones, and then generate standard ones until we have enough withing within the bounds of the interventional data. Finally, all trajectories $\rvz(t)$ are mapped to the higher-dimensional ones $\rvx(t)$ via an experiment-specific mapping, which we describe in the following appendices.
\begin{table}
    \centering
    \caption{Data generation parameters.}
        \begin{tabular}{lccc}
            \toprule
            \textbf{Parameter} & Lotka-Volterra & Lorenz & Pendulum \\
            \midrule
            $n_{\text{train}}$ & $20$ & $1024$ & $100$ \\
            $n_{\text{val}}$ & $5$ & $256$ & $25$ \\
            $n_{\text{test}}$ & $5$ & $256$ & $25$ \\
            $T$ & $50$ & $5$ & $10$ \\
            $\Delta t$ & $0.01$ & $0.02$ & $0.02$ \\
            num. steps & $5000$ & $250$ & $500$ \\
            Sampled Variables & $z_1,z_2$ & $z_1,z_2,z_3$ & $z_1,\dot{z_1}, z_2, \dot{z_2}$\\
            Sample distribution & Uniform $[0,3]^2$ & See \citet{Champion2019DatadrivenDiscoveryCoordinates} & Uniform $[-2,2]^4$\\
            $s_{\text{train}}$ & $100000$ & $256000$ & $50000$ \\
            $s_{\text{val}}$ & $25000$ & $64000$ & $12500$ \\
            $s_{\text{test}}$ & $10000$ & $25600$ & $5000$ \\
            \bottomrule
        \end{tabular}
        \label{tab:data_generation}
\end{table}

\paragraph{CRL model training.}
To obtain disentangled representations $\vzcrl$, we use CITRIS \citep{Lippe2022CITRISCausalIdentifiability} or DMSVAE \citep{Lachapelle2022DisentanglementMechanismSparsity}, which we discuss in Appendix \ref{app:crl_details}. After training, CRL models have learned an encoder $\vzcrl = \psi_{\text{CRL}}^{\text{enc}}(\rvx)$ as well as a permutation $\pi$ that assigns the true latent $z_i$ to the corresponding found latent $\zh_{\pi(i)}$. In general, such an optimal permutation can be found via the best mean correlation coefficient across permutations. On validation trajectories, the $R^2$ between each $z_i$ and the corresponding $\zh_{\pi(i)}$ is calculated, and the average across the variables is taken. We select the model with the highest statistic from a simple grid search on hyperparameters with values in $\{0.001, 0.01, 0.1, 1.0, 10.0\}$. We report the best combinations in Table \ref{tab:crl_models_params}. Details on the architectures for each experiment are provided in the following Appendices. The setup for both models is the same as in \citet{Lippe2022CITRISCausalIdentifiability, Lippe2023BISCUITCausalRepresentation}.

We train the CRL methods on every dataset, except for Lotka-Volterra with $\alpha=\frac{6}{5}, \beta=\frac{3}{2}$, where we used the one from the other Lotka-Volterra experiment to show representation transferability.
\begin{table}
    \centering
    \caption{CRL model hyperparameters. Lotka-Volterra with $\alpha=\frac{6}{5}, \beta=\frac{3}{2}$ values are not reported, as in this experiment, we use the same pre-trained CRL models coming from $\alpha = \frac{2}{3}, \beta=\frac{4}{3}$.}
    \begin{adjustbox}{width=0.9\linewidth}
        \begin{tabular}{lccccc}
            \toprule
            \multicolumn{6}{c}{\textbf{CITRIS}}\\
            {Hyperparameter} & Lotka-Volterra & Lotka-Volterra & Lorenz (Stable) & Lorenz (Chaotic) & Pendulums \\
             & $\alpha = \frac{2}{3}, \beta=\frac{4}{3}$ & $\alpha=\frac{6}{5}, \beta=\frac{3}{2}$ & $\rho=14$ & $\rho=28$ \\
            \midrule
            $\beta_{\text{t1}}$ & $10$ & / & $0.01$ & $0.1$ & $0.01$ \\
            $\beta_{\text{classifier}}$ & $1$ & / & $0.001$ & $0.01$ & $0.001$\\
            \midrule
            \midrule
            \multicolumn{6}{c}{\textbf{DMSVAE}}\\
            {Hyperparameter} & Lotka-Volterra & Lotka-Volterra & Lorenz (Stable) & Lorenz (Chaotic) & Pendulums \\
             & $\alpha = \frac{2}{3}, \beta=\frac{4}{3}$ & $\alpha=\frac{6}{5}, \beta=\frac{3}{2}$ & $\rho=14$ & $\rho=28$ \\
            \midrule
            $\beta_{\text{t1}}$ & $1.0$ & / & $1.0$ & $1.0$ &  / \\
            $\beta_{\text{sparse}}$ & $0.01$ & / & $0.01$ & $0.01$ & / \\
            \bottomrule
        \end{tabular}
        \label{tab:crl_models_params}
    \end{adjustbox}
\end{table}

\paragraph{\modelnameshort{} training} Once the CRL model is trained, we consider the high-dimensional standard trajectories $\rvx$ we generated (with no interventions), and encode them through the CRL model $\vzcrl = \psi_{\text{CRL}}^{\text{enc}}(\rvx)$. Then, we permute the latents following the learned assignments $\pi$, so that we can assume $\zh_i$ corresponds to $z_i$ with $\pi(i)=i$. Then, for the SINDy-based versions of \modelnameshort{}, we calculate the derivatives of these trajectories using second-order finite differences. For \modelnameshort{} models with the MNN-based loss, we instead unroll trajectories with $t=50$ steps per batch, similar to \cite{Pervez2024MechanisticNeuralNetworks}, obtaining batches with shape $(B,t,d)$.
Then, model selection is performed on the hyperparameters of each version of \modelnameshort{}: $\beta_{\rvz}$ and $\beta_{\rvx}$ are chosen from $\{10^{-5}, 10^{-4}, 0.001, 0.01, 0.1, 1.0, 10.0\}$, while $\beta_1$ is chosen from $\{0.001, 0.0001, 0.00001\}$. Final values are reported in \Cref{tab:speedae_model_params}. To select the best model, we consider forecasting performance from the initial condition on validation trajectories. In particular, we first encode the validation trajectories through the CRL model $\vzcrl = \psi_{\text{CRL}}^{\text{enc}}(\rvx)$ and the \modelnameshort{} one $\zspeed_i(0)=\phi_i^{\text{enc}}(\zh_i(0))$. Then,  the trajectory is predicted using the Runge-Kutta 45 solver on the learned ODE $\dot{\vzspeed} = \Theta(\vzspeed)^\top\Xi$. Finally, the whole trajectories are decoded via the component-wise decoders $\phi_i^{\text{dec}}$ and compared to the original values of $\vzcrl$. Table \ref{tab:speedae_model_params} reports the best combinations for each \modelnameshort{} model, which depend on the CRL model as well. In each case, the component-wise encoders and decoders are MLPs with $2$ layers of $64$ units each. The models are trained with the ADAM optimizer \citep{Kingma2017AdamMethodStochastic} with a learning rate of $0.0005$ for the SINDy-based loss and $0.0001$ for the MNN-based one. The models are trained for $1000$ epochs with sequential thresholding: every $100$ epochs, coefficients $\Xi_i$ below $0.1$ are zeroed and masked. The batch size is chosen so that each epoch consists of around $100$ steps and is reported in \Cref{tab:speedae_model_params} for each experiment.
\begin{table}
    \centering
    \caption{\modelnameshort{} models hyperparameters.}
    \begin{adjustbox}{width=0.9\linewidth}
        \begin{tabular}{lccccc}
            \toprule
            \multicolumn{6}{c}{\textbf{\modelnameshort{} (CITRIS + SINDy-based loss)}}\\
            {Hyperparameter} & Lotka-Volterra & Lotka-Volterra & Lorenz (Stable) & Lorenz (Chaotic) & Pendulums \\
             & $\alpha = \frac{2}{3}, \beta=\frac{4}{3}$ & $\alpha=\frac{6}{5}, \beta=\frac{3}{2}$ & $\rho=14$ & $\rho=28$ \\
            \midrule
            $\beta_{\rvx}$ & $0.01$ & $0.01$ & $0.01$ & $0.01$ & $0.01$ \\
            $\beta_{\rvz}$ & $0.001$ & $0.0001$ & $0.001$ & $0.0001$ & $0.0001$\\
            $\beta_{1}$ & $0.0001$ & $0.0001$ & $0.00001$ & $0.0001$ & $0.0001$\\
            batch\_size & 4096 & 4096 & 2048 & 2048 & 512 \\
            \midrule
            \midrule
            \multicolumn{6}{c}{\textbf{\modelnameshort{} (DMSVAE + SINDy-based loss)}}\\
            {Hyperparameter} & Lotka-Volterra & Lotka-Volterra & Lorenz (Stable) & Lorenz (Chaotic) & Pendulums \\
             & $\alpha = \frac{2}{3}, \beta=\frac{4}{3}$ & $\alpha=\frac{6}{5}, \beta=\frac{3}{2}$ & $\rho=14$ & $\rho=28$ \\
            \midrule
            $\beta_{\rvx}$ & $0.01$ & $0.01$ & $0.001$ & $0.001$ & / \\
            $\beta_{\rvz}$ & $0.001$ & $0.0$ & $0.01$ & $0.001$ & / \\
            $\beta_{1}$ & $0.00001$ & $0.00001$ & $0.001$ & $0.00001$ & / \\
            batch\_size & 4096 & 4096 & 2048 & 2048 &/ \\
            \midrule
            \midrule
            \multicolumn{6}{c}{\textbf{\modelnameshort{} (CITRIS + MNN-based loss)}}\\
            {Hyperparameter} & Lotka-Volterra & Lotka-Volterra & Lorenz (Stable) & Lorenz (Chaotic) & Pendulums \\
             & $\alpha = \frac{2}{3}, \beta=\frac{4}{3}$ & $\alpha=\frac{6}{5}, \beta=\frac{3}{2}$ & $\rho=14$ & $\rho=28$ \\
            \midrule
            $\beta_{\rvx}$ & $0.01$ & $0.1$ & $0.01$ & $0.01$ & / \\
            $\beta_{\rvz}$ & $10.0$ & $10.0$ & $10.0$ & $1.0$ & / \\
            $\beta_{1}$ & $0.001$ & $0.001$ & $0.001$ & $0.0001$ & /\\
            batch\_size & 4096 & 4096 & 2048 & 2048 & / \\
            \midrule
            \midrule
            \multicolumn{6}{c}{\textbf{\modelnameshort{} (DMSVAE + MNN-based loss)}}\\
            {Hyperparameter} & Lotka-Volterra & Lotka-Volterra & Lorenz (Stable) & Lorenz (Chaotic) & Pendulums \\
             & $\alpha = \frac{2}{3}, \beta=\frac{4}{3}$ & $\alpha=\frac{6}{5}, \beta=\frac{3}{2}$ & $\rho=14$ & $\rho=28$ \\
            \midrule
            $\beta_{\rvx}$ & $0.01$ & $0.01$ & $0.01$ & $0.1$ & /\\
            $\beta_{\rvz}$ & $10.0$ & $0.1$ & $1.0$ & $1.0$ & /\\
            $\beta_{1}$ & $0.00001$ & $0.00001$ & $0.00001$ & $0.00001$ & /\\
            batch\_size & 4096 & 4096 & 2048 & 2048 & /\\
            \bottomrule
        \end{tabular}
        \label{tab:speedae_model_params}
    \end{adjustbox}
\end{table}

\paragraph{SINDyAE baseline.} The SINDyAutoencoder model implementation is taken directly from \citet{Champion2019DatadrivenDiscoveryCoordinates}. Training follows the same procedure as in their original work. The loss weights are chosen in $\{10^{-5}, 10^{-4}, 0.001, 0.01, 0.1, 1.0, 10.0\}$ for the Lotka-Volterra and Pendulums experiments, while we use the authors' parameters for Lorenz, since the standard trajectory generation is the same. To select the best model, we first encode the validation trajectories $\rvx(t)$ in the learned representation space $\vzspeed(t)$. We calculate a prediction of these trajectories from the initial condition $\vzspeed(0)$ and the learned ODE, obtaining $\vzspeed_{\text{pred}}(t)$, and we decode into the predicted $\rvx_{\text{pred}}(t)$. Finally we calculate the $L^2$ error between these $\rvx_{\text{pred}}$ and $\rvx_{\text{true}}$. This measures how well the learned ODEs describe the encoded data. Models consist of MLPs with the same architecture as \modelnameshort{} for direct comparability. The model is trained with the ADAM optimizer with learning rate $0.0001$ (as in \cite{Champion2019DatadrivenDiscoveryCoordinates}), and the same batch size as \modelnameshort{}. Models are trained for $5000$ epochs with thresholding every $500$, and $1000$ epochs of refinement (same as \cite{Champion2019DatadrivenDiscoveryCoordinates}).

\paragraph{MNNAE baseline.} The implementation of the MNN is the one of \cite{Pervez2024MechanisticNeuralNetworks, Chen2024ScalableMechanisticNeural}, where the underlying MNN is coupled with an Autoencoder, as detailed in \Cref{app:baselines}. The loss weights are chosen among $\{0.001, 0.01, 0.1, 1.0, 10.0\}$ for $\beta_{\rvz},\beta{\rvz}$ and $\{0.001, 0.0001, 0.00001\}$ for $\beta_1$. The model selection procedure follows that of SINDyAE. The architecture is the same as \modelnameshort{} for direct comparability. Finally, MNNAE is trained with the ADAM optimizer and learning rate $0.0005$, with batch size, number of epochs, and sequential thresholding as those of \modelnameshort{}.

\paragraph{Evaluation.} The final evaluation is performed on the test trajectories following the evaluation pipeline in Appendix \ref{app:evaluation}.

\subsection{Hardware and Computational Costs}\label{app:computational}
For all experiments, we use a machine with NVIDIA H100 GPUs, and each model has access to one GPU and at most 2 CPUs. We report the training time for the Lotka-Volterra experiment, for reference, in \Cref{tab:computational_times}. Thus, the full \modelnameshort{} pipeline (including disentanglement), requires approximately the same or less time than the other two baselines (SINDyAE and MNNAE).

\begin{table}
    \centering
    \caption{Training time for each model on the Lotka-Volterra experiment.}
        \begin{tabular}{lc}
            \toprule
            \textbf{Model} & Training Time (mins) \\
            \midrule
            LSTM & $\approx 45$ \\
            Transformer & $\approx 62$\\
            \midrule
            SINDyAE &  $\approx 100$ \\
            MNNAE &  $\approx 191$ \\
            \midrule
            Autoencoder (CITRIS-NF) & $\approx 22$\\
            CITRIS-NF & $\approx 52$\\
            DMSVAE & $\approx 41$ \\
            \modelnameshort{} SINDy-based & $\approx 57$ \\
            \modelnameshort{} MNN-based & $\approx 39$ \\
            \bottomrule
        \end{tabular}
        \label{tab:computational_times}
\end{table}

\section{Evaluation Pipeline}\label{app:evaluation}
In this Section, we detail the two protocols used for evaluating the performance of each model in terms of ODE discovery and trajectory prediction. A detailed graphical representation is provided in Figure \ref{fig:evaluation_pipeline}, which we use in the next paragraphs to describe the protocols. The green boxes represent the true latent variables $\rvz$ and the ground truth ODEs $\dot{\rvz}=\vf(\rvz)$, which are kept unknown during training and used only during evaluation. Models have access only to the images/high-dimensional data $\rvx(t)$ in the top-left corner, while the blue boxes on the right represent their outputs, namely the encodings $\vzspeed(t)$ and the recovered ODEs in terms of $\vzspeed$.

Each model learns a specific set of latent variables, e.g. $\vzspeed_{\text{SINDyAE}}, \vzspeed_{\text{\modelnameshort{}(D+S)}}$, which live in different spaces, each with its own ODE, making them not directly comparable. To evaluate them on a common ground, our approach consists of selecting a representative latent space, where the encoded trajectories and recovered ODEs are converted. Since we want to stay unbiased towards which methodology should be preferred, we make the most natural choice possible, which consists of using the true latent space of $\rvz$ as a common evaluation space. Since the true latent variables $\rvz$ are used only during this evaluation, the models will have no access to them before this phase.

\paragraph{Evaluation protocol: Correlation Discrepancy.} We measure how well the encoded data $\vzspeed$ is correlated to the true latents $\rvz$ by comparing the true Pearson's correlation coefficient matrix $\rho_{ij} = \frac{\text{cov}(z_i,z_j)}{\sigma_{z_i}\sigma_{z_j}}$ and its cross version between true and found latents $\tilde{\rho}_{ij} = \frac{\text{cov}(z_i,\zspeed_j)}{\sigma_{z_i}\sigma_{\zspeed_j}}$. Once these matrices are calculated on the train or test trajectories, we compare the two matrices by averaging the absolute error (Correlation Discrepancy) $\text{CorrD} = \frac{1}{N^2}\sum_{i,j=1}^{N}|\rho_{ij}-\tilde{\rho}_{ij}|$.

\paragraph{Evaluation protocol: time-series prediction.} With the second set of metrics, we evaluate how well the models and, consequently, the learned ODEs, can predict the long-term behavior of the dynamical systems from just their initial conditions. Hence, we let each model encode the given initial condition $\rvx(0)$ and evolve it via the dynamics it has learned, obtaining the full trajectory $\vzspeed_{\text{pred}}(t)$. One natural metric to calculate is the error in the final prediction of $\rvx(t)$, obtained by decoding the predicted trajectory in the latent space onto the higher-dimensional space. However, to gain insight into the learned latents and their dynamics, we also want to evaluate them directly against the true ones $\rvz(t)$. As above, we need to evaluate the models on a reference space, that is, the space of true latents $\rvz$. However, here we do not need to stick to polynomial mappings and transformations, as we instead care solely about the predictive capabilities. Hence, we learn an oracle MLP for each model that converts the model latents $\vzspeed_{\text{pred}}(t)$ into the true ones. Then, the time-averaged $L^2$ distance
\begin{equation}
    \frac{1}{n_{\text{steps}}} \sum_{i=1}^{n_{\text{steps}}}\|\text{MLP}(\vzspeed(t_i)) - \rvz(t_i)\|^2_2
\end{equation}
or any other metric between true and converted latents can be calculated. While CRL models provide an assignment between true and found latents, for SINDyAE and MNNAE we consider the permutation given by the best Mean Correlation Coefficient, as described above.
\begin{figure}
    \centering
    \includegraphics[width=\linewidth]{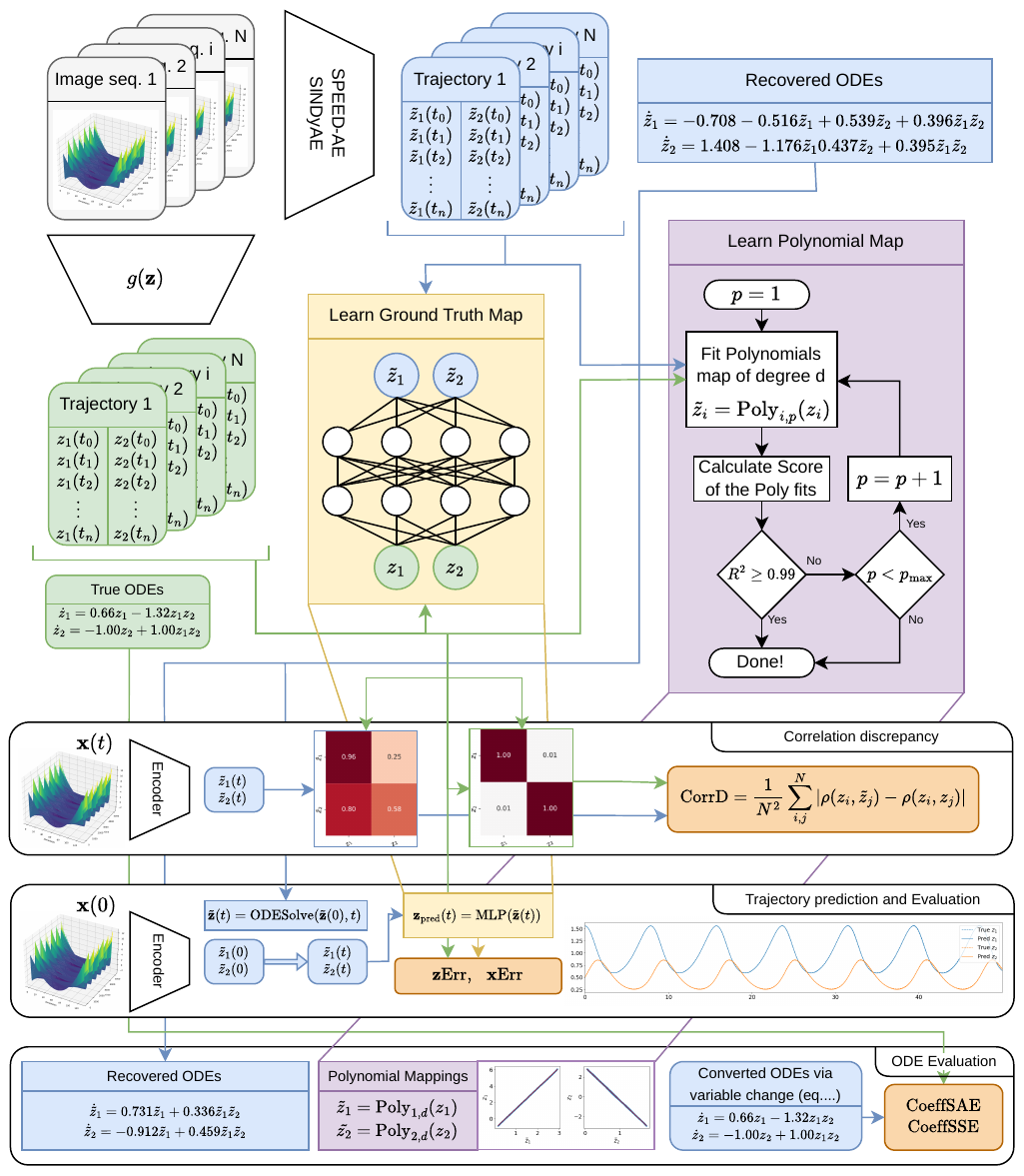}
    \caption{Evaluation pipeline of our experiments.}
    \label{fig:evaluation_pipeline}
\end{figure}

\paragraph{Evaluation protocol: ODE discovery evaluation}
Now that we have set the reference space for evaluation, we define the exact protocol to transform the discovered ODEs from the latent space of the model $\mathcal{Z}_{\text{model}}$ to the true latent space $\mathcal{Z}$. We work with the simplifications used previously, and hence we assume that the two spaces are separated by a component-wise polynomial transformation, that is $\zspeed_i = \text{Poly}_{i,\text{model}}(z_{i,\text{model}})$. This is helpful to convert the ODEs between one space and the other, preserving a polynomial representation. Furthermore, on compact spaces, polynomial functions are a complete basis and can therefore approximate any map between learned and true latents. This polynomial transformation $\zspeed_i = \text{Poly}_{i,\text{model}}(z_{i,\text{model}})$ is calculated via linear regression on polynomial features starting from degree $1$ and up to degree $3$, which we found to be enough in all cases, stopping when the score is higher than $0.99$ (purple box). This greedy approach ensures that the learned transformation is both accurate and simple, using higher-degree polynomials only if necessary. Otherwise, the transformed ODE would present more terms. Finally, the found ODE function $\dot{\vzspeed} = \tilde{f}(\vzspeed)$  (blue box in the bottom rectangle) is transformed via a change of variable with the fitted polynomial mapping (purple box), obtaining an ODE written in terms of the true latent variables $\dot{\rvz} = f_{\text{model}}(\rvz)$ (blue box). Since we always considered a polynomial library and polynomial transformations, this new ODE $f_{\text{model}}(\rvz)$ will necessarily be polynomial, possibly in implicit form. Before calculating the error, we remove coefficients below the $0.001$ threshold, which do not influece the results and provide a cleaner ODE representation. This ODE can be easily compared to the true one $f(\rvz)$ by looking at their coefficients $\Xi$ and $\Xi_{\text{model}}$ when expressed in library form. Practically, we consider the sum of absolute $\|\Xi-\Xi_{\text{model}}\|_1$ and squared $\|\Xi-\Xi_{\text{model}}\|_2^2$ values of their difference as a measure of their distance. In case an ODE is implicit, we also consider all terms $z^\ell\dot{z}_i$ in the library. Since these terms do not appear in the true ODE, we penalize models that find an ODE equivalent to an implicit one. Our theoretical discussion shows that implicit terms appear if and only if the map is monomial or polynomial with degree greater than one instead of linear, and we want to penalize modes that cannot find such simple maps and reward those that do. This way, we reward those modes that are more interpretable and achieve a better identifiability result.

SINDyAE and MNNAE present a particular challenge within this framework, as they do not provide a valid variable assignment between real and found latents. Furthermore, disentanglement is not guaranteed. To solve this issue, we do the following: we first calculate the Mean Correlation Coefficient (MCC) among all possible permutations, hence finding the best variable assignment $z_i = \zspeed_{\pi(i)}$. Then, we proceed as above from the polynomial mapping step.

\section{Ablations}\label{app:ablations}
\paragraph{Independence from the CRL method.} We isolate the main component of \modelnameshort{}, that is, the component-wise autoencoder coupled with sparse ODE recovery, to show that it is independent of the CRL method used. To do so, we assume an oracle CRL method that achieves perfect disentanglement and identifies the true latents up to permutation and component-wise diffeomorphism. To implement this, we first choose random coefficients $\alpha_{i,j}\in (0,1)$ with $i=1,\ldots,d$ and $j=0,1,3$, defining the polynomials
\begin{equation}\label{eqn:oracle_crl_poly}
    P_i(z_i) = \alpha_{i,0} + \alpha_{i,1}(z_i+1) + \alpha_{i,3}(z_i+1)^3,
\end{equation}
and the oracle CRL variables as $\zcrl_i = P_i(z_i)$. We use $(z_i+1)$ to (i) offset the variables from their original values, and (ii) have a polynomial mapping that also has a $z_i^2$ term. Given the form of \Cref{eqn:oracle_crl_poly}, their derivatives are always positive because they include only even-power terms of $(z_i+1)$ with positive coefficients, guaranteeing monotonicity and thus invertibility.

We consider the Lotka-Volterra experiment in \Cref{sec:experiments}. We calculate the $\vzcrl$ variables directly from the true latents $\rvz$ during data preparation. The rest of the experiment proceeds as described in \Cref{app:experimental_details}, starting with \modelnameshort{} training on the latents $\vzcrl$. We remark that \modelnameshort{} never has access to the true latents $\rvz$.

\begin{figure}
    \centering
    \includegraphics[width=\linewidth]{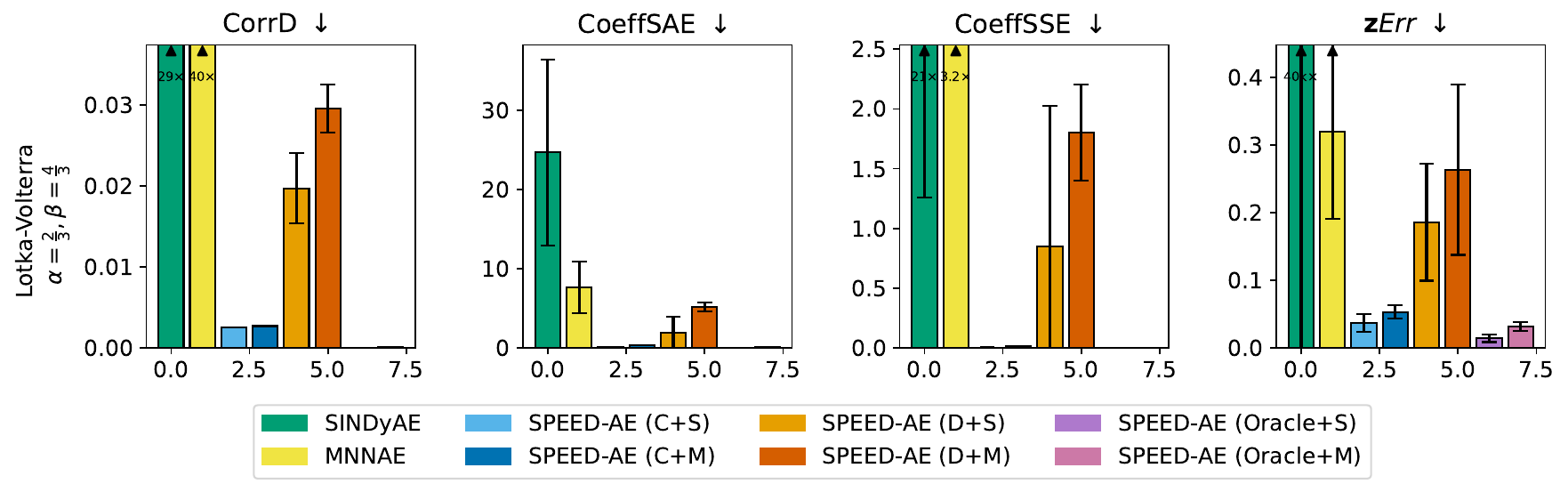}
    \caption{Experimental results for the ablation on the independence from CRL method. Lotka-Volterra dataset with an additional version of \modelnameshort{} with an oracle CRL model that achieves perfect disentanglement and SINDy-based loss.}
    \label{fig:results_lv_synth}
\end{figure}
Results are reported in \Cref{fig:results_lv_synth}. \modelnameshort{} with the SINDy-based loss (Oracle+S) and MNN-based one (Oracle+M) achieve the lowest errors across every metric. Thus, \modelnameshort{} can virtually achieve almost perfect results if the disentanglement is good enough. We notice that the MNN-based model has slightly worse performance, which is also reflected in the real results, where SINDy-based combinations of \modelnameshort{} are often better.

\paragraph{Polynomial diffeomorphisms.} 
Inspired by our discussion in \Cref{sec:identifiability} and \Cref{app:proofs}, we first experimented with a fully polynomial setting. In particular, we ran synthetic experiments similar to the ones above, with perfect disentanglement and identifiability up to polynomial transformation. The only difference is in the \modelnameshort{} encoders and decoders $\speedenc,\speeddec$, which we implemented as polynomial functions. In such a case, we would have the guarantee that the new latents $z_i$ are a polynomial transformation of the found ones, as $z_i = \crldiffeo_i \circ \speeddec_i(\zspeed_{\pi(i)}) $, where both $\speeddec_i$ and $z_i=\crldiffeo(\zcrl_{\pi(i)})$, the diffeomorphism between $z_i$ and $\zcrl_i$, are both polynomial.

Unfortunately, the results were not satisfactory. Using polynomial functions as diffeomorphisms leads to strong instabilities in the model, and we observed either blow-ups in the solutions or flat dynamics. We also tried to simplify the setting by considering two simple oracle CRL variables: $\hat{z}_i = a_i\sqrt[3]{z_i}$, for which the polynomial encoder would only need to learn to use the cubic power term, and $\hat{z}_i = a_iz_i$, i.e., the CRL variables are already an affine transformation of the true ones. Even in these cases, we were not able to effectively train the models, as the polynomial functions make it difficult to train the autoencoder and the dynamics simultaneously. We report some plots in \Cref{fig:ablations_poly}. As a consequence, we use MLPs to implement the encoder and decoders.

\begin{figure}
    \centering
    \begin{subfigure}{0.49\linewidth}
        \centering
        \includegraphics[width=\linewidth]{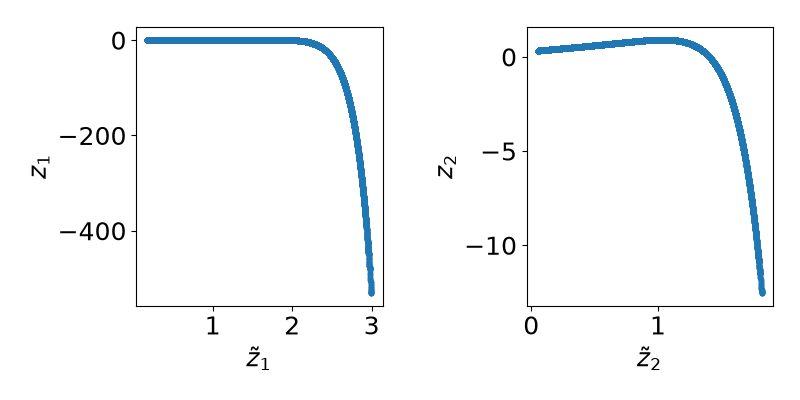}
        \caption{}
        \label{fig:oracle_mono}
    \end{subfigure}
    \begin{subfigure}{0.49\linewidth}
        \centering
        \includegraphics[width=\linewidth]{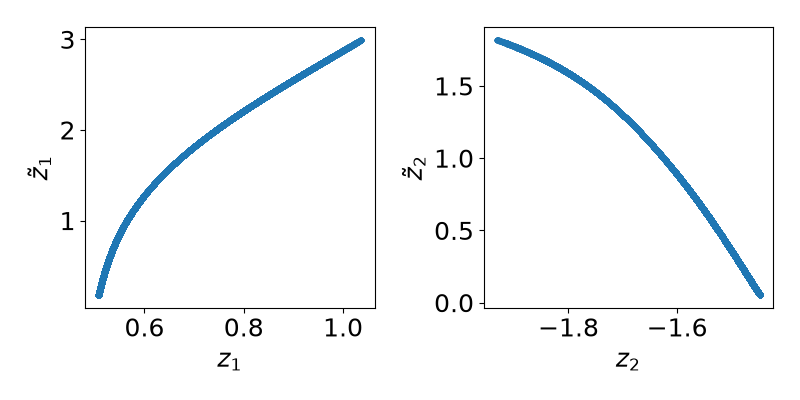}
        \caption{}
        \label{fig:oracle_gen}
    \end{subfigure}
    \caption{Results on the synthetic experiments with perfect disentanglement, polynomial identifiability in the CRL latents, and \modelnameshort{} encoder and decoders implemented as polynomials. (a) Some of the runs show strong instabilities, while others (b) remain stable but with flat dynamics. In general, the polynomial functions are hard to optimize simultaneously to the dynamics.}
    \label{fig:ablations_poly}
\end{figure}

\newpage
\section{Additional Experimental Results}\label{app:add_results}
We report the uncut version of \Cref{fig:results_full} in \Cref{fig:results_full_uncut}. In many plots, especially those related to the Correlation Discrepancy and errors in the Recovered ODE coefficients, SINDyAE and MNNAE are out of scale compared with all combinations of \modelnameshort{}.

\paragraph{Forecasting error on $\rvx$.} We report the results on the $\rvx$Err metric in \Cref{tab:x_baselines}, with the additional baselines of LSTM and Transformer. Traditional ML baselines perform well on the short trajectories of the Lorenz experiment ($250$ steps), even in the chaotic case. However, on the longer trajectories of the Lotka-Volterra experiment ($5000$ steps), the predictions collapse (see \Cref{fig:lv1_lstm_traj} and \Cref{fig:lv1_transformer_traj}) and the errors are almost $20$ times higher than \modelnameshort{}. In general, MNNAE performs well across experiments, although it fails on the chaotic case of Lorenz. However, as already discussed in \Cref{sec:experiments}, this performance comes at the cost of no identifiability of the latents.
\begin{figure}
    \centering
    \includegraphics[width=\linewidth]{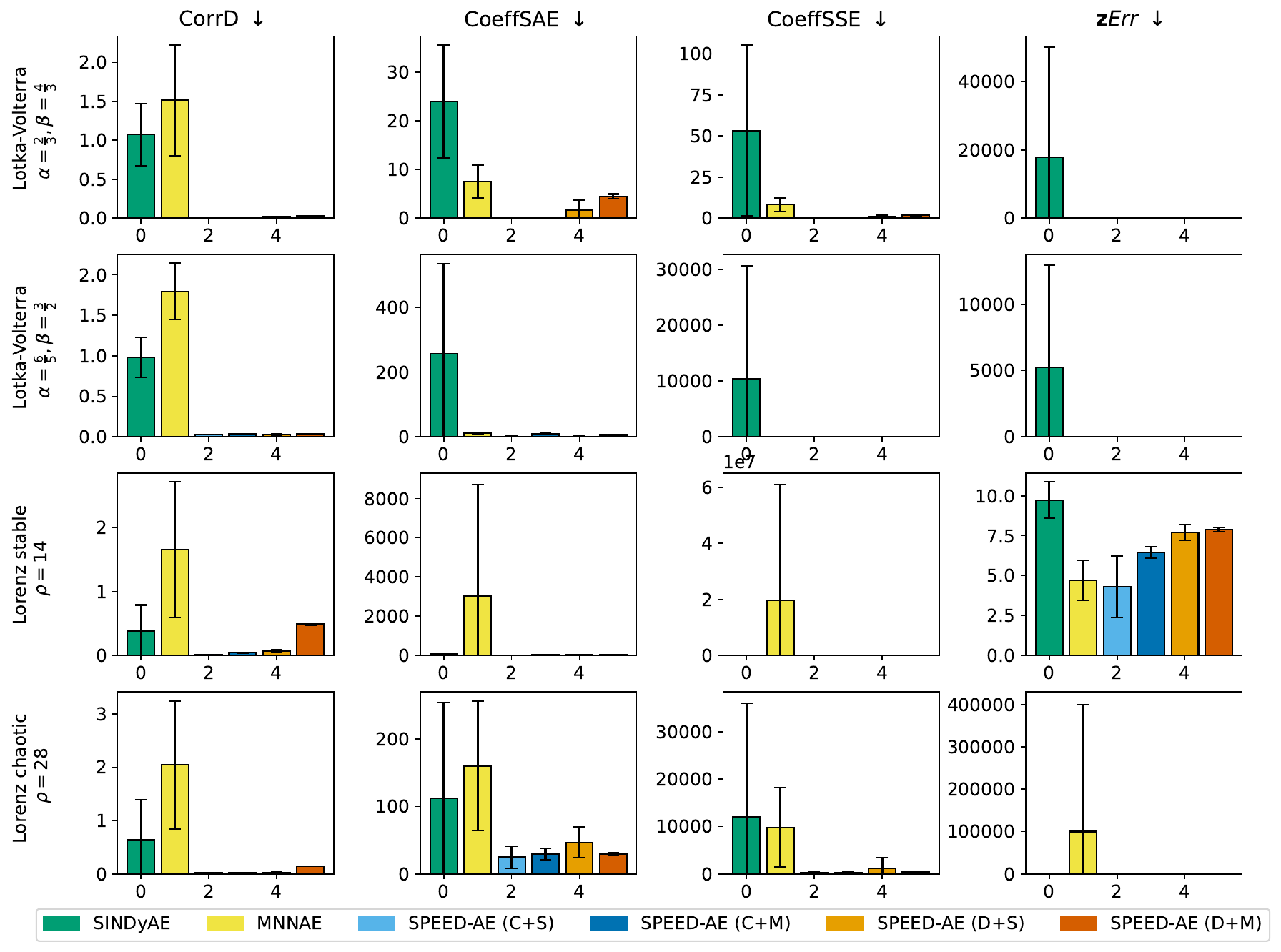}
    \caption{Experimental results for the Lotka-Volterra and the Lorenz experiments.}
    \label{fig:results_full_uncut}
\end{figure}

In the following results, we always permute the recovered latents to match the true ones so that, for the ease of readability, $\pi(i)=i$. We remark that CRL latents identify the true ones up to permutation.

\begin{table}
    \footnotesize
    \centering
    \caption{$\rvx$Err: mean and standard deviation across $10$ seeds.}
        \begin{tabular}{lcccc}
            \toprule
            \textbf{Model} & Lotka-Volterra & Lotka-Volterra & Lorenz (Stable) & Lorenz (Chaotic) \\
             & $\alpha = \frac{2}{3}, \beta=\frac{4}{3}$ & $\alpha=\frac{6}{5}, \beta=\frac{3}{2}$ & $\rho=14$ & $\rho=28$ \\
             
            \midrule
            LSTM & 13.860$_{\pm 2.609}$ & 29.313$_{\pm 14.719}$ & 0.157$_{\pm 0.056}$ & 1.624$_{\pm 0.210}$ \\
            Transformer & 16.001$_{\pm 2.076}$ & 25.999$_{\pm 6.919}$ & 0.422$_{\pm 0.076}$ & 1.902$_{\pm 0.210}$ \\
            \midrule
            SINDyAE & 23.075$_{\pm 11.611}$ & 17.610$_{\pm 2.898}$ & 1.988$_{\pm 0.323}$ & 3.137$_{\pm 0.120}$ \\
            MNNAE & 1.257$_{\pm 0.240}$ & 0.930$_{\pm 0.290}$ & 0.448$_{\pm 0.102}$ & $\approx10^5$ (failed) \\
            \midrule
            \modelnameshort{} (C+S) & 1.279$_{\pm 0.483}$ & 3.999$_{\pm 2.390}$ & 0.851$_{\pm 0.384}$ & 2.411$_{\pm 0.379}$ \\
            \modelnameshort{} (C+M) & 1.552$_{\pm 0.211}$ & 16.041$_{\pm 4.801}$ & 1.256$_{\pm 0.066}$ & 2.984$_{\pm 0.215}$ \\
            \modelnameshort{} (D+S) & 4.669$_{\pm 2.074}$ & 4.881$_{\pm 2.363}$ & 1.399$_{\pm 0.036}$ & 2.736$_{\pm 0.509}$ \\
            \modelnameshort{} (D+M) & 7.239$_{\pm 3.442}$ & 13.686$_{\pm 5.811}$ & 1.619$_{\pm 0.039}$ & 2.684$_{\pm 0.260}$ \\
            \bottomrule
        \end{tabular}
        \label{tab:x_baselines}
\end{table}

\begin{figure}[p]
    \centering
    \begin{subfigure}{\linewidth}
        \centering
        \includegraphics[width=0.85\linewidth]{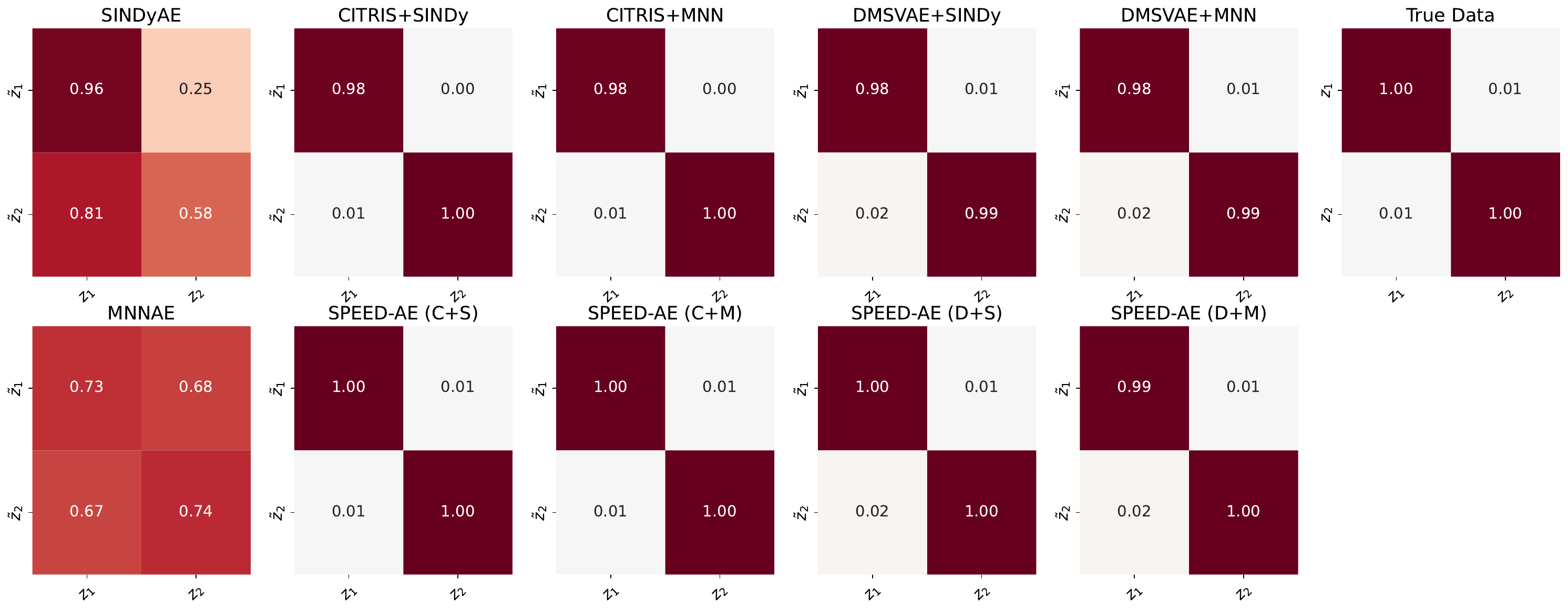}
        \caption{Lotka-Volterra with $\alpha=\frac{2}{3}, \beta=\frac{4}{3}$}
        \label{fig:lv1_cross_pearson}
    \end{subfigure}
    \begin{subfigure}{\linewidth}
        \centering
        \includegraphics[width=0.85\linewidth]{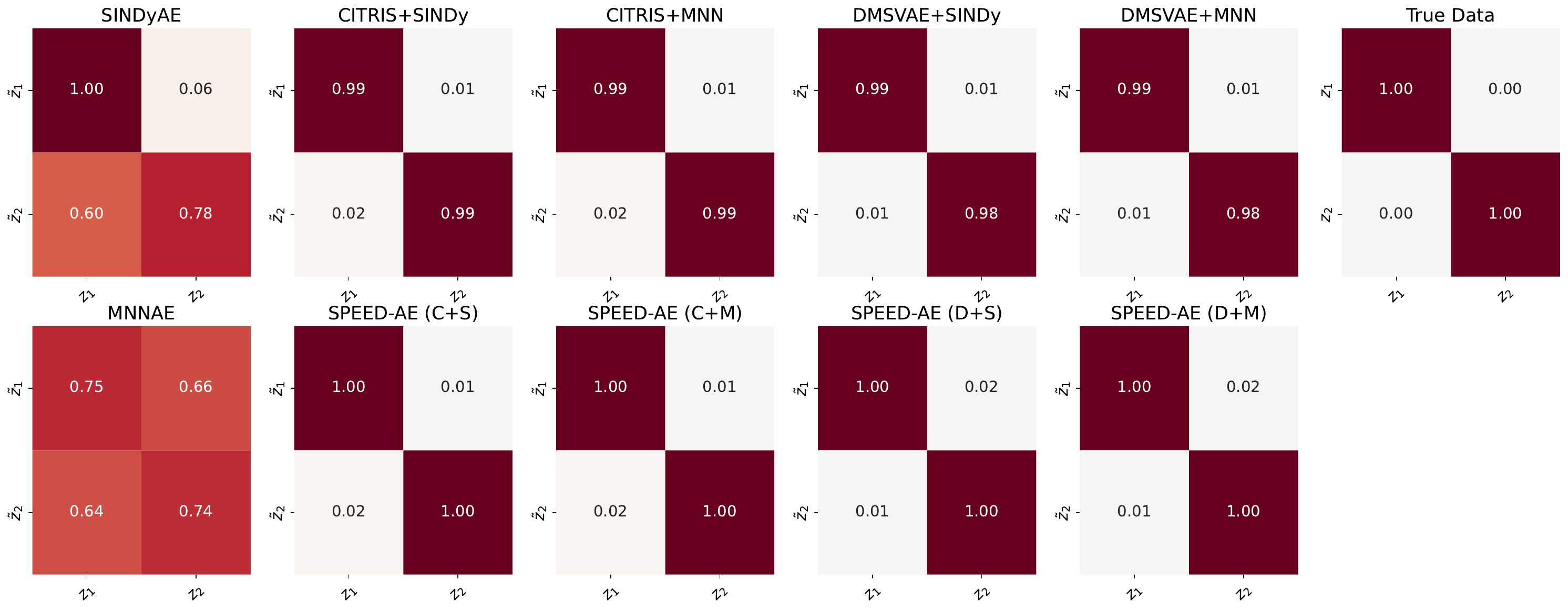}
        \caption{Lotka-Volterra with $\alpha=\frac{6}{5}, \beta=\frac{3}{2}$}
        \label{fig:lv2_cross_pearson}
    \end{subfigure}
    \begin{subfigure}{\linewidth}
        \centering
        \includegraphics[width=0.85\linewidth]{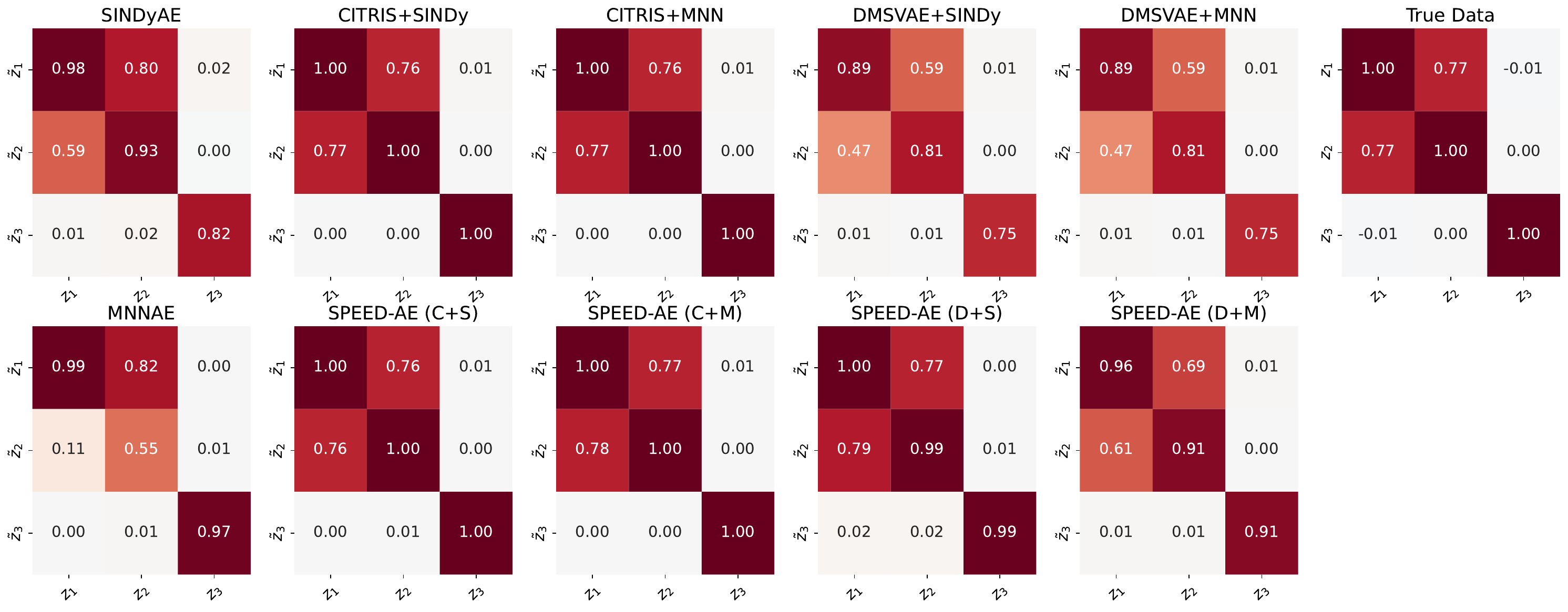}
        \caption{Lorenz stable with $\rho=14$}
        \label{fig:lorenz2_cross_pearson}
    \end{subfigure}
    \begin{subfigure}{\linewidth}
        \centering
        \includegraphics[width=0.85\linewidth]{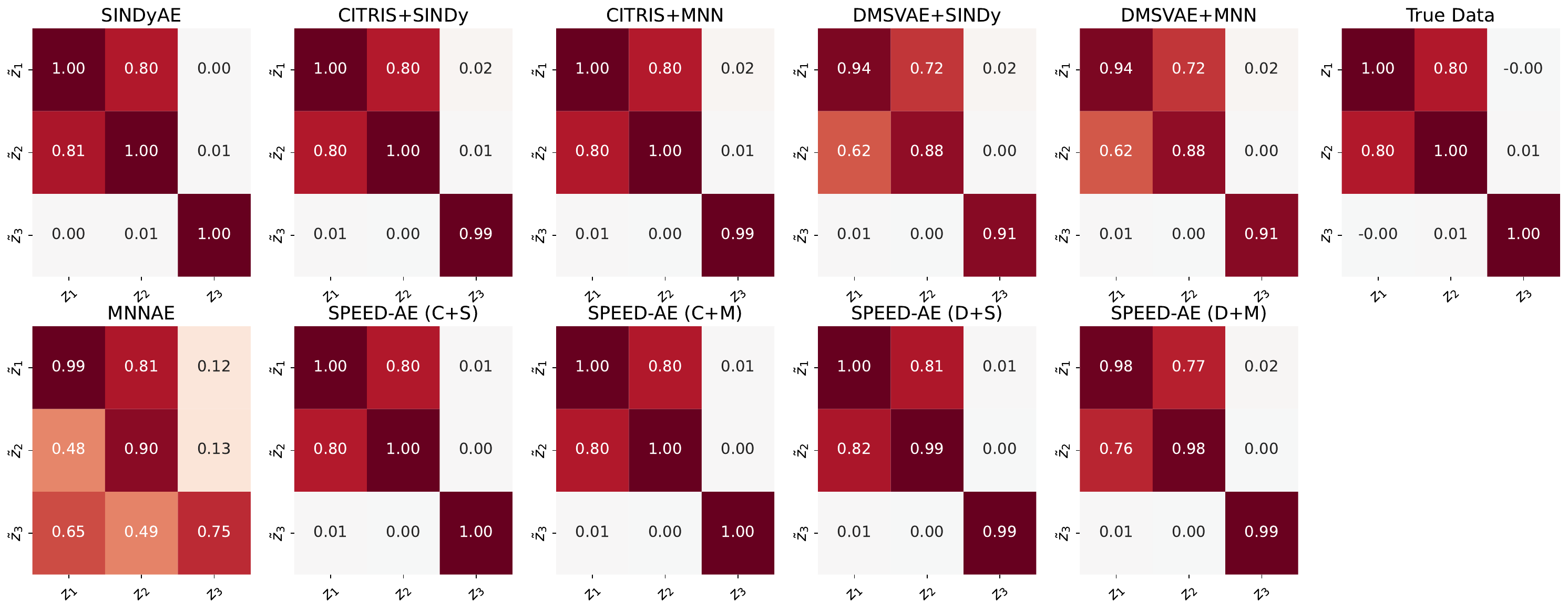}
        \caption{Lorenz chaotic with $\rho=28$}
        \label{fig:lorenz1_cross_pearson}
    \end{subfigure}
    \caption{Pearson cross correlation matrices between $\vzspeed$ and $\rvz$ for all experiments and models. The rightmost matrix in each subfigure corresponds to the true correlation matrix (i.e., between $\rvz$ and $\rvz$).}
    \label{fig:cross_pearson_matrices}
\end{figure}

\newpage
\subsection{Lotka-Volterra Experiment}\label{app:prey_predator}
To visualize how the learned latents of each model $\vzspeed$ are related to the true ones $\rvz$, we scatter plot them for the whole dataset in \Cref{fig:lv1_diffeos} and \Cref{fig:lv2_diffeos}. In both Lotka-Volterra experiments, SINDyAE and MNNAE do not disentangle the variables, as there is no one-to-one correspondence between each $\zspeed_i$ and $z_i$. Similarly, their correlation matrices in \Cref{fig:lv1_cross_pearson} and \Cref{fig:lv2_cross_pearson} are very different from the true one. On the other hand, CITRIS and DMSVAE successfully disentangle the variables, but the mapping is not linear. \modelnameshort{} improves on this in almost every case, indicating that the model achieves identifiability up to linear transformation. We also report the recovered ODE coefficients in the true latent space $\rvz$ for the Lotka-Volterra experiment with transfer of CRL representations and $\alpha=\frac{6}{5}, \beta=\frac{3}{2}$. A similar conclusion to the results in \Cref{sec:experiments} holds even in this case: \modelnameshort{} recovers the sparsest and most similar equations, while SINDy and MNN applied to CRL latents have to sacrifice sparsity and interpretability for good forecasting.

Finally, we report representative plots of $\rvx$ forecasting (first $10$ components out of $128$) for the employed ML baselines in \Cref{fig:lv1_lstm_traj} and \Cref{fig:lv1_transformer_traj}. In both cases, model rollout trajectories collapse in less than $1000$ steps.

\begin{figure}[p]
    \centering
    \begin{subfigure}{0.49\linewidth}
        \centering
        \includegraphics[width=\linewidth]{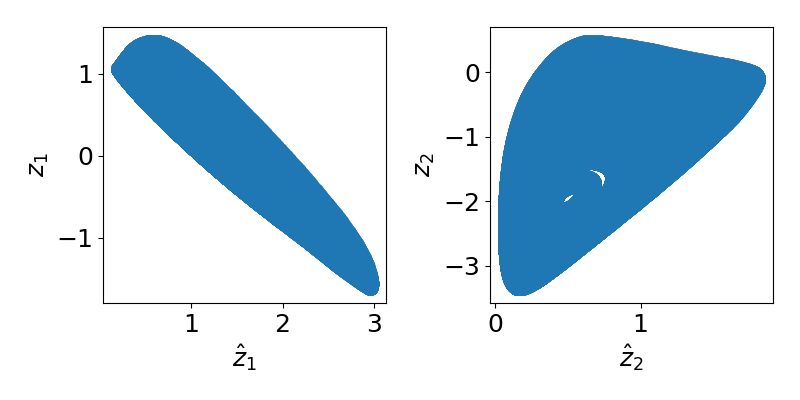}
        \caption{SINDyAE}
        \label{fig:lv1_sindyae_h}
    \end{subfigure}
    \begin{subfigure}{0.49\linewidth}
        \centering
        \includegraphics[width=\linewidth]{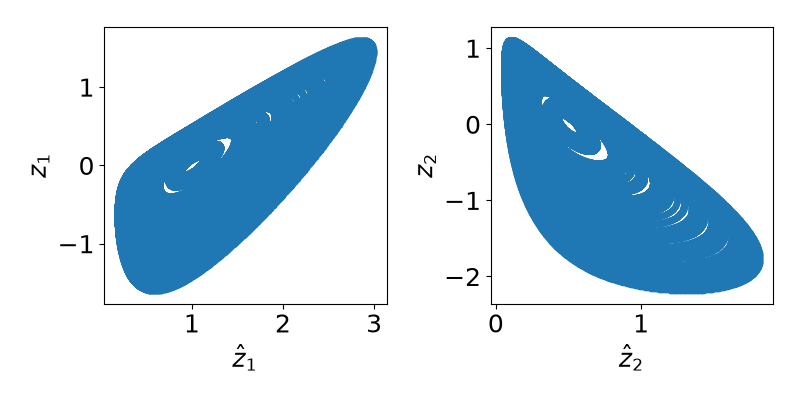}
        \caption{MNNAE}
        \label{fig:lv1_mnnae_h}
    \end{subfigure}
    \begin{subfigure}{0.49\linewidth}
        \centering
        \includegraphics[width=\linewidth]{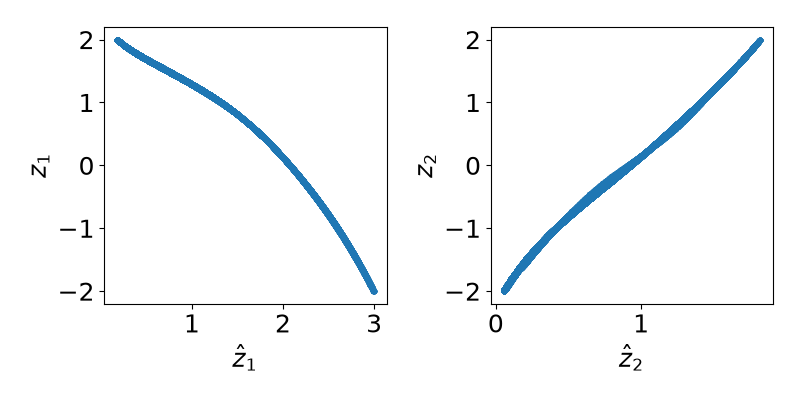}
        \caption{CITRIS}
        \label{fig:lv1_citris_h}
    \end{subfigure}
    \begin{subfigure}{0.49\linewidth}
        \centering
        \includegraphics[width=\linewidth]{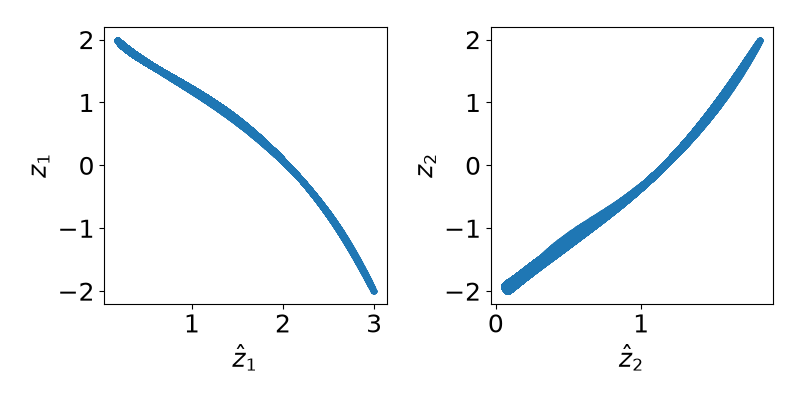}
        \caption{DMSVAE}
        \label{fig:lv1_dmsvae_h}
    \end{subfigure}
    \begin{subfigure}{0.49\linewidth}
        \centering
        \includegraphics[width=\linewidth]{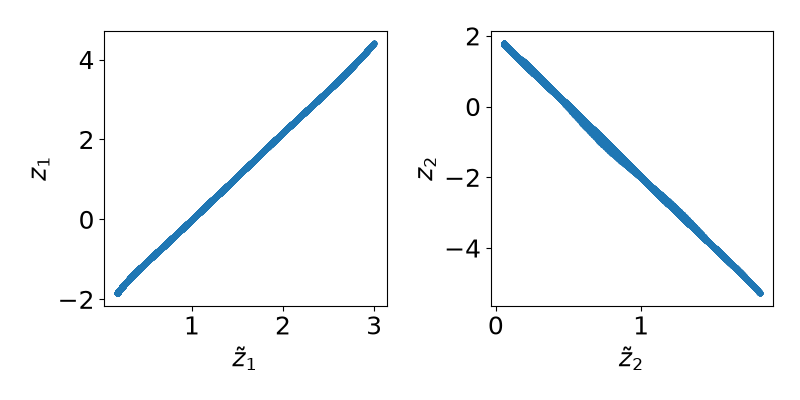}
        \caption{\modelnameshort{} (C+S)}
        \label{fig:lv1_speed_cs_h}
    \end{subfigure}
    \begin{subfigure}{0.49\linewidth}
        \centering
        \includegraphics[width=\linewidth]{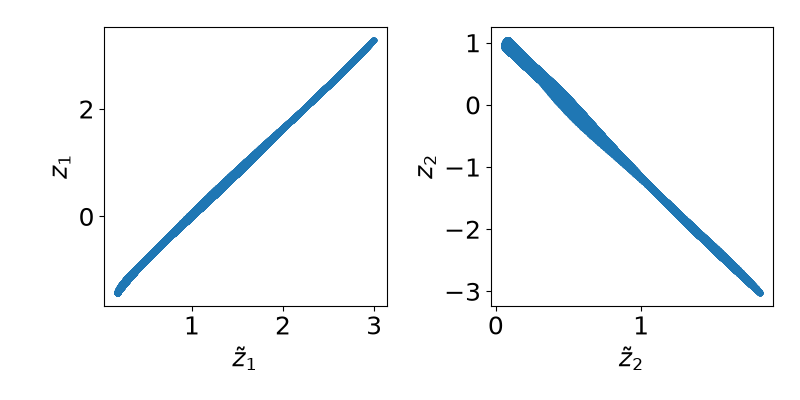}
        \caption{\modelnameshort{} (D+S)}
        \label{fig:lv1_speed_ds_h}
    \end{subfigure}
    \begin{subfigure}{0.49\linewidth}
        \centering
        \includegraphics[width=\linewidth]{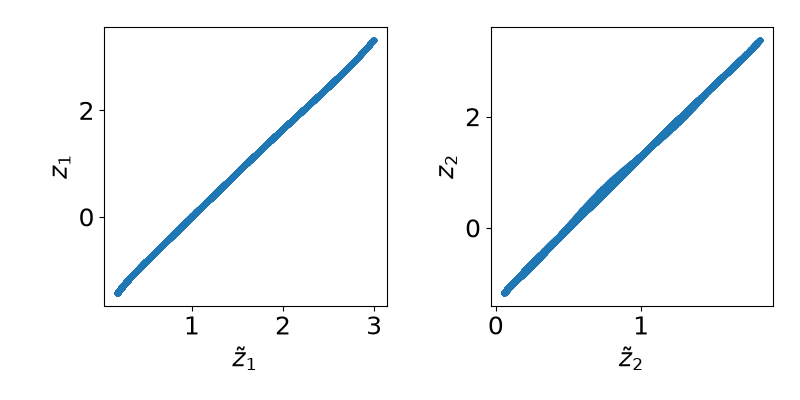}
        \caption{\modelnameshort{} (C+M)}
        \label{fig:lv1_speed_cm_h}
    \end{subfigure}
    \begin{subfigure}{0.49\linewidth}
        \centering
        \includegraphics[width=\linewidth]{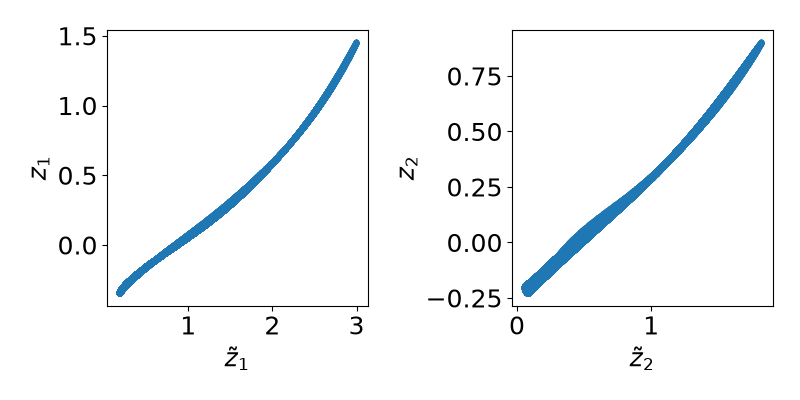}
        \caption{\modelnameshort{} (D+M)}
        \label{fig:lv1_speed_dm_h}
    \end{subfigure}
    \caption{Lotka-Volterra experiment with $\alpha=\frac{2}{3}, \beta=\frac{4}{3}$. Scatter plots of the found latents $\vzspeed$ against the true ones $\rvz$. If their function is well defined, the plots well represent the maps $z_i=h_i(\zspeed_i)$. In the second row, we include the same plots for the CRL latents $\vzcrl$ from CITRIS and DMSVAE. In most cases, \modelnameshort{} can go from a general diffeomorphism $h$ to a linear one.}
    \label{fig:lv1_diffeos}
\end{figure}

\begin{figure}[p]
    \centering
    \begin{subfigure}{0.49\linewidth}
        \centering
        \includegraphics[width=\linewidth]{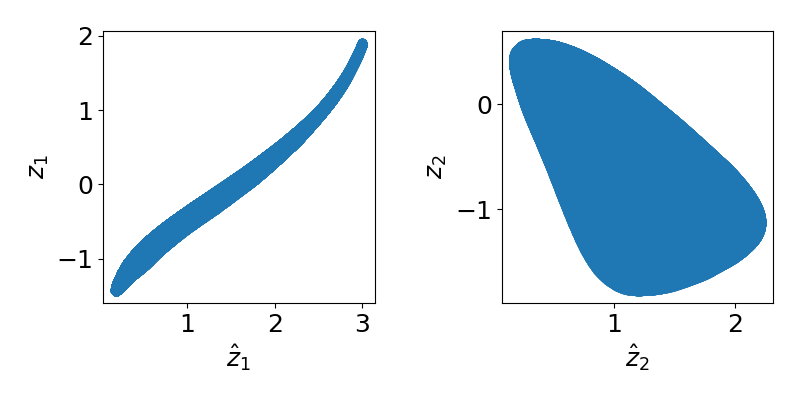}
        \caption{SINDyAE}
        \label{fig:lv2_sindyae_h}
    \end{subfigure}
    \begin{subfigure}{0.49\linewidth}
        \centering
        \includegraphics[width=\linewidth]{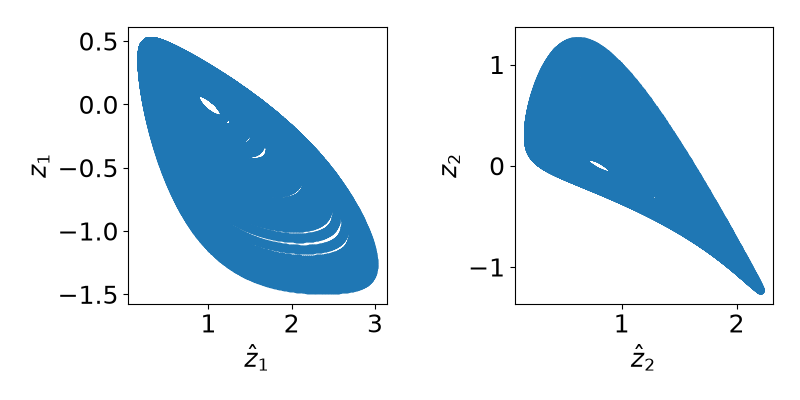}
        \caption{MNNAE}
        \label{fig:lv2_mnnae_h}
    \end{subfigure}
    \begin{subfigure}{0.49\linewidth}
        \centering
        \includegraphics[width=\linewidth]{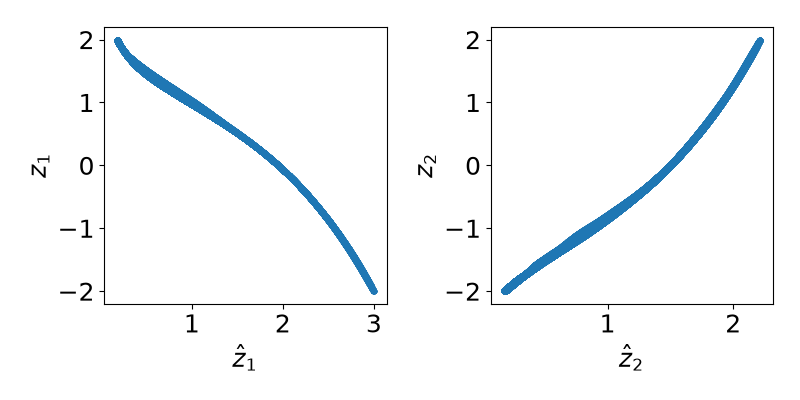}
        \caption{CITRIS}
        \label{fig:lv2_citris_h}
    \end{subfigure}
    \begin{subfigure}{0.49\linewidth}
        \centering
        \includegraphics[width=\linewidth]{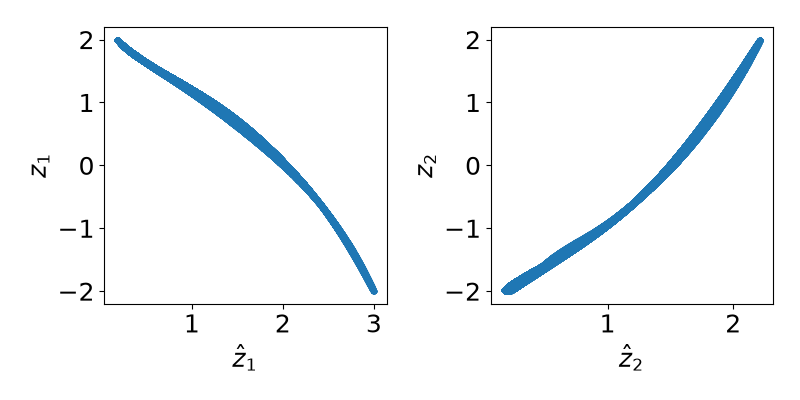}
        \caption{DMSVAE}
        \label{fig:lv2_dmsvae_h}
    \end{subfigure}
    \begin{subfigure}{0.49\linewidth}
        \centering
        \includegraphics[width=\linewidth]{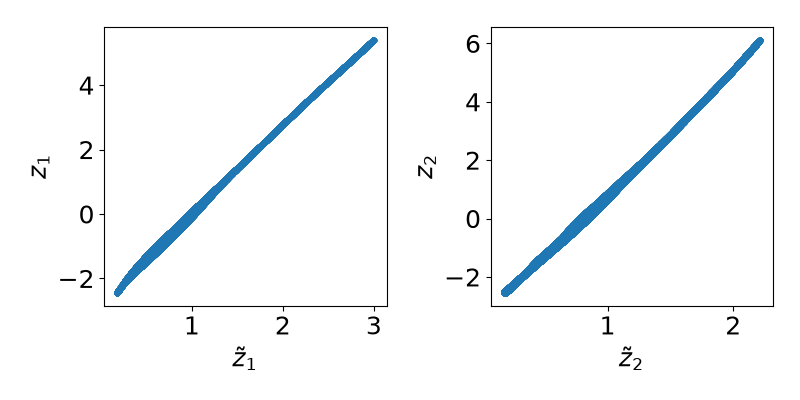}
        \caption{\modelnameshort{} (C+S)}
        \label{fig:lv2_speed_cs_h}
    \end{subfigure}
    \begin{subfigure}{0.49\linewidth}
        \centering
        \includegraphics[width=\linewidth]{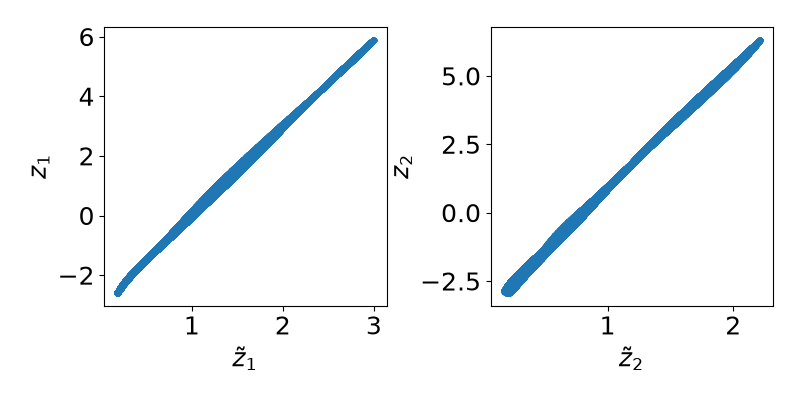}
        \caption{\modelnameshort{} (D+S)}
        \label{fig:lv2_speed_ds_h}
    \end{subfigure}
    \begin{subfigure}{0.49\linewidth}
        \centering
        \includegraphics[width=\linewidth]{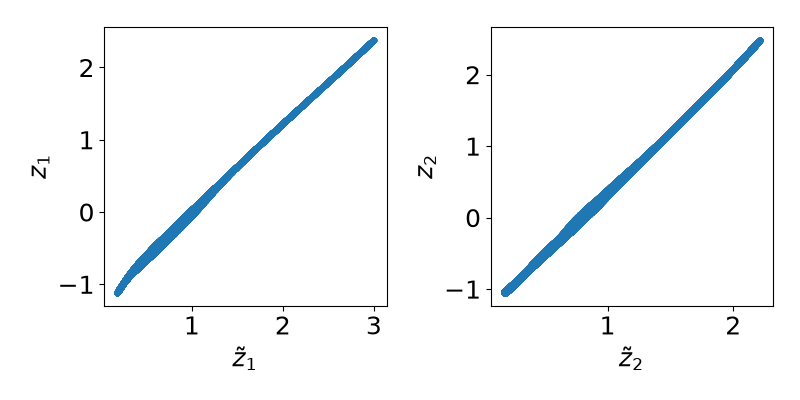}
        \caption{\modelnameshort{} (C+M)}
        \label{fig:lv2_speed_cm_h}
    \end{subfigure}
    \begin{subfigure}{0.49\linewidth}
        \centering
        \includegraphics[width=\linewidth]{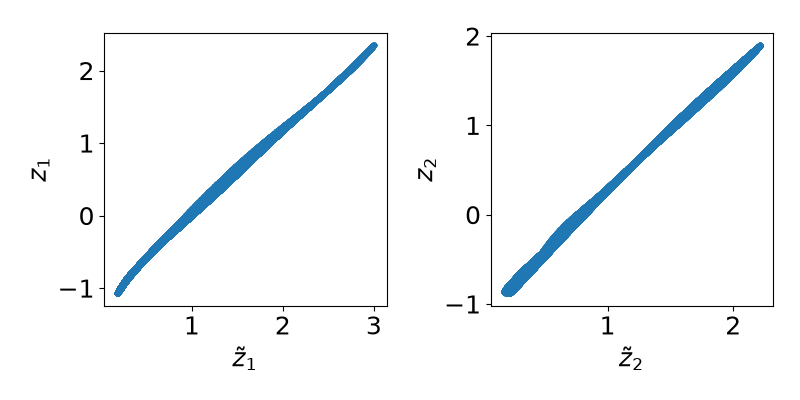}
        \caption{\modelnameshort{} (D+M)}
        \label{fig:lv2_speed_dm_h}
    \end{subfigure}
    \caption{Lotka-Volterra experiment woth $\alpha=\frac{6}{5}, \beta=\frac{3}{2}$. Scatter plots of the found latents $\vzspeed$ against the true ones $\rvz$. If their function is well defined, the plots well represent the maps $z_i=h_i(\zspeed_i)$. In the second row, we include the same plots for the CRL latents $\vzcrl$ from CITRIS and DMSVAE. In most cases, \modelnameshort{} can go from a general diffeomorphism $h$ to a linear one.}
    \label{fig:lv2_diffeos}
\end{figure}

\begin{figure}
    \centering
    \includegraphics[width=\linewidth]{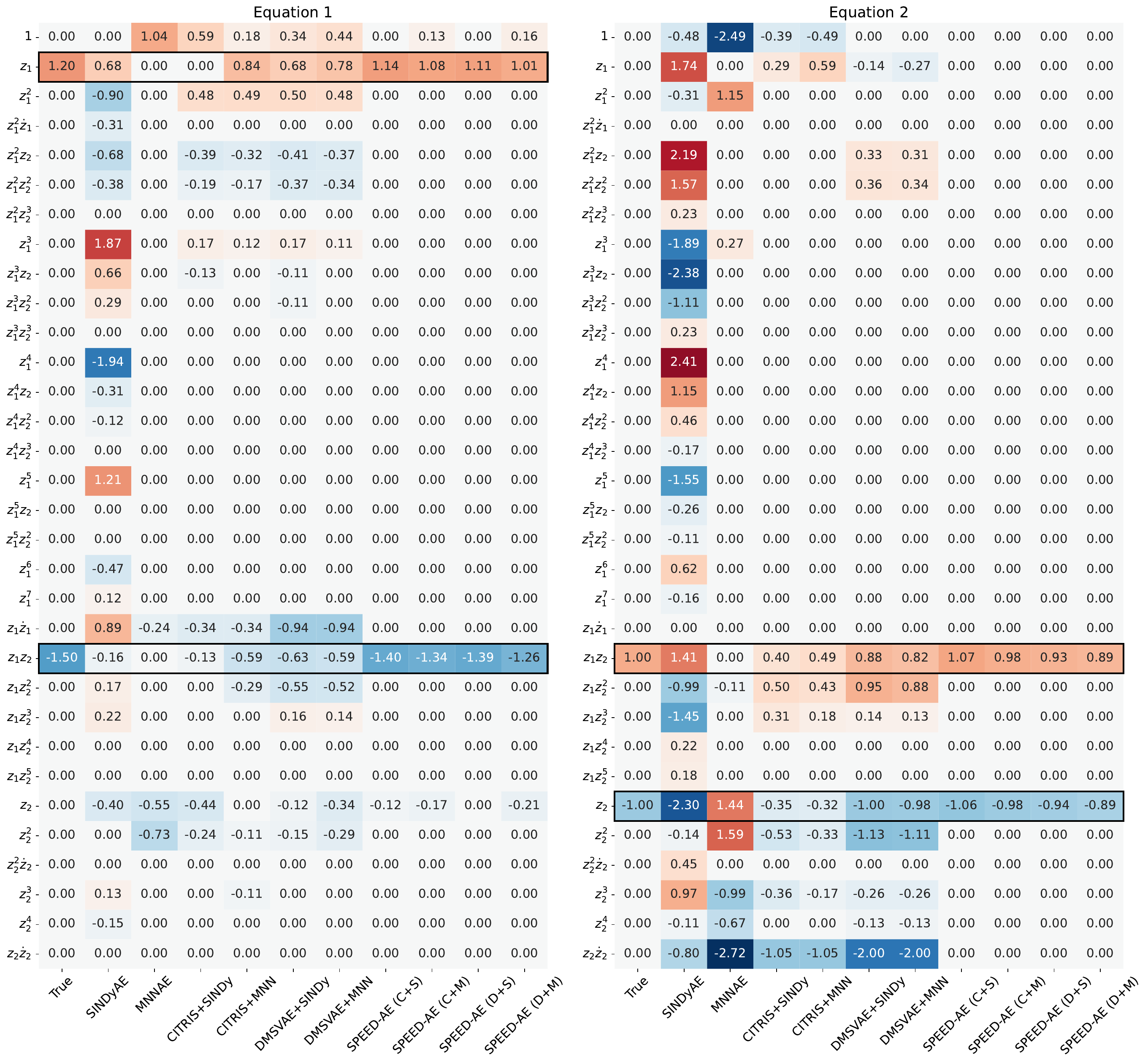}
    \caption{Identified coefficients for the Lotka-Volterra experiment with transfer of CRL models and $\alpha=\frac{6}{5}, \beta=\frac{3}{2}$. The true ones are in the leftmost column and highlighted by the black borders. For comparison, we include four baselines that consist of equation discovery directly applied to the CRL latents $\vzcrl$.}
    \label{fig:prey_predator2_coeffs}
\end{figure}

\begin{figure}
    \centering
    \includegraphics[width=\linewidth]{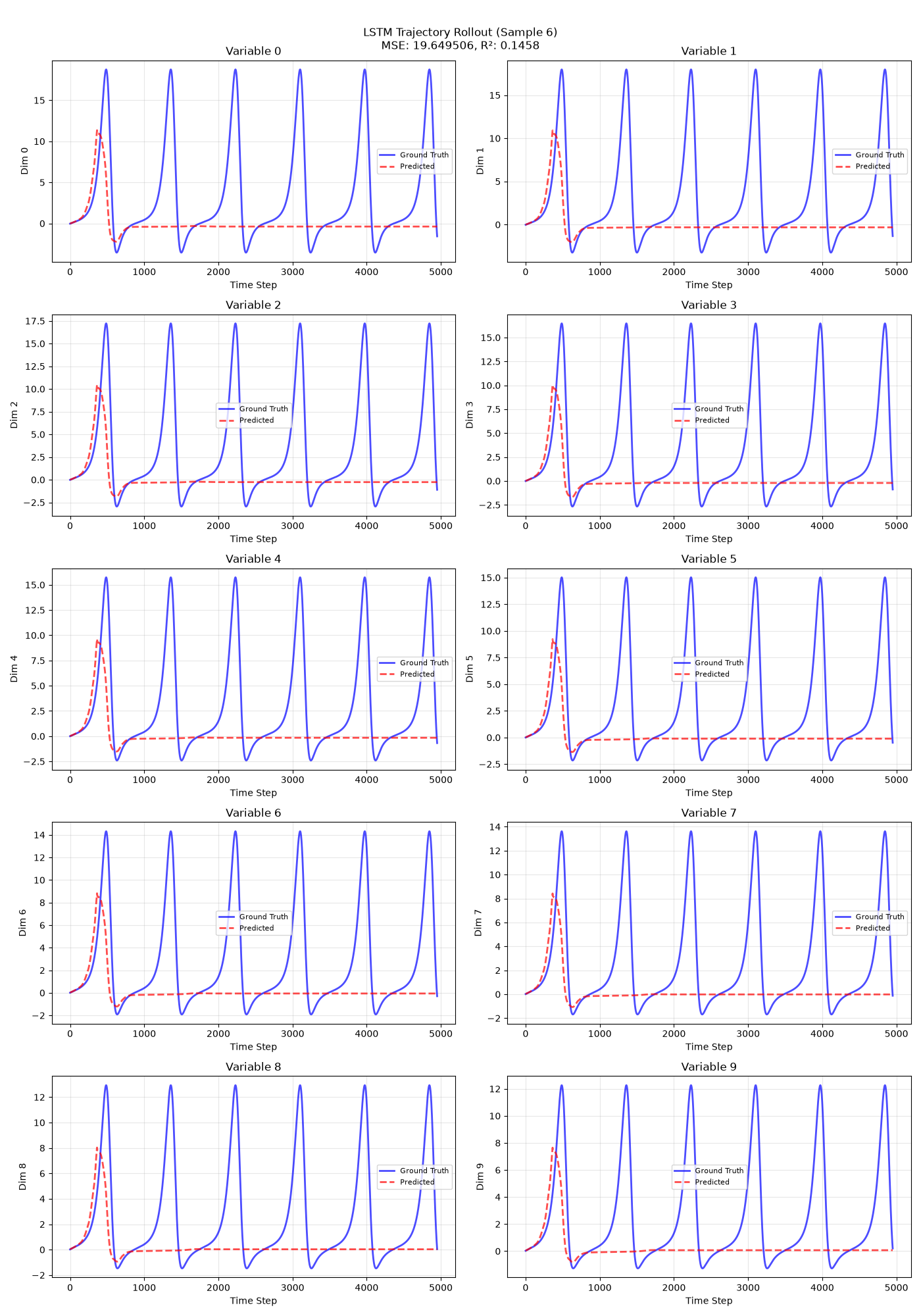}
    \caption{Example test trajectory of the Lotka-Volterra experiment (first $10$ dimensions of $\rvx$). Ground truth (blue) against the predicted trajectory from the LSTM model (red). During rollout, the model collapses after around 500 steps.}
    \label{fig:lv1_lstm_traj}
\end{figure}

\begin{figure}
    \centering
    \includegraphics[width=\linewidth]{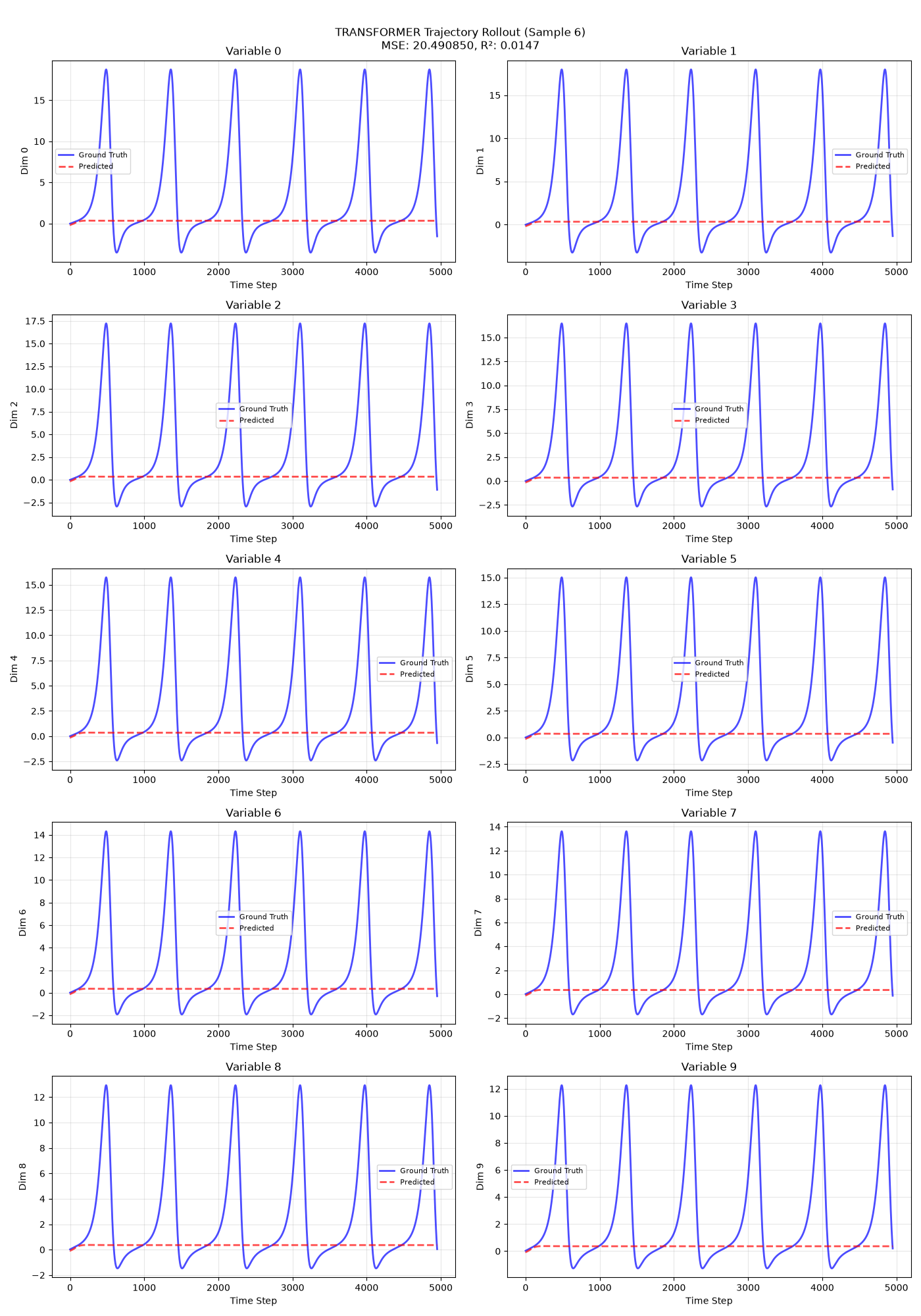}
    \caption{Example test trajectory of the Lotka-Volterra experiment (first $10$ dimensions of $\rvx$). Ground truth (blue) against the predicted trajectory from the Transformer model (red). During rollout, the model collapses after a few iterations.}
    \label{fig:lv1_transformer_traj}
\end{figure}
\newpage

\subsection{Lorenz Experiment}\label{app:lorenz}
We report similar scatter plots for the Lorenz stable and chaotic experiments, respectively, in \Cref{fig:lorenz2_diffeos} and \Cref{fig:lorenz1_diffeos}, with cross-correlation matrices in  \Cref{fig:lorenz2_cross_pearson} and \Cref{fig:lorenz1_cross_pearson}. SINDyAE learns variables with good disentanglement and cross-correlation in the chaotic case, which data and parameters come directly from the original work of \citet{Champion2019DatadrivenDiscoveryCoordinates}, but struggles in the stable case. MNNAE has strong forecasting performance, but learns variables that do not easily relate to the true ones. In these experiments, DMSVAE performs significantly worse than CITRIS. However, \modelnameshort{} is actually able to improve both in terms of cross-correlation and identifiability (compare \Cref{fig:lorenz2_dmsvae_h} with \Cref{fig:lorenz2_speed_ds_h} and \Cref{fig:lorenz1_dmsvae_h} with \Cref{fig:lorenz1_speed_ds_h}, and the matrices in \Cref{fig:lorenz2_cross_pearson}).

Finally, we report the identified ODE coefficients in \Cref{fig:lorenz_coefficients_comparison}.

\begin{figure}
    \centering
    \begin{subfigure}{0.49\linewidth}
        \centering
        \includegraphics[width=\linewidth]{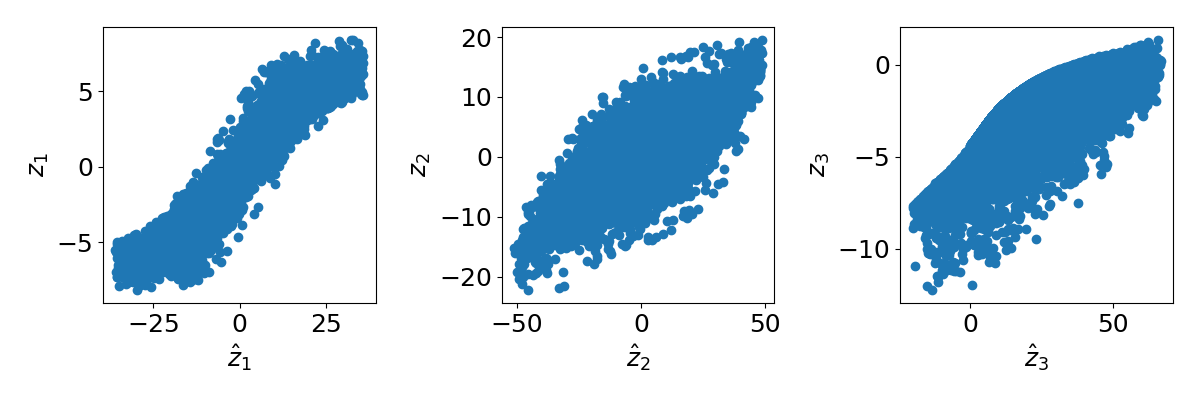}
        \caption{SINDyAE}
        \label{fig:lorenz2_sindyae_h}
    \end{subfigure}
    \begin{subfigure}{0.49\linewidth}
        \centering
        \includegraphics[width=\linewidth]{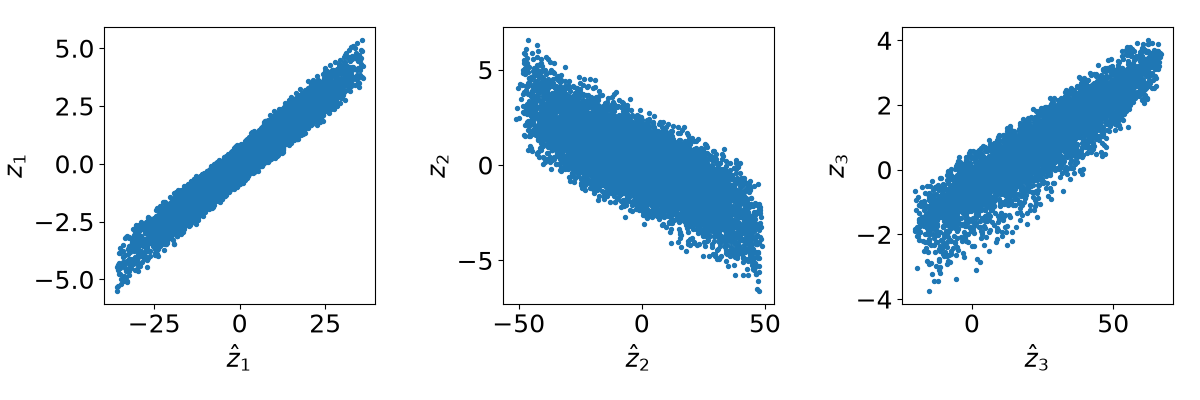}
        \caption{MNNAE}
        \label{fig:lorenz2_mnnae_h}
    \end{subfigure}
    \begin{subfigure}{0.49\linewidth}
        \centering
        \includegraphics[width=\linewidth]{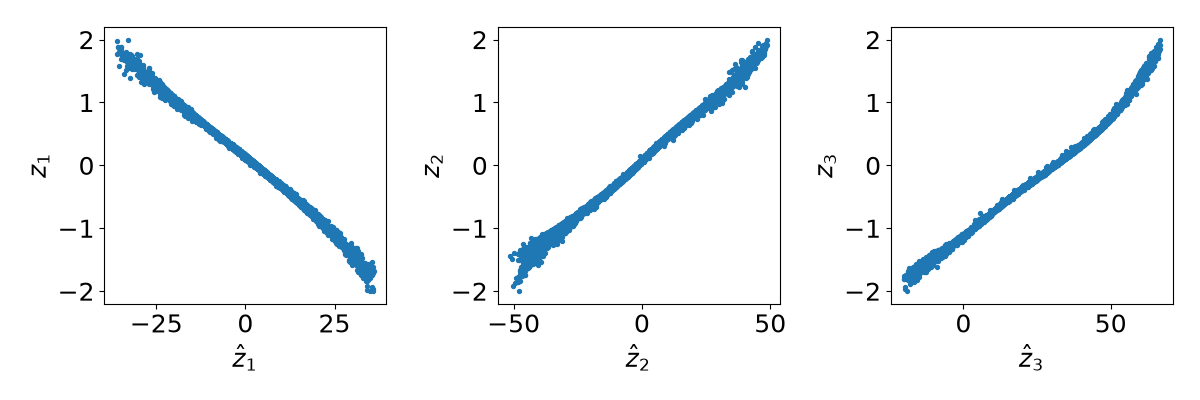}
        \caption{CITRIS}
        \label{fig:lorenz2_citris_h}
    \end{subfigure}
    \begin{subfigure}{0.49\linewidth}
        \centering
        \includegraphics[width=\linewidth]{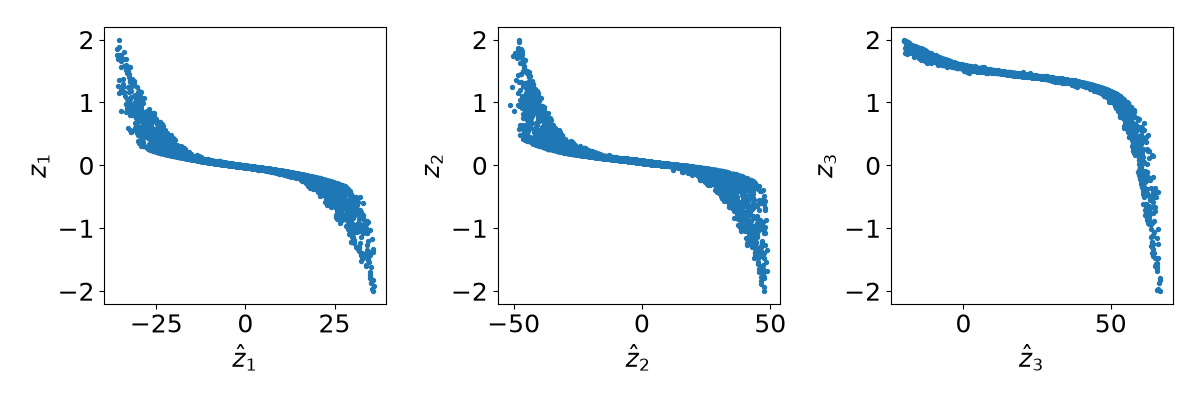}
        \caption{DMSVAE}
        \label{fig:lorenz2_dmsvae_h}
    \end{subfigure}
    \begin{subfigure}{0.49\linewidth}
        \centering
        \includegraphics[width=\linewidth]{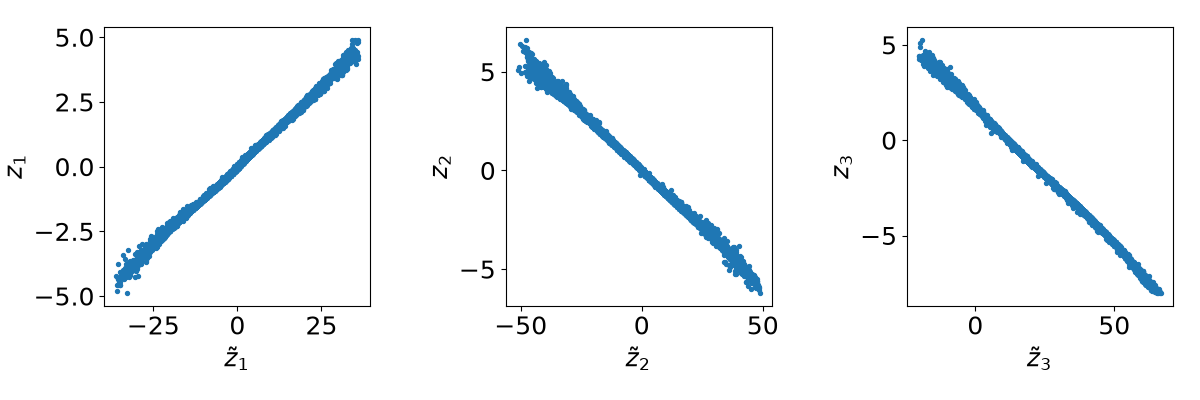}
        \caption{\modelnameshort{} (C+S)}
        \label{fig:lorenz2_speed_cs_h}
    \end{subfigure}
    \begin{subfigure}{0.49\linewidth}
        \centering
        \includegraphics[width=\linewidth]{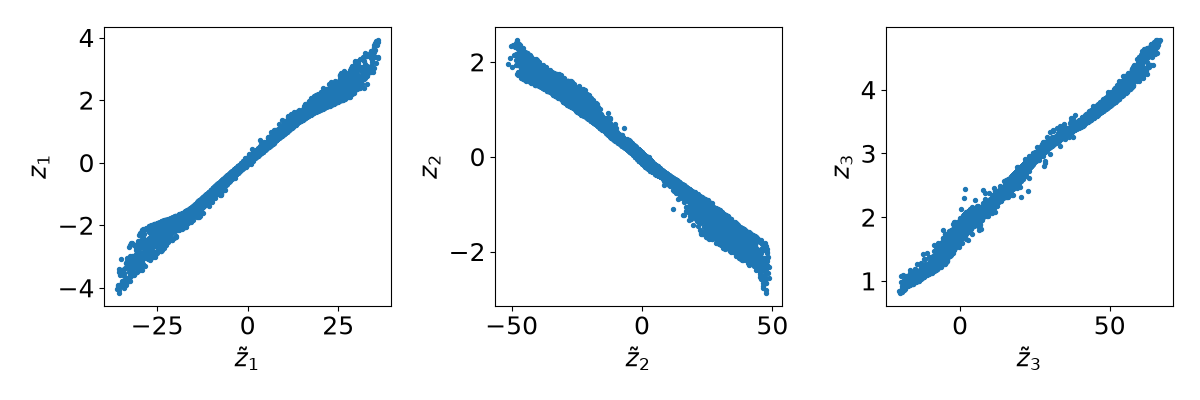}
        \caption{\modelnameshort{} (D+S)}
        \label{fig:lorenz2_speed_ds_h}
    \end{subfigure}
    \begin{subfigure}{0.49\linewidth}
        \centering
        \includegraphics[width=\linewidth]{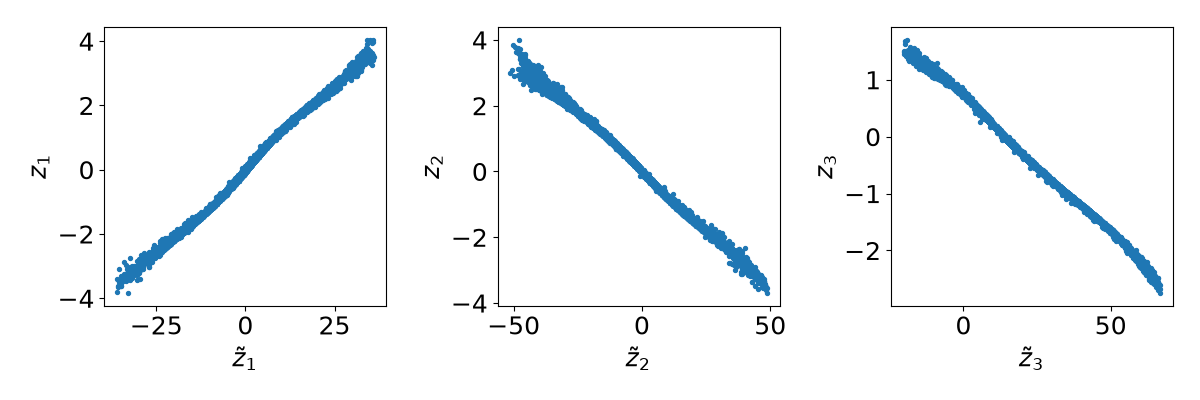}
        \caption{\modelnameshort{} (C+M)}
        \label{fig:lorenz2_speed_cm_h}
    \end{subfigure}
    \begin{subfigure}{0.49\linewidth}
        \centering
        \includegraphics[width=\linewidth]{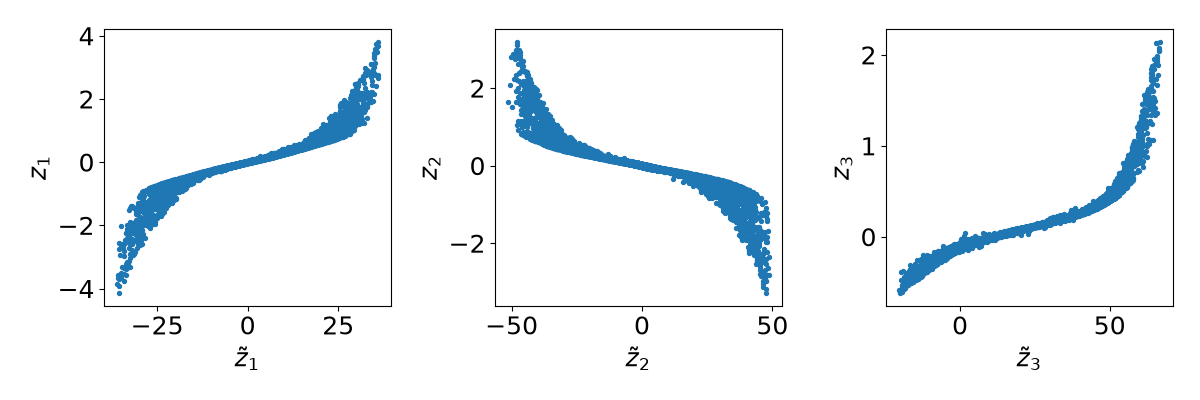}
        \caption{\modelnameshort{} (D+M)}
        \label{fig:lorenz2_speed_dm_h}
    \end{subfigure}
    \caption{Lorenz experiment with stable trajectories $\rho=14$. Scatter plots of the found latents $\vzspeed$ against the true ones $\rvz$. If their function is well defined, the plots well represent the maps $z_i=h_i(\zspeed_i)$. In the second row, we include the same plots for the CRL latents $\vzcrl$ from CITRIS and DMSVAE. \modelnameshort{} models with DMSVAE as disentangler suffer more due to the CRL latents being less disentangled, although the SINDy-based loss is actually able to recover a linear map $h$.}
    \label{fig:lorenz2_diffeos}
\end{figure}

\begin{figure}[p]
    \centering
    \begin{subfigure}{0.49\linewidth}
        \centering
        \includegraphics[width=\linewidth]{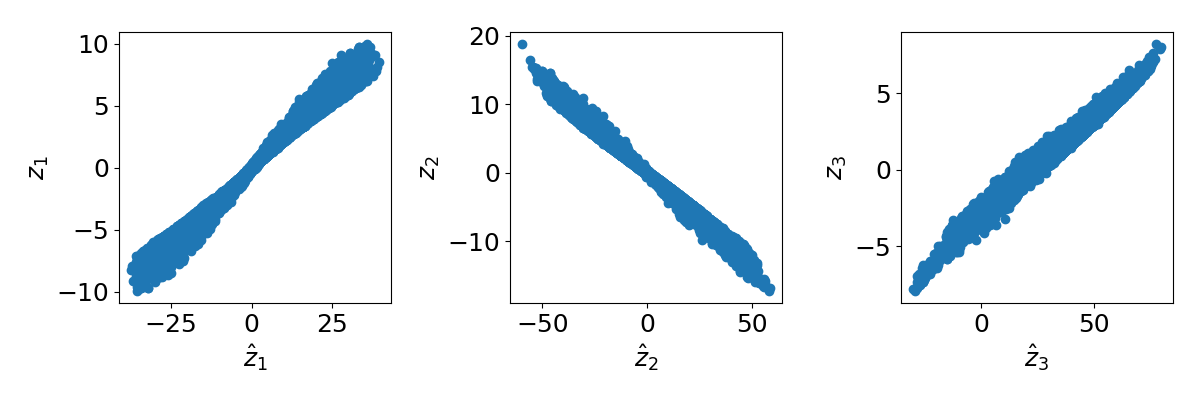}
        \caption{SINDyAE}
        \label{fig:lorenz1_sindyae_h}
    \end{subfigure}
    \begin{subfigure}{0.49\linewidth}
        \centering
        \includegraphics[width=\linewidth]{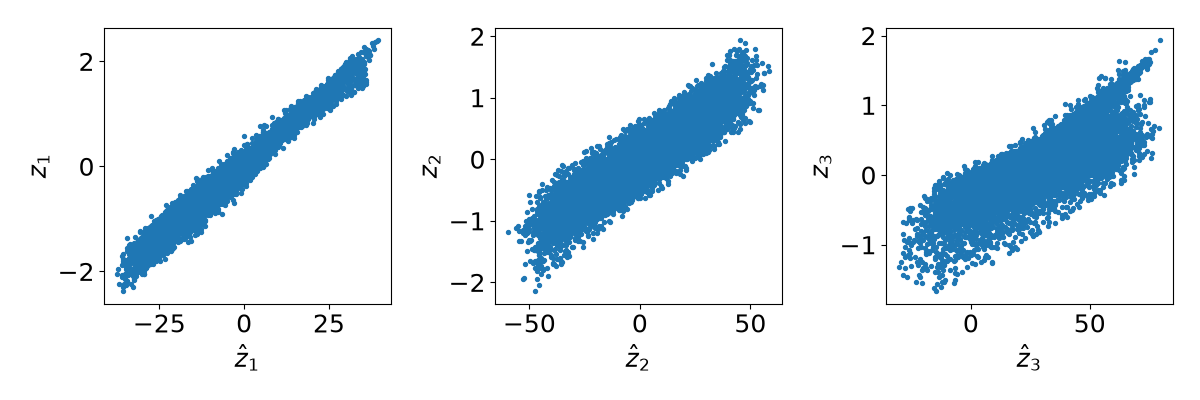}
        \caption{MNNAE}
        \label{fig:lorenz1_mnnae_h}
    \end{subfigure}
    \begin{subfigure}{0.49\linewidth}
        \centering
        \includegraphics[width=\linewidth]{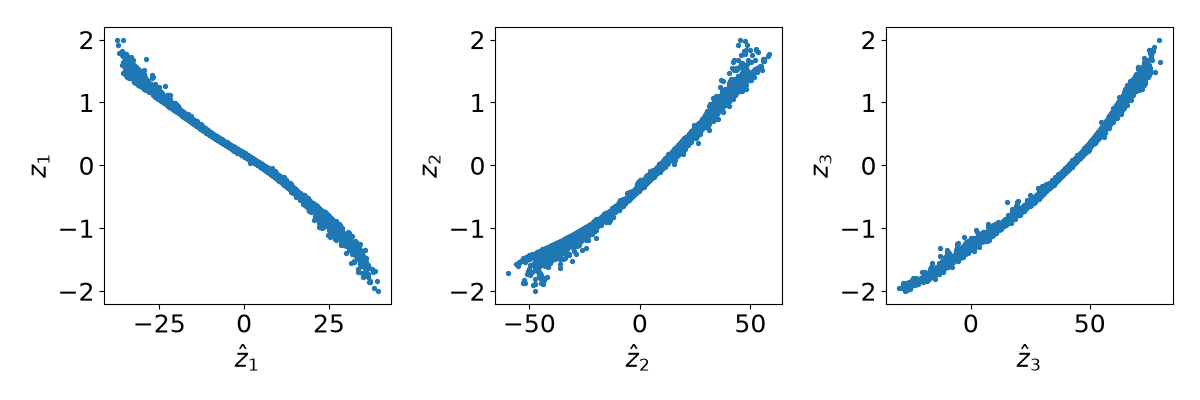}
        \caption{CITRIS}
        \label{fig:lorenz1_citris_h}
    \end{subfigure}
    \begin{subfigure}{0.49\linewidth}
        \centering
        \includegraphics[width=\linewidth]{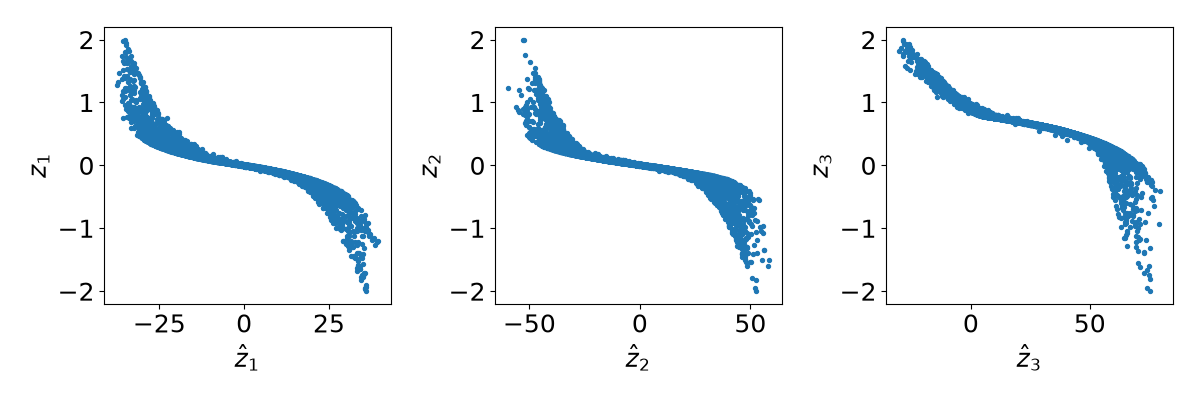}
        \caption{DMSVAE}
        \label{fig:lorenz1_dmsvae_h}
    \end{subfigure}
    \begin{subfigure}{0.49\linewidth}
        \centering
        \includegraphics[width=\linewidth]{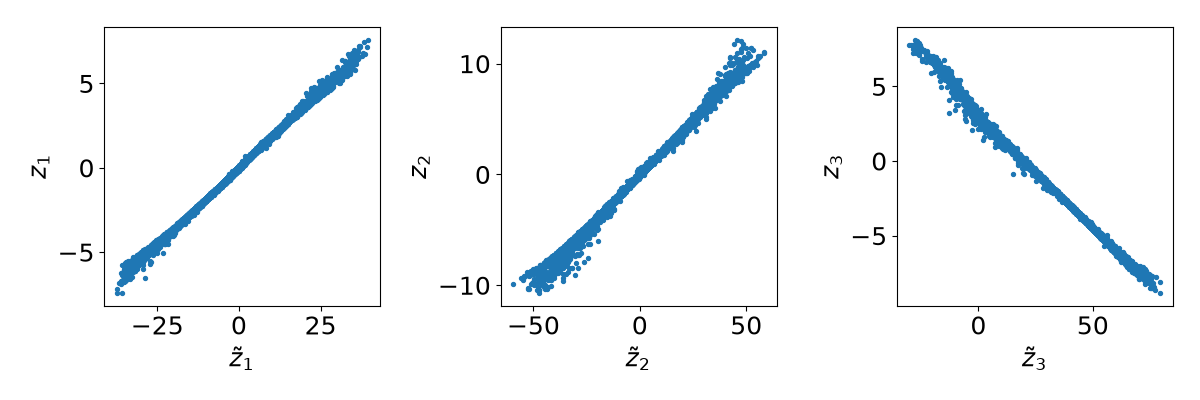}
        \caption{\modelnameshort{} (C+S)}
        \label{fig:lorenz1_speed_cs_h}
    \end{subfigure}
    \begin{subfigure}{0.49\linewidth}
        \centering
        \includegraphics[width=\linewidth]{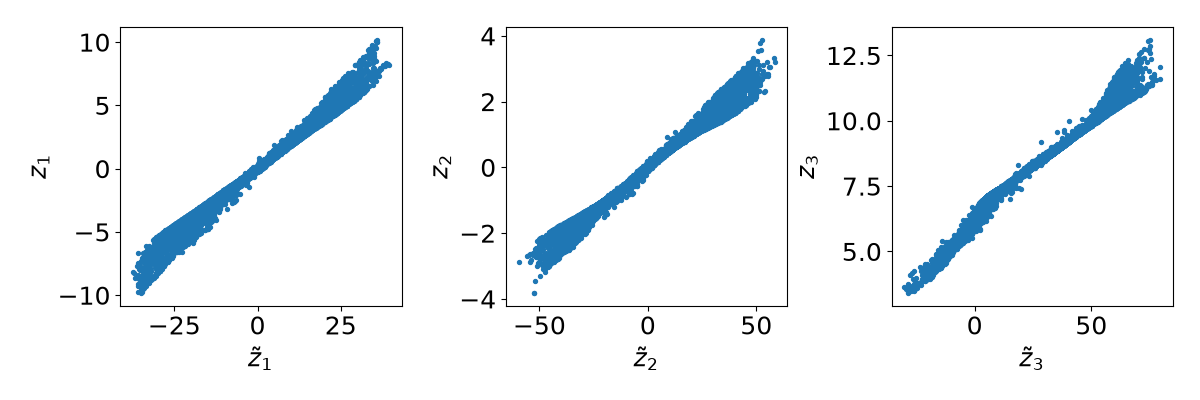}
        \caption{\modelnameshort{} (D+S)}
        \label{fig:lorenz1_speed_ds_h}
    \end{subfigure}
    \begin{subfigure}{0.49\linewidth}
        \centering
        \includegraphics[width=\linewidth]{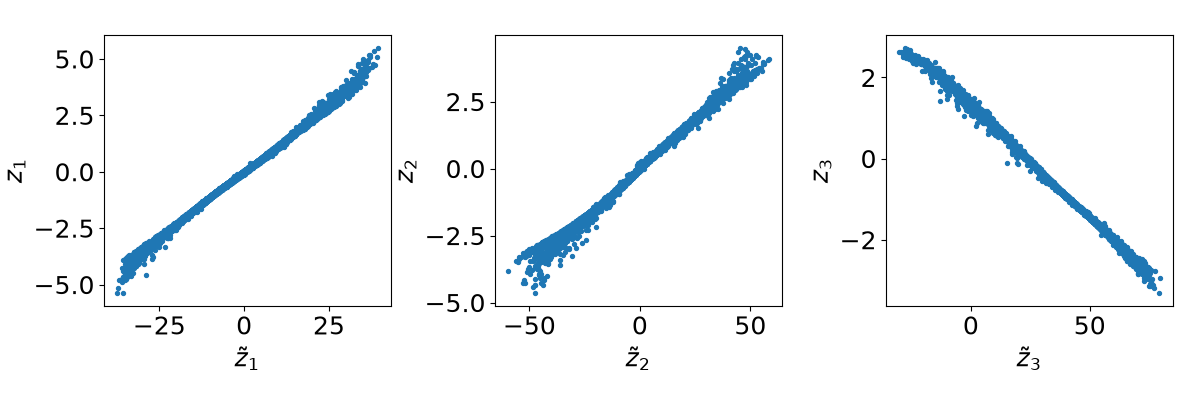}
        \caption{\modelnameshort{} (C+M)}
        \label{fig:lorenz1_speed_cm_h}
    \end{subfigure}
    \begin{subfigure}{0.49\linewidth}
        \centering
        \includegraphics[width=\linewidth]{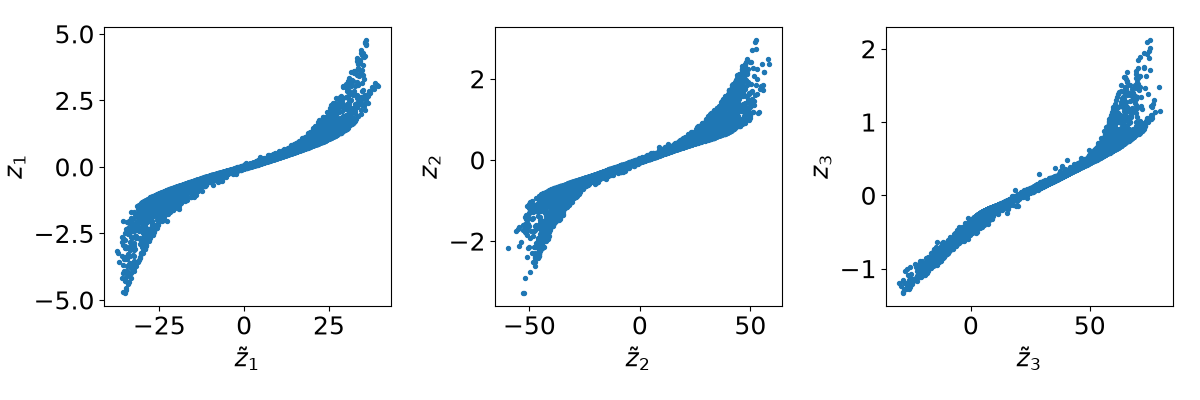}
        \caption{\modelnameshort{} (D+M)}
        \label{fig:lorenz1_speed_dm_h}
    \end{subfigure}
    \caption{Scatter plots of the found latents $\vzspeed$ against the true ones $\rvz$. If their function is well-defined, the plots well represent the maps $z_i=h_i(\zspeed_i)$. In the second row, we include the same plots for the CRL latents $\vzcrl$ from CITRIS and DMSVAE. \modelnameshort{} models with DMSVAE as disentangler suffer more due to the CRL latents being less disentangled, although the SINDy-based loss is actually able to recover a linear map $h$.}
    \label{fig:lorenz1_diffeos}
\end{figure}

\begin{figure}[p]
    \centering
    \begin{subfigure}{0.48\linewidth}
        \centering
        \includegraphics[width=3\linewidth, angle=90]{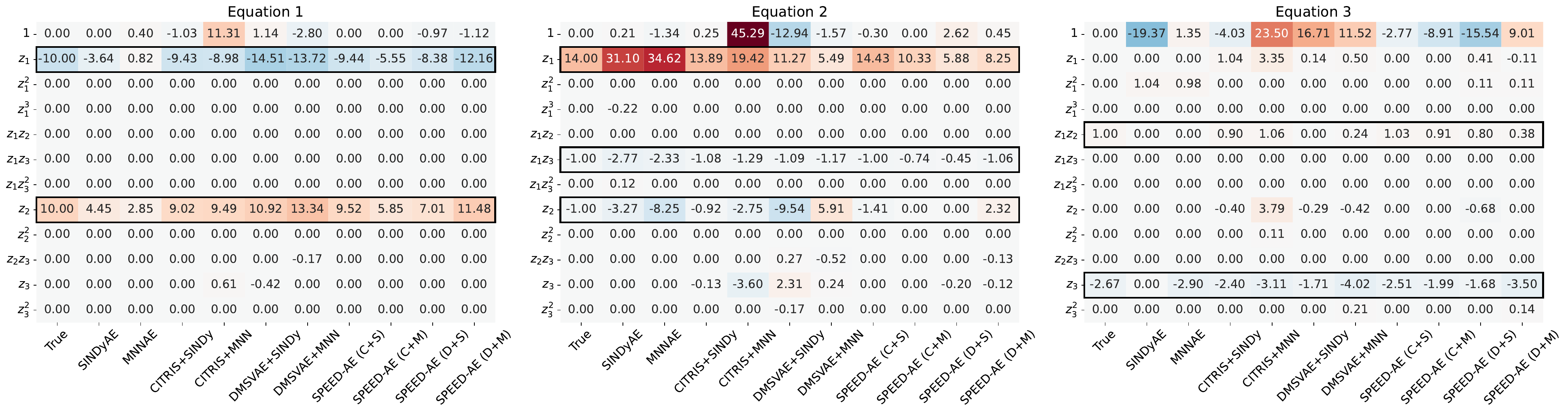}
        \label{fig:lorenz2_coefficients_comparison}
    \end{subfigure}
    \begin{subfigure}{0.48\linewidth}
        \centering
        \includegraphics[width=3\linewidth, angle=90]{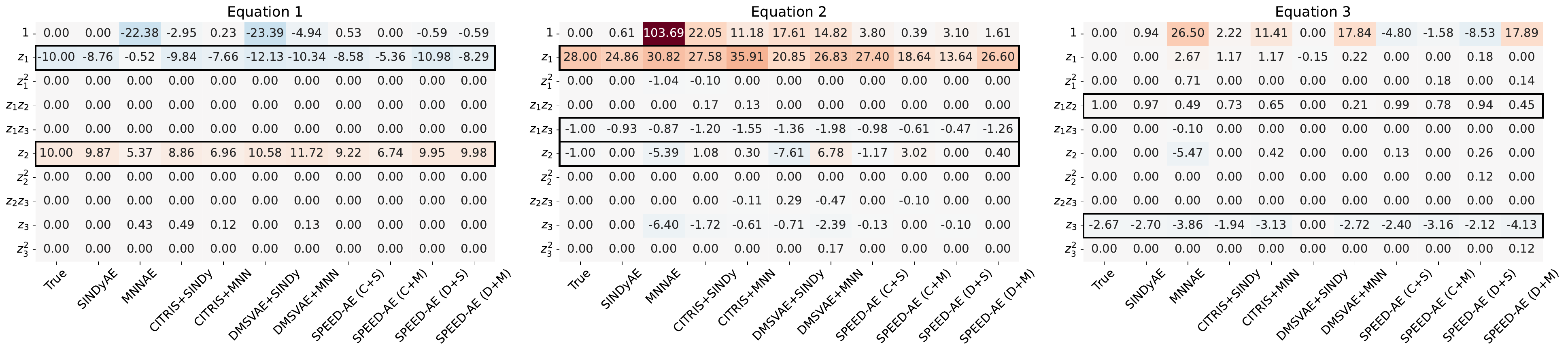}
        \label{fig:lorenz1_coefficients_comparison}
    \end{subfigure}
    \caption{Identified coefficients for the Lorenz experiments in the stable ($\rho=14$, left) and chaotic case ($\rho=28$, right). The true ones are in the leftmost column and highlighted by the black borders. For comparison, we include four baselines that consist of equation discovery directly applied to the CRL latents $\vzcrl$.}
    \label{fig:lorenz_coefficients_comparison}
\end{figure}

\newpage
\subsection{Pendulums Experiment}\label{app:pendulum}
We first include example images of the dataset in \Cref{fig:double_pendulums_examples}. Different pendulums are represented on different channels to ensure the mixing function is invertible and the two pendulums are distinguishable. This should also help disentanglement for all models.

\Cref{fig:double_pendulum_diffeos} shows the scatter plots of true against found latents. CITRIS achieves disentanglement, but not identifiability up to a linear function, where \modelnameshort{} succeeds. SINDyAE has a very interesting behavior: it learns a variable that is linearly related to one of the two pendulum angles, while the other is completely disaligned. Since the pendulums have equal dynamics, and thus their motion is the same except for a constant phase, SINDyAE might have learned one variable for the common dynamics and one that maps the two pendulums to each other. This, however, has the limitation of not being able to identify one of the two pendulums correctly. We identify this behavior in all $10$ seeds we ran the experiment with.

\begin{figure}
    \centering
    \begin{subfigure}{0.32\linewidth}
        \centering
        \includegraphics[width=\linewidth]{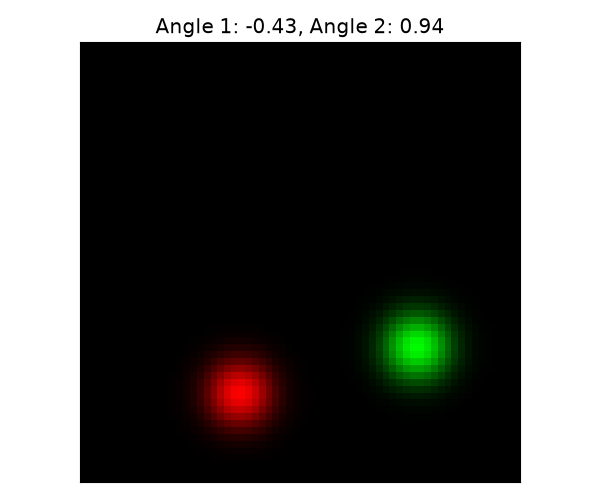}
        \label{fig:double_pendulum_example0}
    \end{subfigure}
    \begin{subfigure}{0.32\linewidth}
        \centering
        \includegraphics[width=\linewidth]{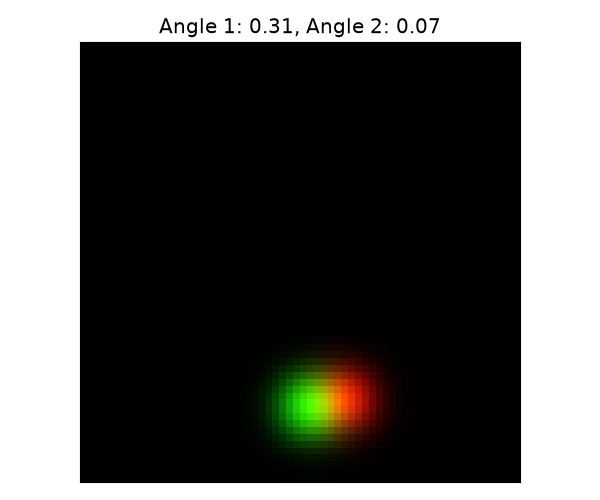}
        \label{fig:double_pendulum_example1}
    \end{subfigure}
    \begin{subfigure}{0.32\linewidth}
        \centering
        \includegraphics[width=\linewidth]{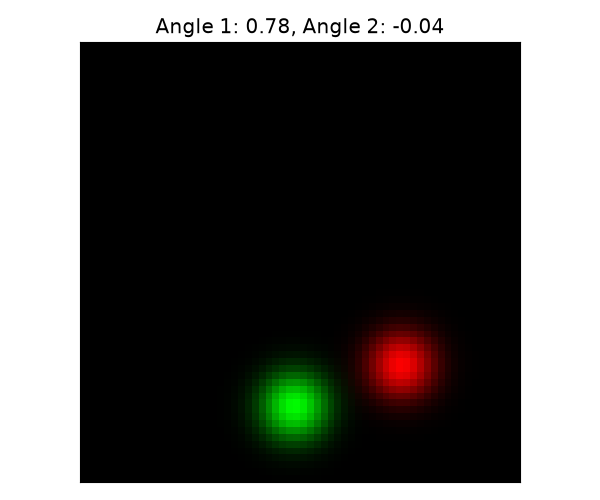}
        \label{fig:double_pendulum_example2}
    \end{subfigure}
    \begin{subfigure}{0.32\linewidth}
        \centering
        \includegraphics[width=\linewidth]{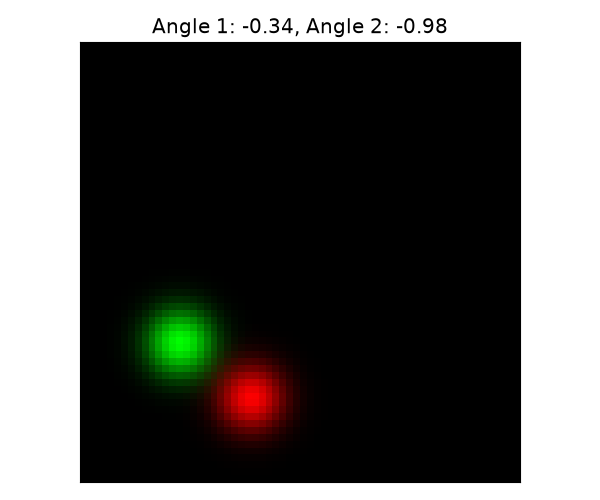}
        \label{fig:double_pendulum_example3}
    \end{subfigure}
    \begin{subfigure}{0.32\linewidth}
        \centering
        \includegraphics[width=\linewidth]{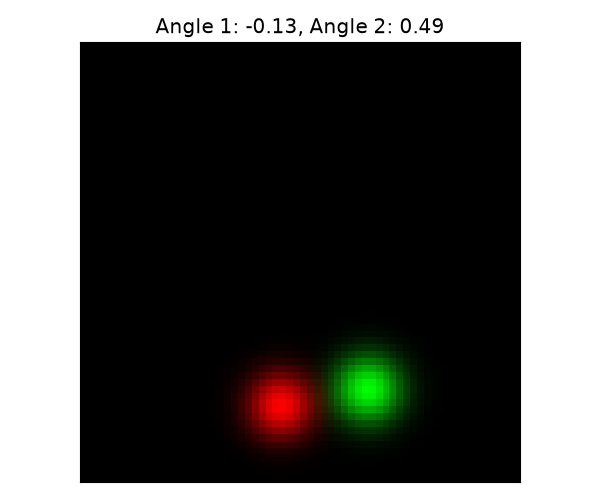}
        \label{fig:double_pendulum_example4}
    \end{subfigure}
    \begin{subfigure}{0.32\linewidth}
        \centering
        \includegraphics[width=\linewidth]{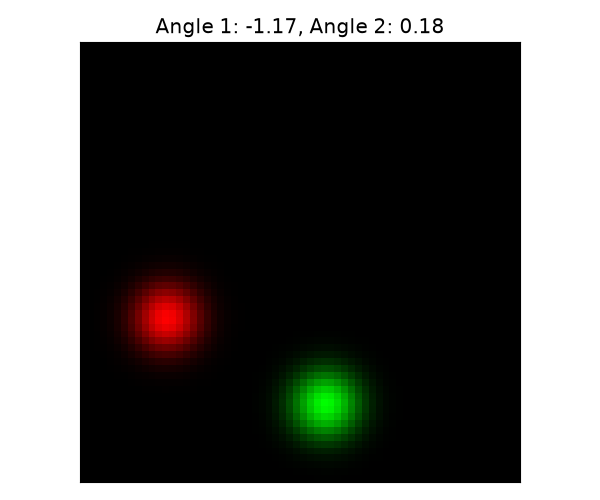}
        \label{fig:double_pendulum_example5}
    \end{subfigure}
    \begin{subfigure}{0.32\linewidth}
        \centering
        \includegraphics[width=\linewidth]{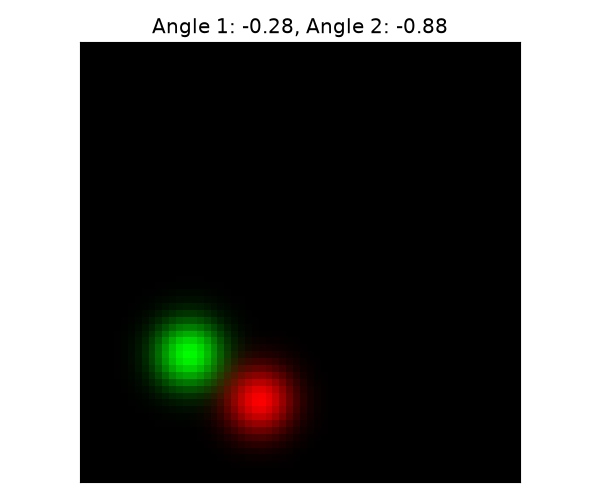}
        \label{fig:double_pendulum_example6}
    \end{subfigure}
    \begin{subfigure}{0.32\linewidth}
        \centering
        \includegraphics[width=\linewidth]{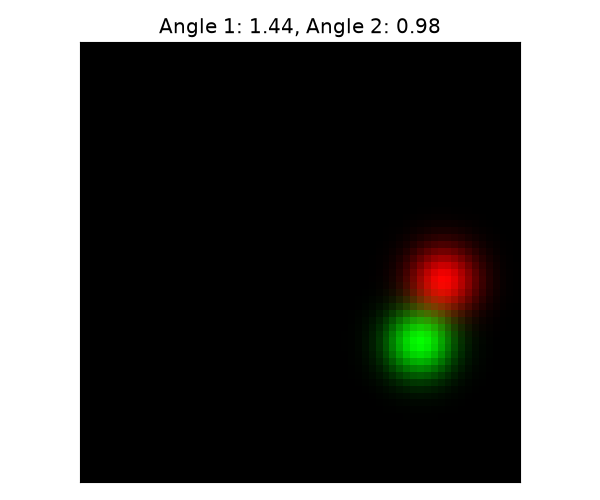}
        \label{fig:double_pendulum_example7}
    \end{subfigure}
    \begin{subfigure}{0.32\linewidth}
        \centering
        \includegraphics[width=\linewidth]{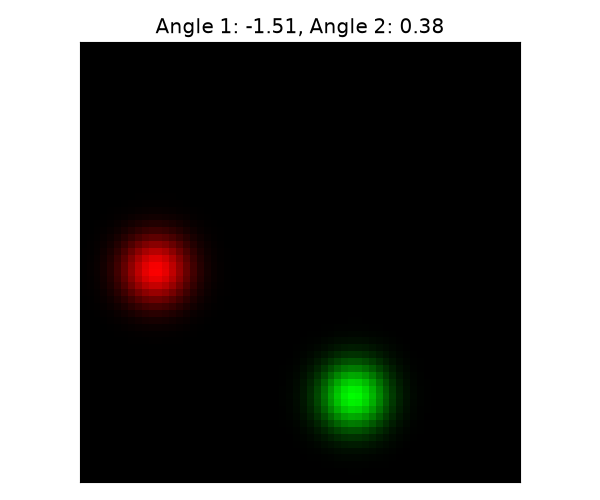}
        \label{fig:double_pendulum_example8}
    \end{subfigure}
    \caption{Example images from the experiment with two pendulums. Pendulums are colored differently to ensure invertibility of the mixing function.}
    \label{fig:double_pendulums_examples}
\end{figure}

\begin{figure}[p]
    \centering
    \begin{subfigure}{0.49\linewidth}
        \centering
        \includegraphics[width=\linewidth]{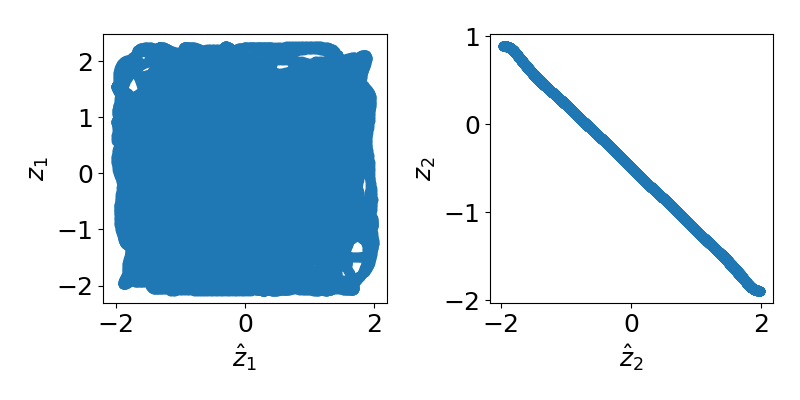}
        \caption{SINDyAE}
        \label{fig:double_pendulum_sindyae_h}
    \end{subfigure}
    \begin{subfigure}{0.49\linewidth}
        \centering
        \includegraphics[width=\linewidth]{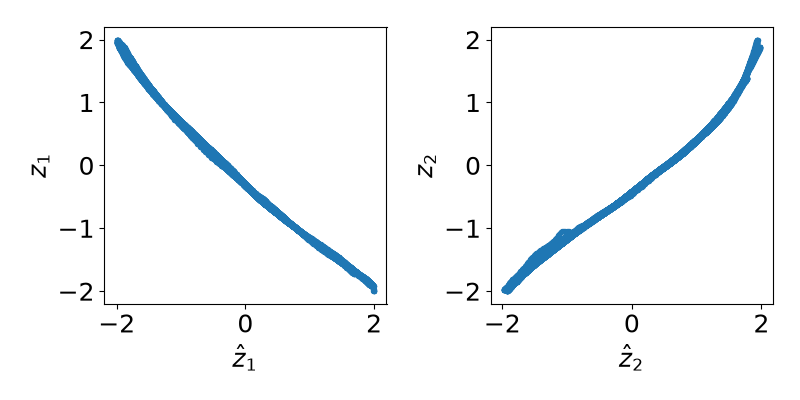}
        \caption{CITRIS}
        \label{fig:double_pendulum_citris_h}
    \end{subfigure}
    \begin{subfigure}{0.49\linewidth}
        \centering
        \includegraphics[width=\linewidth]{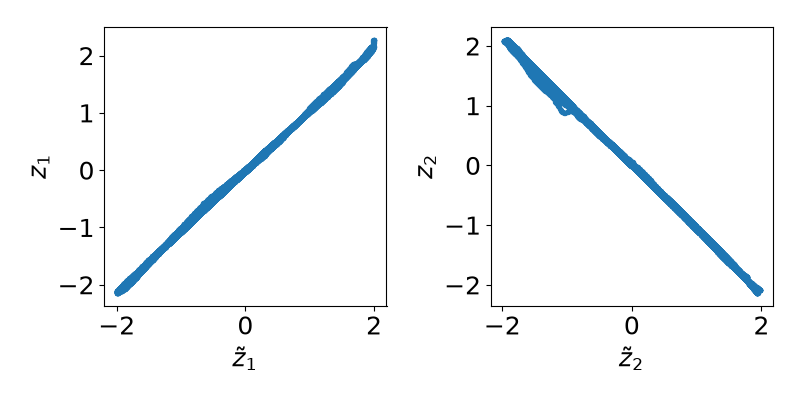}
        \caption{\modelnameshort{} (C+S)}
        \label{fig:double_pendulum_speed_cs_h}
    \end{subfigure}
    \caption{Pendulums experiment. Scatter plots of the found latents $\vzspeed$ against the true ones $\rvz$. If their function is well defined, the plots well represent the maps $z_i=h_i(\zspeed_i)$. In the second plot, we include the same for the CRL latents $\vzcrl$ from CITRIS. \modelnameshort{} is able to achieve linear identifiability from the general diffeomorphism of CITRIS. SINDyAE only identifies one variable, while the other likely represents a mapping between the two pendulums.}
    \label{fig:double_pendulum_diffeos}
\end{figure}